\documentclass{jmlr}

\usepackage{microtype} 

\usepackage{algorithm}
\usepackage{algpseudocode}

\usepackage{mathtools}
\usepackage{bbm}
\usepackage{amsfonts}
\let\vec\undefined
\usepackage{MnSymbol} 

\usepackage{xpatch}

\xpatchcmd{\proof}{\itshape}{\normalfont\proofnameformat}{}{}
\newcommand{\proofnameformat}{\bfseries}

\usepackage{prettyref}
\newcommand{\pref}[1]{\prettyref{#1}}

\newcommand{\savehyperref}[2]{\texorpdfstring{\hyperref[#1]{#2}}{#2}}
\newrefformat{eq}{\savehyperref{#1}{\textup{(\ref*{#1})}}}
\newrefformat{eqn}{\savehyperref{#1}{Equation~\ref*{#1}}}
\newrefformat{lem}{\savehyperref{#1}{Lemma~\ref*{#1}}}
\newrefformat{def}{\savehyperref{#1}{Definition~\ref*{#1}}}
\newrefformat{thm}{\savehyperref{#1}{Theorem~\ref*{#1}}}
\newrefformat{cor}{\savehyperref{#1}{Corollary~\ref*{#1}}}
\newrefformat{sec}{\savehyperref{#1}{Section~\ref*{#1}}}
\newrefformat{app}{\savehyperref{#1}{Appendix~\ref*{#1}}}
\newrefformat{ass}{\savehyperref{#1}{Assumption~\ref*{#1}}}
\newrefformat{ex}{\savehyperref{#1}{Example~\ref*{#1}}}
\newrefformat{fig}{\savehyperref{#1}{Figure~\ref*{#1}}}
\newrefformat{alg}{\savehyperref{#1}{Algorithm~\ref*{#1}}}
\newrefformat{rem}{\savehyperref{#1}{Remark~\ref*{#1}}}
\newrefformat{conj}{\savehyperref{#1}{Conjecture~\ref*{#1}}}
\newrefformat{prop}{\savehyperref{#1}{Proposition~\ref*{#1}}}
\newrefformat{proto}{\savehyperref{#1}{Protocol~\ref*{#1}}}
\newrefformat{prob}{\savehyperref{#1}{Problem~\ref*{#1}}}
\newrefformat{claim}{\savehyperref{#1}{Claim~\ref*{#1}}}
\usepackage[utf8]{inputenc} 
\usepackage[T1]{fontenc}    
\usepackage{url}            
\usepackage{booktabs}       
\usepackage{amsfonts}       

\jmlrpages{}
\jmlrproceedings{}{}

\newtheorem{fact}{Fact}
\newtheorem{assumption}{Assumption}

\title{A Deep Conditioning Treatment of Neural Networks}

\author{%
 \Name{Naman Agarwal} \Email{namanagarwal@google.com}\\
 \addr Google AI, Princeton
 \AND
 \Name{Pranjal Awasthi} \Email{pranjalawasthi@google.com}\\
 \addr Google Research
 \AND
 \Name{Satyen Kale} \Email{satyenkale@google.com}\\
 \addr Google Research
}

\newif\ifnotes\notestrue

\ifnotes
\usepackage{color}
\definecolor{mygrey}{gray}{0.50}
\newcommand{\notename}[2]{{\textcolor{blue}{\footnotesize{\bf (#1:} {#2}{\bf ) }}}}

\else

\newcommand{\notename}[2]{{}}

\fi

\usepackage[mathscr]{euscript}
\usepackage{nicefrac}
\usepackage{dsfont}
\usepackage{wrapfig}
\usepackage{xspace}
\usepackage{enumitem}
\usepackage[font=footnotesize,labelfont=bf]{caption}

\usepackage{color-edits}
\addauthor{sk}{red}
\addauthor{na}{blue}

\newcommand{\vecW}{\vec{W}}
\newcommand{\reals}{\mathbb{R}}
\newcommand{\normal}{\mathcal{N}}
\newcommand{\E}{\mathop{\mathbb{E}}}
\newcommand{\dualact}{\hat{\sigma}}
\newcommand{\cY}{\mathcal{Y}}
\newcommand{\kernel}{k}
\newcommand{\cn}{\kappa}
\newcommand{\ntk}{\mathcal{K}}
\newcommand{\loss}{\mathcal{L}}
\newcommand{\normrelu}{\text{NormReLU}\xspace}

\newcounter{daggerfootnote}

\newcommand{\mutil}{\tilde{\mu}}
\newcommand{\rhofx}{\bar{\rho}}

\begin{document}

\maketitle
\begin{abstract}
 We study the role of depth in training randomly initialized overparameterized neural networks. We give a general result showing that depth improves trainability of neural networks by improving the {\em conditioning} of certain kernel matrices of the input data. This result holds for arbitrary non-linear activation functions under a certain normalization.
 We provide versions of the result that hold for training just the top layer of the neural network, as well as for training all layers, via the neural tangent kernel. As applications of these general results, we provide a generalization of the results of \citet{panigrahy-random-nets} showing that learnability of deep random neural networks with a large class of non-linear activations degrades exponentially with depth.
We also show how benign overfitting can occur in deep neural networks via the results of \citet{bartlett2019benign}. We also give experimental evidence that normalized versions of ReLU are a viable alternative to more complex operations like Batch Normalization in training deep neural networks.
 \end{abstract}

\section{Introduction}
\label{sec:introduction}

Deep neural networks have enjoyed tremendous empirical success, and it has become evident that depth plays a crucial role in this success~\citep{simonyan2014very, szegedy2015going, he2016deep}. However, vanilla deep networks are notoriously hard to train without some form of intervention aimed to improve the optimization process, for example, Batch Normalization~\citep{ioffe2015batch}, Layer Normalization~\citep{layer-norm}, or skip connections in Resnets~\citep{resnet}. Recent theory \citep{santurkar,balduzzi} has shed some light into how these interventions help train deep networks, especially for the widely popular Batch Normalization operation. However the picture is far from clear in light of work such as \citep{yang-bn} which argues that Batch Normalization actually {\em hinders} training by causing gradient explosion. 

In this paper, we investigate {\em improved data conditioning} as a possible factor in explaining the benefits of the aforementioned interventions for training deep neural networks. While standard optimization theory tells us that good data conditioning leads to faster training, empirical performance on test data also seems to be correlated with good conditioning. \pref{fig:test_acc_vs_dp}  presents examples of various deep network architectures trained on the CIFAR-10 dataset using standard techniques such as batch normalization, layer normalization, and a new normalized version of the ReLU activation that we propose in this work. In each case, as the generalization performance increases with the number of epochs, the average normalized dot products between test inputs decrease as well, indicating improved conditioning. This begs the question of whether depth helps in improving conditioning of the data, and as a result affecting optimization and generalization in deep neural networks.
 \begin{figure}[ht]
\centering
\includegraphics[width=.7\textwidth]{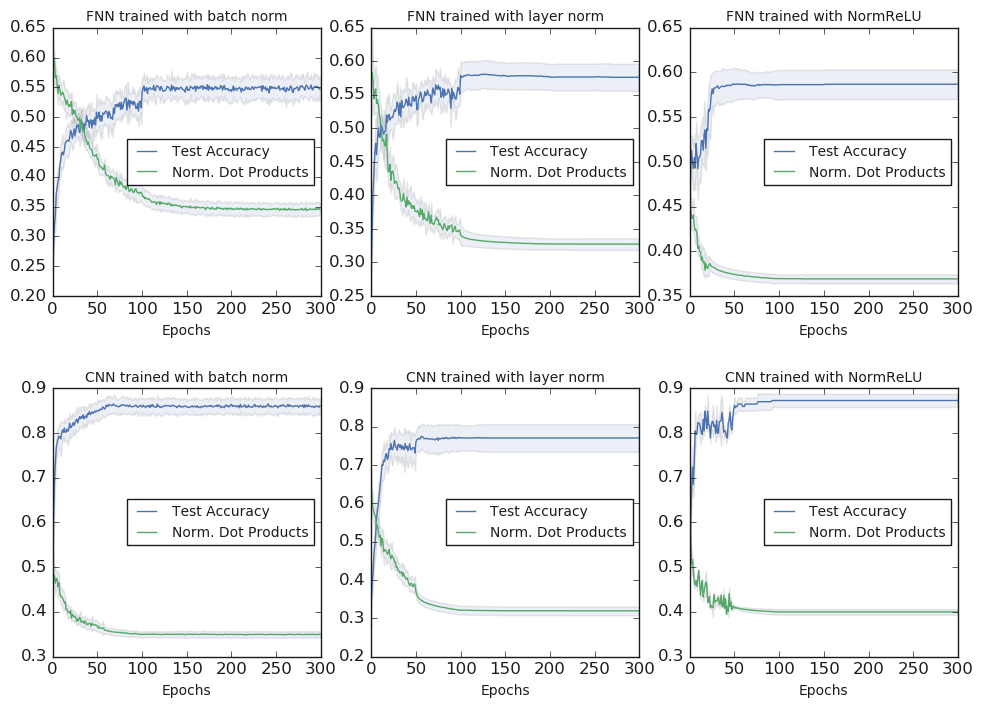}
\captionsetup{indention=0.2cm}
\caption{\label{fig:test_acc_vs_dp}
Test Accuracy and average normalized dot products vs. the number of epochs. The top three plots concern depth-32 feed forward networks with 3000 hidden units per layer. The top left network is trained using batch normalization, the middle network is trained using layer normalization, and the top right network is trained using a normalized version of the ReLU activation. The bottom three plots concern convolutional networks with the ResNet-32 architecture \citep{he2016deep}. The bottom left network is trained using batch normalization, the middle network is trained using layer normalization and the bottom right network is trained using the normalized ReLU activation. }
\end{figure}


We elucidate the role of depth and the non-linearity of activations in improving data conditioning by considering a simple intervention: viz., we {\em normalize} the activations  so that when fed standard Gaussian inputs, the output has zero mean and unit variance. Any standard activation function like ReLU, tanh, etc. can be normalized by centering and scaling it appropriately. Thus normalization of activations is a rather benign requirement (see also \pref{lem:bn-ln-invariance}), but has significant consequences for improving data conditioning theoretically and trainability empirically, as explained next.

\subsection{Our contributions.}
\begin{enumerate}[leftmargin=*]
    \item \textbf{Exponentially improving data conditioning.} We show that for a randomly initialized neural network with an {\em arbitrary non-linear normalized} activation function, the condition number of the certain kernel matrices of the input data tend to the best possible value, $1$, {\em exponentially} fast in the depth of the network. The rate at which the condition number tends to $1$ is determined by a coefficient of non-linearity of the activation function, a concept that we define in this paper. 
    This result holds for either training just the top layer of the neural network, or all layers of the network with a sufficiently small learning rate (the so-called lazy training regime \citep{chizat2018global}). 

    \item \textbf{Fast training.} Our main result implies that when training large width neural networks of sufficient depth, gradient descent with square loss approaches $\epsilon$ training error at a $\log(1/\epsilon)$ rate, {\em regardless} of the initial conditioning of the data. This is in contrast to prior works \citep{arora2019exact, allen2018convergence} and demonstrates the optimization benefits of using deeper networks. 

    \item \textbf{Hardness of learning random neural networks.} Via our main result, we generalize the work of \citet{panigrahy-random-nets} and show that learning a target function that is a sufficiently deep randomly initialized neural network with a general class of activations, requires exponentially (in depth) many queries in the statistical query model of learning. Furthermore, this result holds with constant probability over the random initialization, a considerable strengthening of the prior result. See \pref{sec:sq} for the formal result and a detailed comparison.

    \item \textbf{Benign overfitting in deep neural networks.} We extend the work of \citet{bartlett2019benign} on interpolating classifiers and show that randomly initialized and sufficiently deep neural networks can not only fit the training data, but in fact, the minimum norm (in the appropriate RKHS) interpolating solution can achieve non-trivial excess risk guarantees as well.

    \item \textbf{Empirical benefits of normalized activations.} Guided by our theoretical results, we propose a new family of activation functions called \normrelu which are normalized versions of the standard ReLU activation. Incorporating \normrelu into existing network architectures requires no overhead. Furthermore, we show via experiments on the CIFAR-10 dataset that \normrelu can serve as an effective replacement for techniques such as batch normalization and layer normalization. This leads to an alternate method for training deep networks with no loss in generalization performance and in some cases leads to significant gains in training time. 
\end{enumerate}



\subsection{Related work}
\label{sec:related}

There are two very recent works with similar results to ours independently of our work. The first is the work of \citet{XPS}, which uses the tools of mean-field theory of deep neural network developed in a long line of work \citep{poole,daniely2017,deepinfprop,resurrecting}. This work considers a broader spectrum of initialization schemes and activation functions and studies the effect of depth on data conditioning. Specializing to the setting of our paper, this work shows that if the inputs are already very well-conditioned, then they converge to perfect conditioning exponentially fast in the depth of the network. In contrast, our results show exponential convergence even if the inputs are very poorly conditioned: in fact, for some activations like normalized ReLU, the initial condition number could even be infinite. The second is the work of \citet{PSG}, whose main motivation is studying the effect of smooth vs. non-smooth activation functions in shallow networks. However they do show a very similar exponential convergence result like ours for a kernel matrix closely related to the top-layer kernel matrix in this paper for a more restricted class of activation functions than considered in our paper. Furthermore the results of \citet{PSG} assume unit length inputs, whereas in this paper we extend our results for some activations to non-unit length inputs as well. Neither paper considers the applications to optimization, SQ learning of random neural networks, and benign overfitting as done in this paper.

On the optimization side, a sequence of papers has recently shown the benefits of overparametrization via large {\em width} for training neural networks: see, for example, \citep{li-liang,du-zhai-poczos-singh,allen-zhu-li-song,zou-gu} and the references therein. These papers show that with sufficiently large width, starting from a random initialization of the network weights, gradient descent provably finds a global minimizer of the loss function on the training set. While several of the aforementioned papers do analyze deep neural networks, to our knowledge, there is no prior work that provably demonstrates the benefits of depth for training neural networks in general settings. Prevailing wisdom is that while depth enables the network to express more complicated functions (see, for example, \citep{eldan-shamir,telgarsky,raghu-poole-kleinberg-ganguli-sohl-dickstein,lee-ge-ma-risteski-arora,daniely2017} and the references therein), it hinders efficient training, which is the primary concern in this paper. Indeed, the papers mentioned earlier showing convergence of gradient descent either assume very shallow (one hidden layer) networks, or expend considerable effort to show that depth doesn't {\em degrade} training by more than a polynomial factor. In contrast, we show that after a certain threshold depth (which depends logarithmically on $\delta$, initial separation), increasing depth improves the convergence rate exponentially. 

To provide one such precise comparison, the work of \cite{allen-zhu-li-song} under the same separation assumption as us, proves that overparametrized networks with ReLU activations converge in time polynomial in depth, $1/\delta$ and $\log(1/\epsilon)$. Our results show that if depth is $\Omega(\log(1/\delta))$, then the convergence rate is only proportional to $\log(1/\epsilon)$, {\em independent} of $\delta$, if one uses normalized activations. 

A few exceptions to the above line of work are the papers \citep{arora-cohen-hazan,arora-cohen-hu-luo} which do show that depth helps in training neural networks, but are restricted to very specific problems with linear activations. 

See \pref{app:related-app} for an in-depth discussion of these and other related works.



\section{Notation and preliminaries}
\label{sec:setup}

For two vectors $x$ and $x'$ of like dimension, we denote their inner product by $x\cdot x'$. Unless otherwise specified, $\|\cdot\|$ denotes the Euclidean norm for vectors and the spectral norm for matrices. For a symmetric positive definite matrix $M$, the condition number $\cn(M)$ is defined to be the ratio $\frac{\lambda_{\max}(M)}{\lambda_{\min}(M)}$, where $\lambda_{\max}(M)$ and $\lambda_{\min}(M)$ are the largest and smallest eigenvalues respectively of $M$. For a positive integer $n$, define $[n] = \{1, 2, \ldots, n\}$.

We are given a training set of $n$ examples: $S = \{(x_i, y_i) \in \reals^d \times \cY\}_{i=1}^n$, where $\cY$ is the output space. For simplicity we begin by assuming, as is standard in related literature, that for all $i$ we have $\|x_i\| = 1$. We provide extensions of our results to non-unit-length inputs in \pref{sec:general-norms}. Let $\underline{K} \in \mathbb{R}^{n \times n}$ be the Gram matrix of the training data, i.e. $\underline{K}_{ij} = x_i \cdot x_j$. We make the following (very standard in the literature, see e.g. \citep{allen2018convergence,zou-gu}) assumption on the input data:
\begin{assumption}
\label{ass:separation}
	For all $i, j \in [n]$ with $i \neq j$, we have $|x_i \cdot x_j| \leq 1 - \delta$.
\end{assumption}

To keep the presentation as clean as possible, we assume a very simple architecture of the neural network\footnote{Extending our analysis to layers of different sizes and outputs of length greater than $1$ poses no mathematical difficulty and is omitted for the sake of clarity of notation.}: it has $L$ hidden fully-connected layers, each of width $m$, and takes $x \in \reals^d$ as input and outputs $y \in \reals$, with activation function $\sigma: \reals \rightarrow \reals$ to $\reals^m$ by entry-wise application. The network can thus be defined as the following function\footnote{Note that we're using the so-called neural tangent kernel parameterization \citep{ntk} instead of the standard parameterization here.} $f_{\vecW}: \reals^d \rightarrow \reals$:
\begin{equation} \label{eq:network-def}
   f_{\vecW}(x) = v \cdot \tfrac{1}{\sqrt{m}}\sigma(W_L \tfrac{1}{\sqrt{m}}\sigma(W_{L-1} \cdots \tfrac{1}{\sqrt{m}}\sigma(W_1x)\cdots)), 
\end{equation}
where $W_1 \in \reals^{m \times d}$, $W_2, \ldots, W_L \in \reals^{m \times m}$ denote the weight matrices for the hidden layers, $v \in \reals^m$ denotes the weight vector of the output layer, $\vecW \in \reals^{dm + (L-1)m^2 + m}$ denotes a vector obtained by concatenating vectorizations of the weight matrices.
We use the notation $\normal(\mu, \Sigma)$ for the normal distribution with mean $\mu$ and covariance $\Sigma$. All weights are initialized to independent, standard normal variables (i.e. drawn i.i.d. from $\normal(0, 1)$).
    
Our analysis hinges on the following key normalization assumption on $\sigma$:
\begin{equation}\label{eq:normalization}
	\E_{X \sim \normal(0, 1)}[\sigma(X)] = 0\ \text{ and } \mathop{\text{Var}}_{X \sim \normal(0, 1)}[\sigma(X)] = \E_{X \sim \normal(0, 1)}[\sigma^2(X)] = 1.
\end{equation}
This normalization requirement is rather mild since any standard activation function can be easily normalized by centering it by subtracting a constant and scaling the result by a constant. The only somewhat non-standard part of the normalization is the requirement that the activation is centered so that its expectation on standard normal inputs is 0. This requirement can be relaxed (see \pref{sec:unnormalized-extension}) at the price of worse conditioning. 
Furthermore, the following lemma proved in \pref{app:bnlninvariance} shows that in the presence of other normalization techniques, normalized activations may be assumed without loss of generality:
\begin{lemma} \label{lem:bn-ln-invariance}
If the neural network in \eqref{eq:network-def} incorporates batch normalization in each layer, then the network output is the same regardless of whether the activation $\sigma$ is normalized or not. The same holds if instead layer normalization is employed, but on the post-activation outputs rather than pre-activation inputs.
\end{lemma}



\section{Main results on conditioning of kernel matrices}
\label{sec:conditioning-results}


\subsection{Top layer kernel matrix.} 
\label{sec:top-layer}

The first kernel matrix we study is the one defined by (random) feature mapping generated at the top layer by the lower layer weights, i.e.\footnote{Note that $\Phi_{\vecW}$ does not depend on the $v$ component of $\vecW$; this notation is chosen for simplicity.} $\Phi_{\vecW}(x) := \tfrac{1}{\sqrt{m}}\sigma(W_L \tfrac{1}{\sqrt{m}}\sigma(W_{L-1} \cdots \tfrac{1}{\sqrt{m}}\sigma(W_1x)\cdots))$. The feature mapping $\Phi_{\vecW}$ defines a kernel function $k$ and the associated $n \times n$ kernel matrix $K$ on a training set $S$ as $K_{ij} := \kernel(x_i, x_j)$ where $\kernel(x, x') := \Phi_{\vecW}(x) \cdot \Phi_{\vecW}(x')$.

The main results on conditioning in this paper are cleanest to express in the limit of infinite width neural networks, i.e. $m \rightarrow \infty$. In this limit, the kernel function $k$ and the kernel matrix $K$, tend almost surely to {\em deterministic} limits \citep{daniely2016toward}, denoted as $\bar{\kernel}$ and $\bar{K}$ respectively.
We study the conditioning of $\bar{K}$ next. The rate at which the condition number of $\bar{K}$ improves with depth depends on the following notion of degree of non-linearity of the activation function $\sigma$:
\begin{definition}
\label{def:coefficient-non-linearity}
	The {\em coefficient of non-linearity} of the activation function $\sigma$ is defined to be $\mu := 1 - \left(\E_{X \sim \normal(0, 1)}[X \sigma(X)]\right)^2$.
\end{definition}
The normalization \pref{eq:normalization} of the activation function implies via \pref{lem:mu-bounds} (in \pref{app:conditioning}, where all missing proofs of results in this section can be found) that for any non-linear activation function $\sigma$, we have $0 < \mu \leq 1$. To state our main result, it is convenient to define the following quantities: for any $\nu \in (0,1]$, $\delta \in (0, 1)$ and a positive integer $L$, let $L_0(\delta) = \max\left\{\left\lceil\tfrac{\log(\frac{1}{2\delta})}{\log(1+\frac{\nu}{2})}\right\rceil, 0\right\} = O\left(\tfrac{\log(1/\delta)}{\nu}\right)$, and define
\[B_{\nu}(L, \delta) := \begin{cases}
    1 - \delta(1+\frac{\nu}{2})^L & \text{ if }\ L \leq L_0(\delta) \\
    \frac{1}{2}(1 - \frac{\nu}{2})^{L-L_0(\delta)} & \text{ if }\ L > L_0(\delta).
\end{cases}\]
For clarity of notation, we will denote by $B(L, \delta)$ the quantity $B_{\mu}(L, \delta)$. We are now ready to state our main result on conditioning of the kernel matrix:
\begin{theorem}
\label{thm:main-infinite-width-top-layer}
   Under \pref{ass:separation}, we have $|\bar{K}_{ij}| \leq B(L, \delta)$ for all $i, j \in [n]$ with $i \neq j$. 
\end{theorem}
The following corollary is immediate, showing that the condition number of the kernel matrix $\bar{K}$ approaches the smallest possible value, 1, {\em exponentially} fast as depth increases.  
\begin{corollary}
\label{cor:main-top-layer-cn} 
		Under \pref{ass:separation}, if $L \geq L_1(\delta) := \left\lceil \frac{\log(n)}{-\log(1 - \frac{\mu}{2})}\right\rceil + L_0(\delta)$, then $\cn(\bar{K}) \leq 1 + 2n(1~-~\frac{\mu}{2})^{L - L_1(\delta)}$.
\end{corollary}

\subsection{Neural tangent kernel matrix.}
\label{sec:ntk}

The second kernel matrix we study arises from the {\em neural tangent kernel}, which was introduced by \citet{ntk}. This kernel matrix naturally arises when all the layers of the neural network are trained via gradient gradient. For a given set of network weights $\vecW$, the neural tangent kernel matrix $\ntk \in \mathbb{R}^{n \times n}$ is defined as $\ntk_{ij} = (\partial_{\vecW} f_{\vecW}(x_i)) \cdot (\partial_{\vecW} f_{\vecW}(x_j))$. As in the previous section, as the width $m$ of the hidden layers tends to infinity, the random $\ntk$ tends to a deterministic limit, $\bar{\ntk}$. For this infinite width limit, we have the following  theorem analogous to part 1 of \pref{thm:main-infinite-width-top-layer}:
\begin{theorem}
\label{thm:main-infinite-width-ntk}
	The diagonal entries of $\bar{\ntk}$ are all equal. Assume that $L \geq 2L_0(\delta)$. Under \pref{ass:separation}, we have $|\bar{\ntk}_{ij}| \leq 2B(\nicefrac{L}{2},\delta) \cdot \bar{\ntk}_{11}$ for all $i, j \in [n]$ with $i \neq j$.
\end{theorem}
The following corollary, analogous to \pref{cor:main-top-layer-cn}, is immediate:
\begin{corollary}
\label{cor:main-infinite-width-ntk-cn}
    Under \pref{ass:separation}, if $L \geq L_2(\delta) := \left\lceil \frac{2\log(2n)}{-\log(1 - \frac{\mu}{2})}\right\rceil + 2L_0(\delta)$, then $\cn(\bar{\ntk}) \leq 1+4n(1~-~\tfrac{\mu}{2})^{\nicefrac{L}{2} - L_2(\delta)}$.
\end{corollary}

\subsection{Better conditioning under stronger assumption}
\label{sec:stronger-assumption}

The following somewhat stronger assumption than \pref{ass:separation} leads to a better conditioning result:
\begin{assumption} 
\label{ass:non-singularity}
    $\lambda_{\min}(\underline{K}) \geq \delta$.
\end{assumption}
While \pref{ass:non-singularity} implies \pref{ass:separation}, it still quite benign, and is easily satisfied if $n \geq d$ and there is even a tiny amount of inherent white noise in the data. Furthermore, as discussed in \pref{app:relu_cond}, for certain activations like ReLU,
the representations derived after passing a dataset satisfying \pref{ass:separation} through one layer satisfy \pref{ass:non-singularity}. 

We have the following stronger versions of \pref{thm:main-infinite-width-top-layer} and \pref{cor:main-top-layer-cn} under \pref{ass:non-singularity} (all proofs appear in \pref{app:conditioning}):
\begin{theorem}
\label{thm:main-infinite-width-top-layer-stronger}
    Under \pref{ass:non-singularity}, we have $\lambda_{\min}(\bar{K}) \geq 1 - B(L, \delta)$.
\end{theorem}
\begin{corollary}
\label{cor:main-top-layer-cn-stronger} 
    Under \pref{ass:non-singularity}, we have $\cn(\bar{K}) \leq 1 + \frac{n}{\delta}(1 + \tfrac{\mu}{2})^{-L}$.
\end{corollary}
Similarly, we have the following stronger versions of \pref{thm:main-infinite-width-ntk} and \pref{cor:main-infinite-width-ntk-cn} under \pref{ass:non-singularity}:
\begin{theorem}
\label{thm:main-infinite-width-ntk-stronger}
    Under \pref{ass:non-singularity}, we have $\lambda_{\min}(K) \geq \left(1 -  2B(\nicefrac{L}{2},\delta)\right)\bar{\ntk}_{11}$.
\end{theorem}
\begin{corollary}
\label{cor:main-infinite-width-ntk-cn-stronger}
    Under \pref{ass:non-singularity}, if $L \geq 4L_0(\delta)$, then $\cn(\bar{\ntk}) \leq 1+ \tfrac{2n}{\delta}(1+\tfrac{\mu}{2})^{-\nicefrac{L}{2}}$.
\end{corollary}

\subsection{Extension to uncentered activations}
\label{sec:unnormalized-extension}

The analysis techniques of the previous sections also extend to activations that need not be normalized \eqref{eq:normalization}. Specifically, we only assume that the activation $\sigma$ satisfies $\E_{X \sim \normal(0, 1)}[\sigma^2(X)] = 1$. I.e., we allow $\E_{X \sim \normal(0, 1)}[\sigma(X)]$ to be non-zero. In this case, we can show that for non-affine activations, dot products of input representations at the top layer converge to a fixed point as the depth increases (proof appears in \pref{app:unnormalized-analysis}):
\begin{theorem}
\label{thm:unnormalized-main-infinite-width-top-layer}
 Suppose $\sigma$ is non-affine and \pref{ass:separation} holds. Then, there is a $\rhofx \in [0, 1]$ such that $\lim_{L \rightarrow \infty} \bar{K}_{ij} = \rhofx$. Furthermore, if $\rhofx < 1$, then there are constants $c_\sigma \in (0, 1)$ and $L_0 = O_\sigma(\log(\frac{1}{\delta}))$ such that if $L \geq L_0$, then $|\bar{K}_{ij} - \rhofx| \leq c_{\sigma}^{(L-L_0)}$ for all $i, j \in [n]$ with $i \neq j$.    
\end{theorem}

\subsection{Extension to non-unit length inputs}
\label{sec:general-norms}

In this section we extend the result of \pref{sec:top-layer} to the case when the inputs are do not have to be exactly unit length. To establish these results we require further assumptions on the activation function, which we highlight in the theorem. For a discussion of these assumptions see \pref{app:general_result_discussion}. 

\begin{theorem} 
\label{thm:general_norms}
Let $\sigma$ be a twice-differentiable monotonically increasing odd function which is concave on $\reals^+$. There exists a constant $\alpha_\sigma$ (depending on $\sigma$) such that for any two inputs $x,y$ such that $\|x\|^2,\|y\|^2 \geq 0.5, \frac{|x^{\top}y|}{\|x\|\|y\|} \leq (1 - \delta) $, after a number of layers 
$L \geq \hat{L} := \alpha_{\sigma}\log(\max(|\|x\|^2-1|,|\|y\|^2-1|, \mu/4)\cdot4/\mu)$, we have 
$\frac{ \bar{k}(x,y)}{\sqrt{\bar{k}(x,x) \cdot \bar{k}(y,y)}} \leq B_{\mu/2}\left(L - \hat{L}, \delta\right)$. 
\end{theorem}

The proof of the above theorem as well as a precise description of the constant $\alpha_{\sigma}$ can be found in 
\pref{app:general_norm_proof} in the supplementary material. Our analysis proceeds by first proving the theorem for the norms of the representations induced by the input. The above theorem formalizes a sufficient condition on the activation function for \textit{global} convergence to the fixed point (i.e. 1) of the length map defined in \cite{poole}. In fact, for this part we establish a weaker sufficient condition on the activation than \pref{thm:general_norms}, see \pref{app:general_norm_proof} for details.
\cite{poole} informally mention that monotonicity of activations suffices, although counterexamples exist, see \pref{app:general_result_discussion}. Furthermore, the theorem generalizes the work of \citet{XPS} which only provides the asymptotics close to the fixed point. Next, we show the monotonicity of the normalized dot-product for the representations. This allows us to leverage our previous analysis for the norm $1$ case once the norms have converged. 

A similar analysis and theorem can be obtained for the NormReLU activation we propose in this paper (details in \pref{sec:experiments}). For the precise theorem statement and proof, see \pref{app:normrelubounds}.

\section{Implications for optimization}
\label{sec:optimization}

Suppose we train the network using gradient descent on a loss function $\ell: \reals \times \cY \rightarrow \reals$, which defines the empirical loss function $\loss(\vecW) := \frac{1}{n}\sum_{i=1}^n \ell(f_{\vecW}(x_i), y_i)$.  For the rest of this section we will assume that the loss function $\ell$ is the square loss, i.e. $\ell(\hat{y},y)=(\hat{y}-y)^2$. The results presented can appropriately be extended to the setting where the loss function is smooth and strongly convex. 
Training a finite-width neural network necessitates the study of the conditioning of the finite-width kernel matrices $K$ and $\ntk$, rather than their infinite-width counterparts. In such settings optimization results typically follow from a simple 2-step modular analysis, where in the first step we show via concentration inequalities that conditioning in the infinite-width case transfers to the finite-width case, and in the second step we show that conditioning is not hurt much in the training process.
We now provide a couple of representative optimization results that follow from this type of analysis. 

\subsection{Training only the top layer}

We consider a mode of training where only the top layer weight vector, $v$, is updated, while keeping $W_1, W_2, \ldots, W_L$ frozen at their randomly initialized values. To highlight this we introduce the notation $\vecW_{1:L} = \{W_1 \ldots W_L\}$. Let  $\eta > 0$ be a step size, the update rule at iteration $t$ is given by 
$v_{t+1} = v_t - \eta \cdot \partial_{v}\loss(\{v_t, \vecW_{1:L}\}) = v_t - \eta \cdot \frac{1}{n} \sum_{i=1}^n 2(v_t \cdot \Phi_{\vecW}(x_i) - y_i) \Phi_{\vecW}(x_i)$. 
Note that in this mode of training, the optimization problem is convex in $v$. 
We assume that $\sigma$ satisfies a regularity conditioning, {\em $C$-boundedness}, introduced by \citet{daniely2016toward}, which allows us to apply their concentration bounds. Then, standard convex optimization theory \citep{nesterov-lectures} gives the following result (precise statements 
and proof are in \pref{app:opt-proofs}):
\begin{theorem}
\label{thm:top-layer-general}
Suppose $L = \Theta\left(\frac{\log(n/\delta)}{\mu}\right)$, $\sigma$ is $C$-bounded and the width $m = \text{poly}(n, \frac{1}{\delta})$. Then for an appropriate choice of $\eta$, with high probability over the initialization, gradient descent finds an $\epsilon$ sub-optimal point in $O(\log(\frac{1}{\epsilon}))$ steps. The same result holds for stochastic gradient descent as well.
\end{theorem}



\subsection{Training All The Layers Together}

In this section we provide a representative result for the training dynamics when all the layers are trained together with a fixed common learning rate. The dynamics are given by $\vecW(t+1) = \vecW(t) - \eta \partial_{\vecW}\loss(\vecW(t))$. 
The analysis in this setting follows from carefully establishing that the NTK does not change too much during the training procedure allowing for the rest of the analysis to go through. We have the following theorem, using the concentration bounds of \citet{lee-wide-neural-2019} (precise theorem statement 
and proof are in \pref{app:opt-proofs}):
\begin{theorem}
\label{thm:train_all_layers}
Suppose $\sigma$ is smooth, bounded and has bounded derivatives. If the width is a large enough constant (depending on $L,n,\delta$) and $L = \Theta(\frac{\log(n/\delta)}{\mu})$, then gradient descent with high probability finds an $\epsilon$ suboptimal point in $O(\log(1/\epsilon))$ iterations.
\end{theorem} 



\section{SQ Learnability of Random Deep Neural Nets}
\label{sec:sq}
In this section we 
give a generalization of the recent result of \cite{panigrahy-random-nets} regarding learnability of random neural networks. This work studied randomly initialized deep neural networks with sign activations at hidden units. Motivated from the perspective of complexity of learning, they studied learnability of random neural networks in the popular {\em statistical query learning}~(SQ) framework \citep{kearns1998efficient}. Their main result establishes that any algorithm for learning a function that is a randomly initialized deep network with sign activations, requires exponential (in depth) many statistical queries in the worst case. 

Here we generalize their result in two ways: (a) our result applies to arbitrary activations (as opposed to just sign activations in \cite{panigrahy-random-nets}) satisfying a subgaussianity assumption for standard Gaussian inputs, and (b) our lower bound shows that a randomly initialized network is hard to learn in the SQ model with {\em constant} probability, as opposed to just positive probability, in \cite{panigrahy-random-nets}. We achieve the stronger lower bound by carefully adapting the lower bound technique of \citet{bshouty2002using}. The subgaussianity assumption is there exists a constant $\alpha > 0$ such that for all $\lambda$, we have $\E_{X \sim \normal(0,1)}[e^{\lambda \sigma(X)}] \leq e^{\lambda^2 \alpha^2/2}$ for all $\lambda$. All standard activations (such as the sign, ReLU and tanh), when normalized, satisfy this assumption.

For technical reasons, we will work with networks that normalize the output of each layer to unit length via the operation $\Pi: \mathbb{R}^m \rightarrow \mathbb{R}^m$, and thus the neural network function is of the form 
\begin{equation} \label{eq:nn-with-proj}
f_{\vecW}(x) =  \big(v \cdot \tfrac{1}{\sqrt{m}}\Pi(\sigma(W_L \tfrac{1}{\sqrt{m}} \Pi(\sigma(W_{L-1} \cdots \tfrac{1}{\sqrt{m}} \Pi(\sigma(W_1x)\cdots))) \big).	
\end{equation}
We will consider learning $\text{sgn}(f_{\vecW}(x))$ in the SQ model \citep{kearns1998efficient} where the learning algorithm does not have access to a labeled training set. Instead, for a given target function $f$ and a distribution $D$ over $\mathbb{R}^d$, the algorithm has access to a query oracle $SQ_{f,D}(\psi, \tau)$. The oracle takes as input a query function $\psi$ and a tolerance parameter $\tau \geq 0$, and outputs a value $v$ such that $|\E_D[\psi(x,f(x))] - v| \leq \tau$. The goal of the algorithm is to use the query algorithm to output a function $g$ that is $\epsilon$-correlated with $f$, i.e., $Pr_D[g(x) f(x)] \geq \epsilon$, for a given $\epsilon > 0$. Our main result is the following (proofs can be found in \pref{app:sq-app}):
\begin{theorem}
\label{thm:sq-lower-bound}
Fix any nonlinear activation $\sigma$ with the coefficient of non-linearity $\mu$ that satisfies the subgaussianity assumption. Let $f_{\vecW}$ be an $L$-layer neural network with width $m = \Omega(L\frac{\mu^2}{\delta^2})$ taking inputs of dimension $d$ with weights randomly initialized to standard Gaussians. Any algorithm that makes at most $p(d, L)$ statistical queries with tolerance $1/\text{poly}(d, L)$ and outputs a function that is $1/\text{poly}(d, L)$-correlated with $\text{sgn}(f_{\vecW})$ must satisfy $p(d,L) \geq e^{\Omega(L)}$.
\end{theorem}
 A key component in establishing the above SQ hardness of learning is to show that given two non-collinear unit length vectors, a randomly initialized network of depth $h$ and sufficiently large width makes, in expectation, the pair nearly orthogonal. In other words, the magnitude of the expected dot product between any pair decreases exponentially with depth.  While \citet{panigrahy-random-nets} proved the result for sign activations, we prove the statement for more general activations and then use it to establish hardness of learning in the SQ model.


\section{Benign Overfitting in Deep Neural Networks}
\label{sec:interpolation}

\newcommand{\RKHS}{\mathcal{H}}
\newcommand{\sphere}{\mathbb{S}}
\newcommand{\yvec}{\mathbf{y}}


In this section, we give an application of our conditioning results showing how interpolating classifiers (i.e. classifiers achieving perfect training accuracy) can generalize well in the context of deep neural networks. Specifically, we consider the problem of linear regression with square loss where the feature representation is obtained via a randomly initialized deep network, and an interpolating linear predictor is obtained by training only the top layer (i.e. the $v$ vector). Since there are infinitely many interpolating linear predictors in our overparameterized setting, we focus our attention on the \textit{minimum norm} predictor. 

Our result builds on the prior work of \citet{bartlett2019benign}, which studies benign overfitting in kernel least-squares regression, where the kernel is externally provided. They prove a benign overfitting result (i.e. generalization error going to 0 with increasing sample size) assuming that the spectrum of the kernel matrix decays at a certain slow rate. In this work, we {\em construct} the kernel via a randomly initialized deep neural network, and show, via our conditioning results, that the spectrum of the kernel matrix decays slowly enough for a benign overfitting result to hold. While the results of \citet{bartlett2019benign} assume a certain {\em well-specified} setting for the data to prove their generalization bound, we prove our result in the {\em misspecified} (or {\em agnostic}) setting and give an {\em excess risk} bound.

In our setting, the input space is the $d$ dimensional unit sphere $\sphere^{d-1}$,  the output space $\cY = [-1, 1]$, and samples $(x, y) \in \sphere^{d-1} \times [-1, 1]$ are drawn from an unknown distribution $\mathcal{D}$. The training set is $S = \{(x_i, y_i) \in \sphere^{d-1} \times [-1, 1]\}_{i=1}^n$. To simplify the presentation, we work in the infinite width setting, i.e. we learn the minimum norm linear predictor in the RKHS $\RKHS$ corresponding to the kernel function $\bar{\kernel}$ (see \pref{sec:top-layer}). 
Let $\Phi: \mathbb{S}^{d-1} \rightarrow \RKHS$ be the feature map corresponding to $\bar{\kernel}$. 
The loss of a linear predictor parameterized by $v \in \RKHS$ on an example $(x, y)$ is $(y - v^\top \Phi(x))^2$. We denote by $v^*$ the optimal linear predictor, i.e. a vector in $\arg\min_{v \in \RKHS} \E_{(x, y)}[(y - v^\top \Phi(x))^2]$, and by $v_S$ the minimum norm interpolating linear predictor, if one exists.
A key quantity of interest is the function $\Delta: \mathbb{N} \times [0, 1] \rightarrow [0, 1]$ defined as follows: 
if $T = \{x'_1, x'_2, \ldots, x'_m\}$ denotes a sample set of size $m$ drawn i.i.d. from the marginal distribution of $\mathcal{D}$ over the $x$-coordinate, then
\[\Delta(m, \gamma) := \sup\left\{\delta:\ \Pr_{T}\left[ \max_{i,j \in [m]:\ i \neq j}|x'_i \cdot x'_j| \leq 1-\delta\right] \geq 1-\gamma\right\}.\]
With this definition, we have the following excess risk bound (proof in \pref{app:interpolation}):
\begin{theorem}
\label{thm:interpolation}
For any $\gamma \in (0, \nicefrac{1}{2})$, let $L = \left\lceil \frac{\log(n^2)}{-\log(1 - \frac{\mu}{2})}\right\rceil + L_0(\Delta(n^2, \gamma))$. Then, with probability at least $1 - \gamma$ over the choice of $S$, there exists an interpolating linear predictor, and we have
\[\E_{(x, y)}[(y - v_S^\top \Phi(x))^2] - \E_{(x, y)}[(y - {v^*}^\top \Phi(x))^2] \leq O\left(\frac{\log(n/\gamma)}{n}\|v^*\|^2\right).\]
\end{theorem}
A few caveats about the theorem are in order. Note that the number of layers, $L$, and therefore $\RKHS$ and the optimal linear predictor $v^*$ depends on the sample size $n$. Thus, the excess risk goes to $0$ when $n$ increases if $\|v^*\| = o(\sqrt{n})$. 


\section{Experiments}
\label{sec:experiments}

In this section we present empirical results supporting our theoretical findings, and evaluate the effectiveness of using normalized activations as a replacement for standard operations such as batch normalization and layer normalization in training deep networks.

\paragraph{Normalized ReLU.} 
\begin{wrapfigure}{r}{0.3\textwidth}
    \centering
    \includegraphics[width=0.23\textwidth]{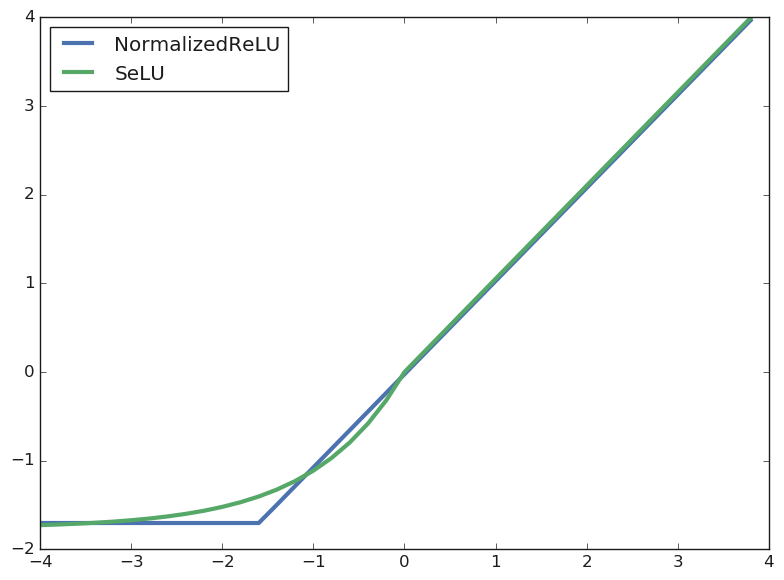} 
    \captionsetup{indention=0.2cm}
    \caption{\label{fig:norm_relu} \normrelu vs SeLU.}
\end{wrapfigure}
Motivated by the practical success of ReLU, we propose a family of normalized ReLU-like functions, parameterized by a scalar $c$, the location of the kink:
\begin{align}
\label{eq:newrelu}
\normrelu_c(x) = \lambda(c) \cdot [\max \{x-c, 0\} + b(c)],
\end{align}
where the constants $\lambda(c)$ and $b(c)$ are chosen to normalize the function (i.e. \eqref{eq:normalization} holds). In \pref{app:normrelu}, we derive the following closed form expressions for $b(c)$ and $\lambda(c)$: if $\varphi(x)$ and $\varPhi(x)$ are the Gaussian density and cumulative distribution functions respectively, then $b(c) = (1 - \varPhi(c))c - \varphi(c)$ and $\lambda(c) = [(1-\varPhi(c))\varPhi(c)c^2 + (1-2\varPhi(c))\varphi(c)c + (1-\varPhi(c)-\varphi(c)^2)]^{-1/2}$.


In our experimental setup we choose $c=-1.5975$ since this gives $\lambda(c) \approx 1.05$, which is the same scaling factor in the SeLU activation~\cite{klambauer2017self}.
For this value of $c$, we have $b(c) \approx -1.6209$. In the following, we refer to $\normrelu_{-1.5975}$ simply as \normrelu for convenience. Figure~\ref{fig:norm_relu} shows this \normrelu activation compared to SeLU. 
\pref{app:normrelu} has a comparison of the two activations in terms of training and generalization behavior. \pref{app:normrelu} also contains additional experimental details for this section including the choice of hyperparameters and the number of training and evaluation runs.

\paragraph{Effectiveness of \normrelu for training deep neural networks.} We first train fully connected feedforward networks of depth $32$ and $64$ on the CIFAR-10 dataset using 
either batch normalization, layer normalization or \normrelu. For each method, the best learning rate is chosen via cross validation. Figure \ref{fig:fnn} below shows how the training and the test accuracy increases with the number of epochs. As predicted by our theory, using \normrelu results in significantly faster optimization. Furthermore, the model trained via \normrelu also generalizes significantly better than using either batch or layer normalization. However, when using \normrelu we observed that we had to use small learning rates, of the order of $10^{-5}$ to stabilize training. Batch or Layer normalizations on the other hand are less sensitive and can be applied in conjunction with large learning rates.
\begin{figure}[ht]
    \centering
    \begin{minipage}{0.45\textwidth}
        \centering
        \includegraphics[width=1\textwidth]{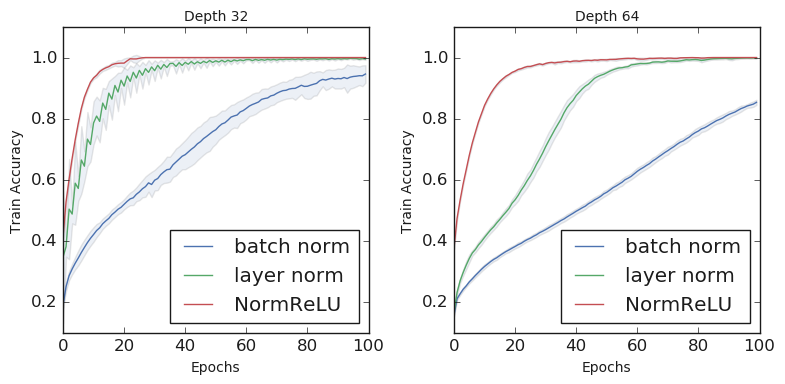} 
    \end{minipage}
    \begin{minipage}{0.45\textwidth}
        \centering
        \includegraphics[width=1\textwidth]{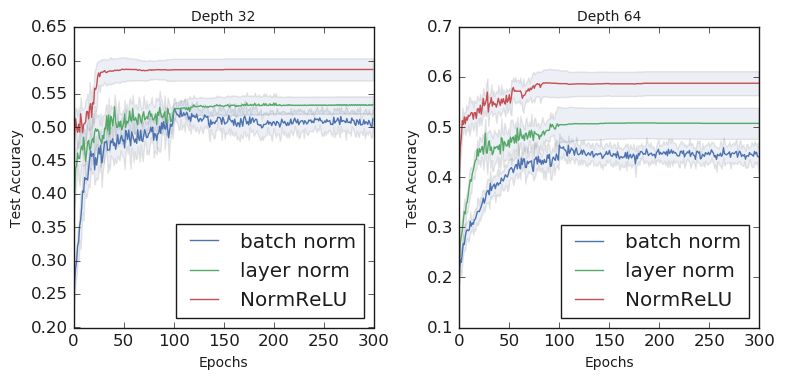} 
    \end{minipage}
    \captionsetup{indention=0.2cm}
   \caption{\label{fig:fnn}
Train and test accuracy vs. the number of epochs. The left plots shows the optimization performance of the \normrelu as compared to standard methods such as batch normalization and layer normalization. The plots are obtained by training depth 32 fully connected networks with 3000 hidden units in each layer. Similarly, the right plot compares the test accuracy of the three methods.} 
\end{figure}

\paragraph{\normrelu as a replacement of batch normalization for other architectures.} 
We train deep convolutional networks with the ResNet architecture \cite{he2016deep} 
using the standard practice of using batch normalization with skip connections and also by replacing batch normalization with \normrelu. We do not use layer normalization since that is not the standard way to train CNNs. Figure \ref{fig:cnn} shows the train and test accuracies obtained on both network architectures. As can be seen, the use of \normrelu is indeed competitive with batch normalization achieving similar test accuracies and slightly outperforming batch normalization at depth 110.
\begin{figure}[ht]
    \centering
    \begin{minipage}{0.45\textwidth}
        \centering
        \includegraphics[width=1\textwidth]{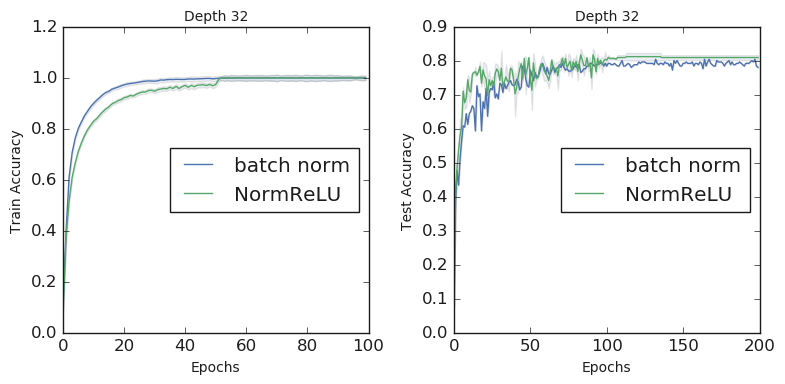} 
    \end{minipage}
    \begin{minipage}{0.45\textwidth}
        \centering
        \includegraphics[width=1\textwidth]{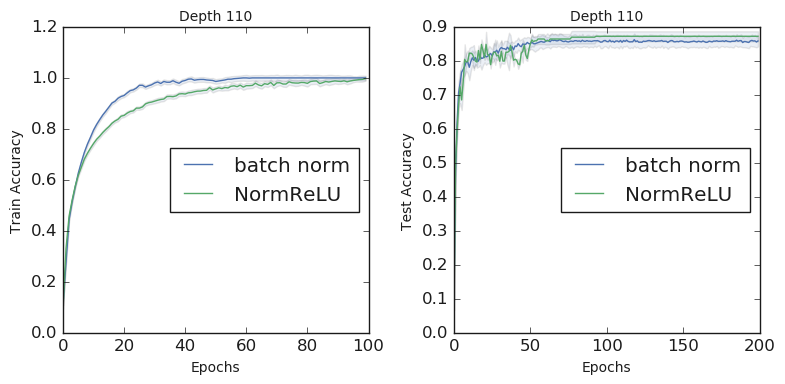} 
    \end{minipage}
    \captionsetup{indention=0.2cm}
   \caption{\label{fig:cnn}Train and test accuracy vs. the number of epochs. The plots are obtained by training a depth 32 and a depth 110 ResNet architecture.} 
\end{figure}

\paragraph{Comparison with Fixup initialization.} 
\begin{wrapfigure}{r}{0.3\textwidth}
\centering
\includegraphics[width=.23\textwidth]{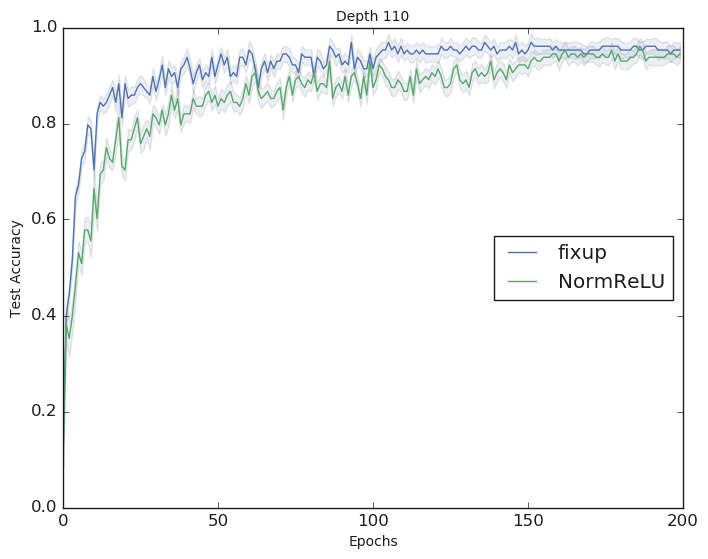}
\captionsetup{indention=0.2cm}
\caption{\label{fig:fixup} \normrelu vs Fixup on CIFAR-10 with a
 ResNet 110 architecture.} 
\end{wrapfigure}
We show that on the CIFAR-10 dataset, using \normrelu with standard initialization we can achieve comparable results to those obtained using standard ReLU with the Fixup initialization method of \citet{zhang2019fixup}. We turn on data augmentation and train the same 110 depth architecture as in \citet{zhang2019fixup}. 
We replace Fixup initialization with standard Gaussian initialization where the kernel weights are initialized with a mean zero Gaussian and a variance of $1/(\text{kernel\_size}^2 \cdot \text{num\_out\_channels})$, and replace the standard ReLU activation with \normrelu. Figure \ref{fig:fixup} shows the test accuracies achieved by both the methods. Note that training with \normrelu achieves a similar accuracy as with Fixup initialization.


\section{Conclusions and Future Directions}
\label{sec:conclusions}

In this work we further elaborated the role of depth in training of modern neural networks by showing that the conditioning of the input data improves exponentially with
depth at random initialization. It would be interesting to further rigorously understand how the conditioning behaves during the course of training. An excellent open question is to 
analyze more realistic parameter regimes~(low width in particular). 
While it is reasonably straightforward to extend our analysis to architectures such as Convolutional Neural Networks and ResNets, extending it to architectures such as Recurrent Networks and Transformers would be quite interesting. 


\bibliographystyle{plainnat}
\bibliography{notes}

\appendix

\newpage
\section{Related Work}
\label{app:related-app}
\paragraph{Representational Benefits of Depth.} Analogous to depth hierarchy theorems in circuit complexity, many recent works have aimed to characterize the representational power of deep neural networks when compared to their shallow counterparts. The work of \citet{delalleau2011shallow} studies sum-product networks and constructs examples of functions that can be efficiently represented by depth $4$ or higher networks and require exponentially many neurons for representation with depth one networks. The works of \citet{martens2014expressive} and \citet{kane2016super} study networks of linear threshold gates and provide similar separation results. \citet{eldan-shamir} show that for many popular activations such as sigmoid, ReLU etc. there are simple functions that can be computed by depth $3$ feed forward networks but require exponentially (in the input dimensionality) many neurons to represent using two layer feed forward networks. \citet{telgarsky} generalizes this to construct, for any integer $k$, a family of functions that can be approximated by $\Theta(k^3)$ layers and $\Theta(k^3)$ size and require exponential in $k$ neurons to represent with $O(k)$ depth.

\paragraph{Optimization Benefits of Depth.} While the benefits of depth are well understood in terms of the representation power using a small number of neurons, the question of whether increasing depth helps with optimization is currently poorly understood. The recent work of \citet{arora-cohen-hazan} aims to understand this question for the special case of linear neural networks. For the case of $\ell_p$ regression, they show that gradient descent updates on a depth $2$ linear network correspond to accelerated gradient descent type updates on the original weight vector. Similarly, they derive the form of the weight updates for a general over parameterized deep linear neural network and show that these updates can be viewed as performing gradient descent on the original network but with a preconditioning operation applied to the gradient at each step. Empirically this leads to faster convergence. The works of \citet{bartlett2019gradient} and \citet{arora2018convergence} study the convergence of gradient descent on linear regression problems when solved via an over parameterized deep linear network. These works establish that under suitable assumptions on the initialization, gradient descent on the over parameterized deep linear networks enjoys the same rate of convergence as performing linear regression in the original parameter space which is a smooth and strongly convex problem. 

In a similar vein, the recent work of \citet{arora2019implicit} analyzes over parameterized deep linear networks for solving matrix factorization, and shows that the solution to the gradient flow equations approaches the minimum nuclear norm solution at a rate that increases with the depth of the network. The recent work of \citet{malach2019deeper} studies depth separation between shallow and deeper networks over distributions that have a certain fractal structure. In certain regimes of the parameters of the distribution the authors show that, surprisingly, the stronger the depth separation is, the harder it becomes to learn the distribution via a deep network using gradient based algorithms.

\paragraph{Optimization of Neural Networks via Gradient Descent} In recent years there has been a large body of work in analyzing the convergence of gradient descent and stochastic gradient descent~(SGD) on over parameterized neural networks. The work of \citet{andoni2014learning} shows that depth one neural networks with quadratic activations can efficiently represent low degree polynomials and performing gradient descent on the network starting with random initialization can efficiently learn such classes. The work of \citet{li2017convergence} shows convergence of gradient descent on the population loss and under Gaussian input distribution, of a two layer feed forward network with relu activations and the identity mapping mimicking the ResNet architecture. Under similar assumptions the work of \citet{soltanolkotabi2018theoretical} analyzes SGD for two layer neural networks with quadratic activations. The work of \citet{li-liang} extends these results to more realistic data distributions. 

Building upon the work of \citet{daniely2016toward}, \citet{daniely2017sgd} shows that SGD when run on over parameterized neural networks achieves at most $\epsilon$ excess loss~(on the training set) over the best predictor in the conjugate kernel class at the rate that depends on $1/\epsilon^2$ and $M$, the norm of the best predictor. This result is extended in the work of \citet{du-zhai-poczos-singh} showing that by running SGD on a randomly initialized two layer over parameterized networks with relu activations, one can get $\epsilon$ loss on the training data at the rate that depends on $\log(1/\epsilon)$ and the smallest eigenvalue of a certain kernel matrix. While the authors show that this eigenvalue is positive, no explicit bound is provided. These results are extended to higher depth in \citep{du2018gradient} at the expense of an exponential dependence on the depth on the amount of over parameterization needed. In \citep{allen2018convergence} the authors provide an alternate analysis under the weaker \pref{ass:separation} and at the same time obtain convergence rates that depend on $\log(1/\epsilon)$ and only polynomially in the depth of the network. A few recent papers \citep{oymak,zou-gu,SuYang} provide an improved analysis with better dependence on the parameters. We would like to point out that all the above works fail to explain the optimization benefits of depth, and in fact the resulting bounds degrade as the network gets deeper. 

The work of \citet{ntk} proposed the Neural Tangent Kernel~(NTK) that is associated with a randomly initialized neural network in the infinite width regime. The authors show that in this regime performing gradient descent on the parameters of the network is equivalent to kernel regression using the NTK. The work of \citet{lee-wide-neural-2019} and \citet{yang} generalizes this result and the recent work of \citet{arora2019exact} provides a non-asymptotic analysis and an algorithm for exact computation of the NTK for feed forward and convolutional neural networks. There have also been works analyzing the mean field dynamics of SGD on infinite width neural networks \citep{mei2018mean, chizat2018global, rotskoff2018parameters, sirignano2018mean} as well as works designing provable learning algorithms for shallow neural networks under certain assumptions \citep{arora2016understanding, ge2017learning, goel2017learning, ge2018learning, goel2018learning, bakshi2018learning, vempala2018gradient}. Recent works have also explored the question of providing sample complexity based separation between training via the NTK vs. training all the layers \citep{wei2019regularization, allen2020backward}.
 
\paragraph{SQ Learnability of Neural Networks.} Several recent works have studied the {statistical query}~(SQ) framework of \citet{kearns1998efficient} to provide lower bounds on the number of queries needed to learn neural networks with a certain structure \citep{song2017complexity, vempala2018gradient, panigrahy-random-nets}. The closest to us is the recent work of \citet{panigrahy-random-nets} that shows that learning a function that is a randomly initialized deep neural network with sign activations requires exponential in depth many statistical queries. A crucial part of their analysis requires showing that for randomly initialized neural networks with sign activations, the pairwise (normalized) dot products decrease exponentially fast with depth. Our main result in \pref{thm:main-infinite-width-top-layer} strictly generalizes this result for arbitrary non-linear activations (under mild assumptions) thereby implying exponential SQ lower bounds for networks with arbitrary non linear activations. In particular, we show any algorithm that works in the statistical query framework, and learns (with high probability) a sufficiently deep randomly initialized network with an arbitrary non-linear activation, must necessarily use exponentially (in depth) many queries in the worst case. The only requirement we impose on the non-linear activations is that they satisfy subgaussianity~(see Section~\ref{app:sq-app}), a condition satisfied by popular activations such as relu, sign, and tanh. 

\paragraph{Generalization in Neural Networks.} It has been observed repeatedly that modern deep neural networks have sufficient capacity to perfectly memorize the training data, yet generalize to test data very well (see, e.g., \citep{hardt-interpolation}). This observation flies in the face of conventional statistical learning theory which indicates that such overfitting should lead to poor generalization. Since then there has been a line of work providing generalization bounds for neural networks that depend on compressibility of the network \citep{arora2018stronger}, norm based bounds \citep{neyshabur2015norm, bartlett2017spectrally}, bounds via PAC-bayes analysis \citep{neyshabur2017pac, dziugaite2017computing, nagarajan2019deterministic} and bounds that depend on the distance to initialization \citep{long2019size}. Since randomly initialized neural networks are {interpolating classifiers}, i.e., they achieve zero error on the training set, there have also been recent works (e.g. \citep{belkin-hsu-mitra,belkin-rakhlin-tsybakov,liang-rakhlin,liang-rakhlin-zhai,bartlett2019benign,belkin-hsu-xu,hastie-montanari-rosset-tibshirani,mei-montanari}) that study the generalization phenomenon in the context of specific \textit{interpolating} methods (i.e. methods which perfectly fit the training data) and show how the obtained predictors can generalize well. 


\section{Proof of \pref{lem:bn-ln-invariance}}
\label{app:bnlninvariance}
\newcommand{\BN}{\text{BN}}
\newcommand{\LN}{\text{LN}}
\newcommand{\vecone}{\vec{\mathbf{1}}}

In this section, we prove \pref{lem:bn-ln-invariance}, restated here for convenience:
\begin{lemma}
If the neural network in \eqref{eq:network-def} incorporates batch normalization in each layer, then the network output is the same regardless of whether the activation $\sigma$ is normalized or not. The same holds if instead layer normalization is employed, but on the post-activation outputs rather than pre-activation inputs.
\end{lemma}
\begin{proof}
Let $\sigma$ be an arbitrary activation function, and $\tilde{\sigma}$ be its normalized version. Thus $\tilde{\sigma}(u) = \frac{\sigma(u) - c}{s}$ for some constants $c$ and $s$. 

\paragraph{Invariance with Batch Normalization.} For a batch size $b$, the vanilla Batch Normalization operation $\BN: \mathbb{R}^{b} \rightarrow \mathbb{R}^{b}$ is defined as follows:
\[\BN(v) = \frac{1}{\nu}(v - \mu \vecone),\]
where $\vecone$ is the $d$-dimensional all ones vector, $\mu = \frac{1}{b}v\cdot \vecone$, and $\nu = \sqrt{\frac{1}{b}\|v - \mu \vecone\|_2^2}$.

Now, fix a particular layer in the neural network, and a particular hidden unit in that layer. Let $w$ be the weight vector corresponding to that hidden unit. Thus if $x$ is the pre-activation input to the previous layer, then $w\cdot \frac{1}{\sqrt{m}}\sigma(x)$ is the pre-activation value\footnote{We can handle the input layer of the network by simply setting $\sigma$ to be the identity multiplied by $\sqrt{m}$.} for the hidden unit in question. Suppose the batch size is $b$, and let $x_1, x_2, \ldots, x_b$ be the pre-activation inputs to the previous layer. Define $X = [x_1, x_2, \ldots, x_b]$, and $\sigma(X) = [\sigma(x_1), \sigma(x_2), \ldots, \sigma(x_b)]$. Then $v = w^\top \tfrac{1}{\sqrt{m}}\sigma(X)$ is the vector of pre-activation inputs to the hidden unit for the batch. Now, we claim that
\begin{equation} \label{eq:bn-invariance}
	\BN(w^\top \tfrac{1}{\sqrt{m}}\sigma(X)) = \BN(w^\top \tfrac{1}{\sqrt{m}}\tilde{\sigma}(X)),
\end{equation}
which establishes the desired invariance. Let $\tilde{v} = w^\top \tfrac{1}{\sqrt{m}}\tilde{\sigma}(X)$. Define $\mu = \frac{1}{b} v \cdot \vecone$, $\nu = \sqrt{\frac{1}{b}\|v - \mu \vecone\|_2^2}$, $\tilde{\mu} = \frac{1}{b} \tilde{v} \cdot \vecone$ and $\tilde{\nu} = \sqrt{\frac{1}{b}\|\tilde{v} - \tilde{\mu} \vecone\|_2^2}$. Then by direct calculation we have $\tilde{\mu} = \frac{\mu - c}{s}$, and $\tilde{\nu} = \frac{\nu}{s}$. These facts imply the claimed equality \eqref{eq:bn-invariance}.

Finally, the standard Batch Normalization operation also includes constants $\gamma$ and $\beta$ so that the final output is $\gamma\BN(\cdot) + \beta$. The above analysis immediately implies that 
\[\gamma\BN(w^\top \tfrac{1}{\sqrt{m}}\sigma(X)) + \beta = \gamma\BN(w^\top \tfrac{1}{\sqrt{m}}\tilde{\sigma}(X)) + \beta,\]
establishing invariance for the standard Batch Normalization operation as well.

\paragraph{Invariance with post-activation Layer Normalization.} For a layer with $m$ hidden units, the Layer Normalization operation $\LN: \mathbb{R}^m \rightarrow \mathbb{R}^m$ is defined as follows:
\[\LN(v) = \frac{1}{\nu}(v - \mu\vecone),\]
where $\vecone$ is the $m$-dimensional all ones vector, $\mu = \frac{1}{m}v\cdot \vecone$, and $\nu = \sqrt{\frac{1}{m}\|v - \mu \vecone\|_2^2}$.

Now suppose we apply Layer Normalization to the post-activation outputs (i.e. on $v = \tfrac{1}{\sqrt{m}}\tilde{\sigma}(x)$, for a pre-activation input $x \in \mathbb{R}^m$) the rather than pre-activation inputs. Then, we claim that
\begin{equation} \label{eq:ln-invariance}
	\LN(\tfrac{1}{\sqrt{m}}\sigma(x)) = \LN(\tfrac{1}{\sqrt{m}}\tilde{\sigma}(x)),
\end{equation}
which establishes the desired invariance. Let $\tilde{v} = \tfrac{1}{\sqrt{m}}\tilde{\sigma}(x)$. Define $\mu = \frac{1}{m} v \cdot \vecone$, $\nu = \sqrt{\frac{1}{m}\|v - \mu \vecone\|_2^2}$, $\tilde{\mu} = \frac{1}{m} \tilde{v} \cdot \vecone$ and $\tilde{\nu} = \sqrt{\frac{1}{m}\|\tilde{v} - \tilde{\mu} \vecone\|_2^2}$. Then by direct calculation we have $\tilde{\mu} = \frac{\mu - c}{s}$, and $\tilde{\nu} = \frac{\nu}{s}$. These facts imply the claimed equality \eqref{eq:ln-invariance}.
\end{proof}

\section{Conditioning Analysis}
\label{app:conditioning}

Recall the notion of the dual activation $\dualact$ for the activation $\sigma$:
\begin{definition}
	For $\rho \in [-1, 1]$, define matrix $\Sigma_\rho = \begin{bmatrix} 1 & \rho \\ \rho & 1 \end{bmatrix}$. Define the conjugate activation function $\dualact: [-1, 1] \rightarrow [-1, 1]$ as follows:
    \[ \dualact(\rho) := \E_{(X, X') \sim \normal(0, \Sigma_\rho)}[\sigma(X)\sigma(X')].\]
\end{definition}
The following facts can be found in \cite{daniely2016toward}:
 \begin{enumerate}
    \item Let $x, x' \in \mathbb{R}^d$ such that $\|x\| = \|x'\| = 1$. Then
    \[ \E_{w \sim \normal(0, I_d)}[\sigma(w\cdot x)\sigma(w \cdot x')] = \dualact(x\cdot x').\]

    \item Since $\E_{X \sim \normal(0, 1)}[\sigma^2(X)] = 1$, $\sigma$ is square integrable w.r.t. the Gaussian measure. The (probabilitist's) Hermite polynomials $h_0, h_1, \ldots$ form an orthogonal basis for the Hilbert space of square integrable functions w.r.t. the Gaussian measure, and hence $\sigma$ can be written as $\sigma(u) = \sum_{i=0}^\infty a_i h_i(u)$, where $a_i = \E_{X \sim \normal(0, 1)}[\sigma(X)h_i(X)]$. This expansion is known as the Hermite expansion for $\sigma$. 

    \item We have $\dualact(\rho) = \sum_{i=0}^\infty a_i^2 \rho^i$. 

    \item The normalization \pref{eq:normalization} has the following consequences. Since $\E_{X \sim \mathcal{N}(0, 1)}[\sigma(X)] = 0$, we have $a_0 = 0$, and since $\E_{X \sim \mathcal{N}(0, 1)}[\sigma^2(X)] = 1$ we have $\sum_{i=1}^\infty a_i^2 = 1$.

    \item If $\dot{\sigma}$ denotes the derivative of $\sigma$, then $\hat{\dot{\sigma}} = \dot{\dualact}$.
\end{enumerate}

The above facts imply the following simple bound on the coefficient of non-linearity $\mu$:
\begin{lemma}
\label{lem:mu-bounds}
For any normalized non-linear activation function $\sigma$, we have $0 < \mu \leq 1$.	
\end{lemma}
\begin{proof}
The degree 1 Hermite polynomial is $h_1(u) = u$, so $a_1 = \E_{X \sim \normal(0, 1)}[\sigma(X)X]$. Since $\sigma$ is non-linear, for at least one $i \neq 1$, we have $a_i \neq 0$. This, coupled with the fact that $\sum_{i=1}^\infty a_i^2 = 1$ implies that $a_1 \in (-1, 1)$, which implies that $\mu = 1 - a_1^2 \in (0, 1]$.
\end{proof}

The random initialization of the neural network induces a feature representation of the input vectors at every depth $l$ in the neural network: $\Phi_{\vecW}^{(l)}(x) := \tfrac{1}{\sqrt{m}}\sigma(W_l \tfrac{1}{\sqrt{m}}\sigma(W_{l-1} \cdots \tfrac{1}{\sqrt{m}}\sigma(W_1x)\cdots))$. This feature representation naturally yields a kernel function $\kernel^{(l)}(x, x') := \Phi_{\vecW}^{(l)}(x) \cdot \Phi_{\vecW}^{(l)}(x')$. In particular, after the first layer, the kernel function $\kernel^{(1)}(x, x') = \frac{1}{m}\sigma(W_1x)\cdot \sigma(W_1x')$. The central limit theorem implies that as the width $m$ goes to infinity, this kernel function tends to a deterministic value, viz. its expectation, which is $\E_{w \sim \normal(0, I_d)}[\sigma(w\cdot x)\sigma(w \cdot x')]$, which equals $\dualact(x\cdot x')$ if $x$ and $x'$ are unit vectors. Furthermore, the normalization $\E_{X \sim \mathcal{N}(0, 1)}[\sigma^2(X)] = 1$ implies that the feature representation is itself normalized in the sense for any unit vector $x$, that as $m \rightarrow \infty$, we have $\|\Phi_{\vecW}^{(1)}(x)\|_2^2 = \kernel^{(1)}(x, x) \rightarrow 1$. Applying these observations recursively, we get \pref{lem:kernel-convergence}, which was also proved by \citet{daniely2016toward}.
\begin{lemma} \label{lem:kernel-convergence}
Suppose $\|x\|_2 = \|x'\|_2 = 1$. Then for any depth $l$, as $m \rightarrow \infty$
\[\kernel^{(l)}(x, x') \stackrel{a.s.}{\longrightarrow} \dualact^{(l)}(x \cdot x'),\]
where $\dualact^{(l)}$ denotes the $l$-fold composition of $\dualact$ with itself.
\end{lemma}

The following technical lemma shows how one application of $\dualact$ behaves:
\begin{lemma} \label{lem:one-layer}
    Let $\delta \in [0, 1]$. Then 
    \[|\dualact(-(1-\delta))| \leq \dualact(1-\delta) \leq \begin{cases}
        1 - (1+\frac{\mu}{2})\delta & \text{ if }\ \delta \leq \frac{1}{2} \\
        (1 - \frac{\mu}{2})(1 - \delta) & \text{ if }\ \delta > \frac{1}{2}.
    \end{cases}\]
\end{lemma}
\begin{proof}
    The fact that $|\dualact(-(1-\delta))| \leq \dualact(1-\delta)$ follows from the fact that the power series $\dualact(\rho) = \sum_{i=1}^\infty a_i^2 \rho^i$ has only non-negative coefficients. Next, we have
    \[
        \dualact(1-\delta) =  \sum_{i=1}^\infty a_i^2 (1-\delta)^i \leq a_1^2 (1-\delta) + \sum_{i=2}^\infty a_i^2 (1-\delta)^2 = a_1^2 (1-\delta) + (1-a_1^2)(1-\delta)^2 = (1 - \delta)(1 - \mu\delta).
    \]
    Now if $\delta > \frac{1}{2}$, we have $(1 - \delta)(1 - \mu\delta) \leq (1 - \frac{\mu}{2})(1 - \delta)$. If $\delta \leq \frac{1}{2}$, we have $(1 - \delta)(1 - \mu\delta) = 1 - (1 + \mu)\delta + \mu\delta^2 \leq 1 - (1+\frac{\mu}{2})\delta$.
\end{proof}
Recall the definition of $B_{\nu}(L, \delta)$: for any $\nu \in (0,1]$, $\delta \in (0, 1)$ and a positive integer $L$, let $L_0(\delta) = \max\left\{\left\lceil\tfrac{\log(\frac{1}{2\delta})}{\log(1+\frac{\nu}{2})}\right\rceil, 0\right\} = O\left(\tfrac{\log(1/\delta)}{\nu}\right)$, and define
\[B_{\nu}(L, \delta) := \begin{cases}
    1 - \delta(1+\frac{\nu}{2})^L & \text{ if }\ L \leq L_0(\delta) \\
    \frac{1}{2}(1 - \frac{\nu}{2})^{L-L_0(\delta)} & \text{ if }\ L > L_0(\delta).
\end{cases}\]
Remember, we denote by $B(L, \delta)$, the quantity $B_{\mu}(L, \delta)$.

The following lemma is an immediate consequnce via repeated application of \pref{lem:one-layer}:

\begin{lemma}[\textbf{Correlation decay lemma}] \label{lem:recursive}
    Suppose $|\rho| \leq 1 - \delta$ for some $\delta \in (0, 1]$. Then 
    \[|\dualact^{(L)}(\rho)| \leq B(L, \delta).\]
\end{lemma}

The final technical ingredient we need is the following linear-algebraic lemma which gives a lower bound on the smallest eigenvalue of a matrix obtained by the application of a given function to all entries of another positive definite matrix:
\begin{lemma}[\textbf{Eigenvalue lower bound lemma}]\label{lem:general-eigenvalue-lb}
Let $f: [-1, 1] \rightarrow \mathbb{R}$ be an arbitrary function whose power series $f(\rho) = \sum_{i=0}^\infty a_i \rho^i$ converges everywhere in $[-1, 1]$ and has non-negative coefficients $a_i \geq 0$. Let $K \in \mathbb{R}^{n \times n}$ be a positive definite matrix with $K \succeq \delta I_n$ for some $\delta > 0$, and all diagonal entries equal to $1$. Let $f[K]$ be matrix obtained by entrywise application of $f$. Then we have \[f[K] \succeq (f(1) - f(1-\delta))I_n.\]
\end{lemma}
\begin{proof}
We have $f[K] = \sum_{i=1}^\infty a_i K^{\odot i}$, where $K^{\odot i}$ denotes the $i$-fold Hadamard (i.e. entrywise) product of $K$ with itself. Since all diagonal entries of $K$ equal $1$, we can also write $K^{\odot i}$ as 
\[ K^{\odot i} = (K - \delta I_n)^{\odot i} + (1 - (1-\delta)^i)I_n.\]
By assumption, $K - \delta I_n \succeq 0$. Since the Hadamard product of positive semidefinite matrices is also positive semidefinite, we have $(K - \delta I_n)^{\odot i} \succeq 0$. Thus, $K^{\odot i} \succeq (1 - (1-\delta)^i)I_n$. Thus, we have
\[f[K] = \sum_{i=0}^\infty a_i K^{\odot i} \succeq \sum_{i=0}^\infty a_i(1 - (1-\delta)^i) I_n = (f(1) - f(1-\delta))I_n,\]
as required.
\end{proof}

\subsection{Top Layer Kernel Matrix}
We can now prove \pref{thm:main-infinite-width-top-layer} and \pref{thm:main-infinite-width-top-layer-stronger}: 
which we restate here in a combined form for convenience:
\begin{theorem}
    The following bounds hold:
    \begin{enumerate}
        \item Under \pref{ass:separation}, we have $|\bar{K}_{ij}| \leq B(L, \delta)$ for all $i, j \in [n]$ with $i \neq j$.
        
        \item Under \pref{ass:non-singularity}, we have $\lambda_{\min}(\bar{K}) \geq 1 - B(L, \delta)$.
    \end{enumerate}   
\end{theorem}
\begin{proof}
    Part 1 follows directly from \pref{lem:recursive}.

    As for part 2, \pref{ass:non-singularity} implies that $\underline{K} \succeq \delta I_n$. Since the function $\dualact^{(L)}: [-1, 1] \rightarrow \mathbb{R}$ defines a kernel on the unit sphere, by Schoenberg's theorem \citep{schoenberg}, its power series expansion has only non-negative coefficients, so \pref{lem:general-eigenvalue-lb} applies to $\bar{K} = \dualact^{(L)}[\underline{K}]$, and we have
    \[\bar{K} \succeq (\dualact^{(L)}(1) - \dualact^{(L)}(1-\delta))I_n \succeq (1 - B(L, \delta))I_n,\]
    using \pref{lem:recursive} and the fact that $\dualact^{(L)}(1) = 1$.
\end{proof}

We can now prove \pref{cor:main-top-layer-cn} and \pref{cor:main-top-layer-cn-stronger}, restated here in a combined form for convenience:
\begin{corollary}
The following bounds on $\cn(\bar{K})$ hold:
    \begin{enumerate}
        \item Under \pref{ass:separation}, if $L \geq L_1(\delta) := \left\lceil \frac{\log(n)}{-\log(1 - \frac{\mu}{2})}\right\rceil + L_0(\delta)$, then $\cn(\bar{K}) \leq 1 + 2n(1-\frac{\mu}{2})^{L - L_1(\delta)}$.

        \item Under \pref{ass:non-singularity}, we have $\cn(\bar{K}) \leq 1 + \frac{n}{\delta}(1 + \tfrac{\mu}{2})^{-L}$.
    \end{enumerate}
\end{corollary}
\begin{proof}
\noindent To prove part 1, note that the normalization~\eqref{eq:normalization} implies that $\bar{K}_{ii} = 1$ for all $i \in [n]$. This fact, coupled with \pref{thm:main-infinite-width-top-layer} (part 1) and the Gershgorin circle theorem implies the following bounds on the largest and smallest eigenvalues of $\bar{K}$: we have $\lambda_{\max}(\bar{K}) \leq 1 + (n-1)B(L, \delta)$ and $\lambda_{\min}(\bar{K}) \geq 1 - (n-1)B(L, \delta)$, which implies that $\cn(\bar{K}) \leq \frac{1 + (n-1)B(L, \delta)}{1 - (n-1)B(L, \delta)}$. Since $L \geq L_1(\delta) = \left\lceil \frac{\log(n)}{-\log(1 - \frac{\mu}{2})}\right\rceil + L_0(\delta)$, we have $(n-1)B(L,  \delta) \leq \frac{1}{2}$, and then using the inequality $\frac{1+x}{1-x} \leq 1+4x$ for $x \in [0, \frac{1}{2}]$, the bound on the condition number follows.

As for part 2, using \pref{thm:main-infinite-width-top-layer} (part 2) and the bound $\lambda_{\max}(\bar{K}) \leq 1 + (n-1)B(L, \delta)$, we have $\cn(\bar{K}) \leq \frac{1 + (n-1)B(L, \delta)}{1 - B(L, \delta)}$. Now if $L \leq L_0(\delta)$, using the definition of $B(L, \delta)$, we have
\[\cn(\bar{K}) - 1 \leq \frac{n(1 - \delta(1+\frac{\mu}{2})^L)}{1 - (1 - \delta(1+\tfrac{\mu}{2})^L)} \leq \tfrac{n}{\delta}(1+\tfrac{\mu}{2})^{-L}.\]
If $L > L_0(\delta)$, then we have
\[\cn(\bar{K}) - 1 \leq \frac{\frac{n}{2}(1 - \frac{\mu}{2})^{L-L_0(\delta)}}{1 - \frac{1}{2}(1 - \frac{\mu}{2})^{L-L_0(\delta)}} \leq n(1 - \tfrac{\mu}{2})^{L - L_0(\delta)} \leq n(1 + \tfrac{\mu}{2})^{-L}(1 + \tfrac{\mu}{2})^{L_0(\delta)} \leq \tfrac{n}{\delta}(1+\tfrac{\mu}{2})^{-L},\]
as required.
\end{proof}

\subsection{Neural Tangent Kernel Matrix}
The following formula for the NTK was given by \citet{arora2019exact}: defining $\rho := x_i \cdot x_j$, we have
\begin{equation}
\label{eq:ntk-formula}
\bar{\ntk}_{ij} = \sum_{h=1}^{L+1} \dualact^{(h-1)}(\rho)\left(\prod_{h'=h}^{L} \hat{\dot{\sigma}}(\dualact^{(h')}(\rho)) \right).
\end{equation}

Using this formula, we have the following bound:
\begin{lemma}\label{lem:dualact-dot-ratio}
For any $\rho \in [-1, 1]$, we have $\frac{\hat{\dot{\sigma}}(\rho)}{\hat{\dot{\sigma}}(1)} \leq 1 - \mu(1 - |\rho|)$.
\end{lemma}
\begin{proof}
We have $\hat{\dot{\sigma}}(\rho) = \dot{\dualact}(\rho) = \sum_{i=1}^{\infty} i \cdot a_i^2 \cdot \rho^{i-1}$. This implies that $\hat{\dot{\sigma}}(\rho) \leq \hat{\dot{\sigma}}(|\rho|)$, so it suffices to prove the bound for $\rho \geq 0$. Also note that by definition $\hat{\dot{\sigma}}(\rho)$ is non-negative for all $\rho \geq 0$ as well as an increasing function over $\rho \geq 0$. 
Thus, using the fact that $\sum_{i=2}^\infty a_i^2 = 1 - a_1^2 = \mu$, we have
\[\frac{\hat{\dot{\sigma}}(\rho)}{\hat{\dot{\sigma}}(1)} = \frac{a_1^2 + \sum_{i=2}^{\infty} i \cdot a_i^2 \cdot \rho^i}{a_1^2 + \sum_{i=2}^{\infty} i \cdot a_i^2} \leq \frac{a_1^2 + \sum_{i=2}^{\infty} i \cdot a_i^2 \cdot \rho}{a_1^2 + \sum_{i=2}^{\infty} i \cdot a_i^2} = 1 - \left(\frac{\sum_{i=2}^{\infty} i \cdot a_i^2}{a_1^2 + \sum_{i=2}^{\infty} i \cdot a_i^2}\right)(1-\rho) \leq 1 - \mu(1 - \rho), \]
as required.
\end{proof}

We can now prove \pref{thm:main-infinite-width-ntk} and \pref{thm:main-infinite-width-ntk-stronger}, which we restate here in a combined form for convenience:
\begin{theorem}
    The diagonal entries of $\bar{\ntk}$ are all equal. Furthermore, the following bounds hold if $L \geq 2L_0(\delta)$:
    \begin{enumerate}
        \item Under \pref{ass:separation}, we have $|\bar{\ntk}_{ij}| \leq 2B(\nicefrac{L}{2},\delta) \cdot \bar{\ntk}_{11}$ for all $i, j \in [n]$ with $i \neq j$.
        \item Under \pref{ass:non-singularity}, we have $\lambda_{\min}(K) \geq \left(1 -  2B(\nicefrac{L}{2},\delta)\right)\bar{\ntk}_{11}$.
    \end{enumerate}
\end{theorem}
\begin{proof}
First, we show that all diagonal values of $\bar{\ntk}$ are equal. For every $i$, we have $x_i \cdot x_i = 1$, and since $\dualact^{(h)}(1) = 1$ for any $h$, we have from \pref{eq:ntk-formula},
\[\bar{\ntk}_{ii} = \sum_{h=1}^{L+1} \dualact^{(h-1)}(1)\left(\prod_{h'=h}^{L} \hat{\dot{\sigma}}(\dualact^{(h')}(1)) \right) = \sum_{h=1}^{L+1} \left(\prod_{h'=h}^{L} \hat{\dot{\sigma}}(1) \right) = \frac{\hat{\dot{\sigma}}(1)^{L+1} - 1}{\hat{\dot{\sigma}}(1) - 1},\] 
which is a fixed constant.

To prove part 1, let $\rho := x_i \cdot x_j$. It is easy to show (say, via the Hermite expansion of $\sigma$) that $\hat{\dot{\sigma}}(1) > 0$. Thus, we have
\begin{gather*}
\frac{\bar{\ntk}_{ij}}{\bar{\ntk}_{11}} = \frac{\sum_{h=1}^{L+1} \dualact^{(h-1)}(\rho)\left(\prod_{h'=h}^{L} \hat{\dot{\sigma}}(\dualact^{(h')}(\rho)) \right)}{\sum_{h=1}^{L+1} \left(\prod_{h'=h}^{L} \hat{\dot{\sigma}}(1) \right)} \leq \max_{h \in [L+1]} \dualact^{(h-1)}(\rho) \cdot \prod_{h'=h}^{L}\frac{\hat{\dot{\sigma}}(\dualact^{(h')}(\rho))}{\hat{\dot{\sigma}}(1)}\\
\leq \max_{h \in [L+1]} |\dualact^{(h-1)}(\rho)| \cdot \prod_{h'=h}^{L}(1-\mu(1 - |\dualact^{(h')}(\rho)|)) \leq \max_{h \in [L+1]} B(h-1, \delta) \cdot \prod_{h'=h}^{L}(1-\mu(1 - B(h', \delta))),
\end{gather*}
where the penultimate inequality follows \pref{lem:dualact-dot-ratio} and the final one from \pref{lem:recursive}. We now show that since $L \geq 2L_0(\delta)$, for any any $h \in [L+1]$, we have 
\[B(h-1, \delta) \cdot \prod_{h'=h}^{L}(1-\mu(1 - B(h', \delta))) \leq 2B(\nicefrac{L}{2},\delta),\] which gives the bound of part 1. We do this in two cases: if $h-1 \geq \nicefrac{L}{2}$, then $B(h-1, \delta) \leq B(\nicefrac{L}{2},\delta)$, which gives the required bound since all terms in the product are at most $1$. Otherwise, if $h-1 < \nicefrac{L}{2}$, then there are at least $L/2 - L_0(\delta)$ values of $h'$ in $\{h, h+1, \ldots, L\}$ which are larger than $L_0(\delta)$, and for these values of $h'$, we have $B(h', \delta) \leq \frac{1}{2}$, so $1-\mu(1 - B(h', \delta))) \leq 1 - \frac{\mu}{2}$. The product of these terms is therefore at most $(1 - \frac{\mu}{2})^{\nicefrac{L}{2} - L_0(\delta)} = 2B(\nicefrac{L}{2},\delta)$, which gives the required bound in this case.

To prove part 2, define $f:[-1, 1] \rightarrow \mathbb{R}$ as $f(\rho) = \sum_{h=1}^{L+1} \dualact^{(h-1)}(\rho)\left(\prod_{h'=h}^{L} \hat{\dot{\sigma}}(\dualact^{(h')}(\rho)) \right)$. Equation \pref{eq:ntk-formula} shows that this defines a kernel on the unit sphere, and so by Schoenberg's theorem \citep{schoenberg}, its power series expansion has only non-negative coefficients. Thus, applying \pref{lem:general-eigenvalue-lb} to $\bar{\ntk} = f[\underline{K}]$, we conclude that 
\[ \bar{\ntk} \succeq (f(1) - f(1 - \delta))I_n \succeq (1 - 2B(\nicefrac{L}{2}, \delta))f(1)I_n,\]
using the calculations in part 1. Since $f(1) = \bar{\ntk}_{11}$, the bound of part 2 follows.
\end{proof}

Finally, we prove \pref{cor:main-infinite-width-ntk-cn} and \pref{cor:main-infinite-width-ntk-cn-stronger}, restated here in a combined form:
\begin{corollary}
The following bounds on the condition number $\cn(\bar{\ntk})$ hold:
    \begin{enumerate}
        \item Under \pref{ass:separation}, if $L \geq L_2(\delta) := \left\lceil \frac{2\log(2n)}{-\log(1 - \frac{\mu}{2})}\right\rceil + 2L_0(\delta)$, then $\cn(\bar{\ntk}) \leq 1+4n(1~-~\tfrac{\mu}{2})^{\nicefrac{L}{2} - L_2(\delta)}$.
        \item Under \pref{ass:non-singularity}, if $L \geq 4L_0(\delta)$, then $\cn(\bar{\ntk}) \leq 1+ \tfrac{2n}{\delta}(1+\tfrac{\mu}{2})^{-\nicefrac{L}{2}}$.
    \end{enumerate} 
\end{corollary}
\begin{proof}
\noindent To prove part 1, note that \pref{thm:main-infinite-width-ntk} (part 1) and the Gershgorin circle theorem implies the following bounds on the largest and smallest eigenvalues of $\bar{\ntk}$: we have $\lambda_{\max}(\bar{\ntk}) \leq (1 + 2(n-1)B(\nicefrac{L}{2}, \delta))\bar{\ntk}_{11}$ and $\lambda_{\min}(\bar{\ntk}) \geq (1 - 2(n-1)B(\nicefrac{L}{2}, \delta))\bar{\ntk}_{11}$, which implies that $\cn(\bar{\ntk}) \leq \frac{1 + 2(n-1)B(\nicefrac{L}{2}, \delta)}{1 - 2(n-1)B(\nicefrac{L}{2}, \delta)}$. Since $L \geq L_2(\delta) = \left\lceil \frac{2\log(2n)}{-\log(1 - \frac{\mu}{2})}\right\rceil + 2L_0(\delta)$, we have $2(n-1)B(\nicefrac{L}{2},  \delta) \leq \frac{1}{2}$, and then using the inequality $\frac{1+x}{1-x} \leq 1+4x$ for $x \in [0, \frac{1}{2}]$, the bound on the condition number follows.

As for part 2, using \pref{thm:main-infinite-width-ntk} (part 2) and the bound $\lambda_{\max}(\bar{\ntk}) \leq (1 + 2(n-1)B(\nicefrac{L}{2}, \delta))\bar{\ntk}_{11}$, we have $\cn(\bar{\ntk}) \leq \frac{1 + 2(n-1)B(\nicefrac{L}{2}, \delta)}{1 - 2B(\nicefrac{L}{2}, \delta)}$. Thus,
\[\cn(\bar{\ntk}) - 1 \leq \frac{n(1 - \frac{\mu}{2})^{\nicefrac{L}{2}-L_0(\delta)}}{1 - (1 - \frac{\mu}{2})^{\nicefrac{L}{2}-L_0(\delta)}} \leq 2n(1 - \tfrac{\mu}{2})^{\nicefrac{L}{2} - L_0(\delta)} \leq 2n(1 + \tfrac{\mu}{2})^{-\nicefrac{L}{2}}(1 + \tfrac{\mu}{2})^{L_0(\delta)} \leq \tfrac{2n}{\delta}(1+\tfrac{\mu}{2})^{-\nicefrac{L}{2}},\]
the second inequality follows since $\nicefrac{L}{2} - L_0(\delta) \geq L_0(\delta)$, and so $(1 - \frac{\mu}{2})^{\nicefrac{L}{2} - L_0(\delta)} \leq (1 - \frac{\mu}{2})^{L_0(\delta)} \leq \frac{1}{2}$.
\end{proof}

\subsection{Evolution of dot products for uncentered activations}
\label{app:unnormalized-analysis}

In this section, we generalize the analysis in the beginning of \pref{app:conditioning} to activations that need not be normalized \eqref{eq:normalization}. Specifically, we only assume the following condition on the activation $\sigma$:
\begin{equation} \label{eq:square-normalization}
  \E_{X \sim \normal(0, 1)}[\sigma^2(X)] = 1,  
\end{equation}
i.e., we allow $\E_{X \sim \normal(0, 1)}[\sigma(X)]$ to be non-zero. In this case, we will show that there is a certain value $\rhofx \in [0, 1]$ such that dot products of input representations at the top layer converge to $\rhofx$ as the depth increases. Furthermore, when $\rhofx < 1$, this convergence is at an \emph{exponential} rate governed by the following notion of the coefficient of non-affinity which accounts for the non-zero expectation of $\sigma$ under standard Gaussian inputs:
\begin{definition}
\label{def:coefficient-non-affinity}
    The {\em coefficient of non-affinity} of the activation function $\sigma$ is defined to be $\mutil := 1 - \left(\E_{X \sim \normal(0, 1)}[\sigma(X)]\right)^2 - \left(\E_{X \sim \normal(0, 1)}[X \sigma(X)]\right)^2$.
\end{definition}
Clearly, for normalized activations, the two coefficients coincide, i.e. $\mu = \mutil$. Let $\sigma(u) = \sum_{i=0}^\infty a_i h_i(u)$ as usual, and then the dual activation is $\dualact(\rho) = \sum_{i=0}^\infty a_i^2 \rho^i$. The normalization \eqref{eq:square-normalization} implies that $\sum_{i=0}^\infty a_i^2 = 1$. Furthermore, since $h_0(u) \equiv 1$ and $h_1(u) = u$, we have 
\[a_0 = \E_{X \sim \normal(0, 1)}[\sigma(X)] \quad \text{and} \quad a_1 = \E_{X \sim \normal(0, 1)}[X\sigma(X)],\] 
and so $\mutil = 1 - a_0^2 - a_1^2$. We call an activation function $\sigma(u)$ ``non-affine'' if it cannot be written as $a + bu$ for constants $a$ and $b$. It is easy to see from the Hermite expansion of $\sigma$ that this is equivalent to saying that for some $i > 1$, $a_i \neq 0$. We then have the following analogue of \pref{lem:mu-bounds}:
\begin{lemma}
\label{lem:mutil-bounds}
For any non-affine activation function $\sigma$, we have $0 < \mutil \leq 1$.   
\end{lemma}
\begin{proof}
Since $\sigma$ is non-affine, for at least one $i > 1$, we have $a_i \neq 0$. This, coupled with the fact that $\sum_{i=0}^\infty a_i^2 = 1$ implies that $\mutil = 1 - a_0^2 - a_1^2 \in (0, 1]$.
\end{proof}
The following property of $\dualact$ immediately follows from the formula $\dualact(\rho) = \sum_{i=0}^\infty a_i^2 \rho^i$:
\begin{lemma} \label{lem:dualact-convexity}
The dual activation $\dualact$ is non-decreasing and convex on $[0, 1]$.
\end{lemma}
Note that we always have $\dualact(1) = 1$. We now have the following important property regarding the other fixed points of $\dualact$:
\begin{lemma} \label{lem:fixed-point}
    Let $\rhofx \in [0, 1]$ such that $\dualact(\rhofx) = \rhofx$. Then for any $\lambda \in [0, 1]$, if $\rho_\lambda := \lambda + (1-\lambda)\rhofx$, then
    \[\rho_\lambda - \dualact(\rho_\lambda) \geq \mutil\lambda(1-\lambda)(1-\rhofx)^2. \]
    Thus, if $\sigma$ is non-affine, then there can be at most one fixed point of $\dualact$ in $[0, 1)$, i.e. a point $\rhofx \in [0, 1)$ such that $\dualact(\rhofx) = \rhofx$.
\end{lemma}
\begin{proof}
First, since $\rhofx$ and $1$ are both fixed points of $\dualact$, we have
\begin{equation} \label{eq:rho_lambda}
  \rho_\lambda =  \lambda  + (1-\lambda)\rhofx = \lambda \dualact(1) + (1-\lambda)\dualact(\rhofx) = \sum_{i=0}^\infty a_i^2 (\lambda + (1-\lambda)\rhofx^i),  
\end{equation}
whereas
\begin{equation} \label{eq:dualact-rho_lambda}
  \dualact(\rho_\lambda) = \sum_{i=0}^\infty a_i^2 (\lambda + (1-\lambda)\rhofx)^i.  
\end{equation}
Now, we show that the sequence $(\lambda + (1-\lambda)\rhofx^i) - (\lambda + (1-\lambda)\rhofx)^i$ for $i = 0, 1, 2, \ldots$  is non-decreasing. This is because by direct calculation,
\begin{align*}
&[(\lambda + (1-\lambda)\rhofx^{i+1}) - (\lambda + (1-\lambda)\rhofx)^{i+1}] - [(\lambda + (1-\lambda)\rhofx^i) - (\lambda + (1-\lambda)\rhofx)^i] \\
&= (1-\lambda)(1-\rhofx)((\lambda + (1-\lambda)\rhofx)^i - \rhofx^i) \\
&\geq 0,    
\end{align*}
since $\rhofx \leq \lambda + (1-\lambda)\rhofx$. Thus, for all $i \geq 2$, we have
\begin{equation} \label{eq:gap-bound}
(\lambda + (1-\lambda)\rhofx^i) - (\lambda + (1-\lambda)\rhofx)^i \geq (\lambda + (1-\lambda)\rhofx^2) - (\lambda + (1-\lambda)\rhofx)^2 = \lambda(1-\lambda)(1-\rhofx)^2.    
\end{equation}
Hence, using \eqref{eq:rho_lambda} and \eqref{eq:dualact-rho_lambda}, we have
\begin{align*}
    \rho_\lambda - \dualact(\rho_\lambda) &= \sum_{i=0}^\infty a_i^2 [(\lambda + (1-\lambda)\rhofx^i) - (\lambda + (1-\lambda)\rhofx)^i] \\
    &= \sum_{i=2}^\infty a_i^2 [(\lambda + (1-\lambda)\rhofx^i) - (\lambda + (1-\lambda)\rhofx)^i] \\
    &\geq \sum_{i=2}^\infty a_i^2 \lambda(1-\lambda)(1-\rhofx)^2 & \text{(Using \eqref{eq:gap-bound})} \\
    &= \mutil\lambda(1-\lambda)(1-\rhofx)^2,
\end{align*}
as required.
\end{proof}
Since $\dualact(1) = 1$, $1$ is a fixed point of $\dualact$. The above lemma shows that for non-affine $\sigma$, $\dualact$ can have at most one more fixed point in $[0, 1]$. This fact also gives us the following useful consequence:
\begin{lemma} \label{lem:dualact-rotfx-derivative}
Let $\sigma$ be a non-affine activation, and let $\rhofx \in [0, 1]$ be the smallest fixed point of $\dualact$. Then, if $\rhofx < 1$, then $\dot\dualact(\rhofx) < 1$. If $\rhofx = 1$, then $\dot\dualact(1) \leq 1$.
\end{lemma}
\begin{proof}
Suppose $\rhofx < 1$. If $\dot\dualact(\rhofx) \geq 1$, then since $\dualact(1) = 1$, the mean value theorem and the convexity of $\dualact$ imply that $\dot\dualact(\rhofx) = 1$, which implies that all $\rho \in [\rhofx, 1]$ are fixed points of $\dualact$, which is a contradiction by \pref{lem:dualact-convexity}. 

Next, if $\rhofx = 1$, then there is no fixed point of $\dualact$ in $[0, 1)$. If $\dot\dualact(1) > 1$, then there exists a value $\rho < 1$ such that $\dualact(\rho) < \rho$. Since $\dualact(0) = a_0^2 \geq 0$, and $\dualact$ is continuous, we conclude that there exists a value $\rho' \in [0, \rho)$ such that $\dualact(\rho') = \rho'$, a contradiction.
\end{proof}
We can now prove an analogue of \pref{lem:one-layer}:
\begin{lemma}\label{lem:one-layer-unnormalized}
Suppose $\sigma$ is a non-affine activation. Let $\rhofx$ be the smallest fixed point of $\dualact$ in $[0, 1]$, and $\delta \in (0, 1]$. Then, the following bounds hold:
\begin{itemize}
    \item If $\rho \in [\frac{1+\rhofx}{2}, 1]$, then
    \[\dualact(\rho) \leq \rho - \frac{\mutil}{2}(1-\rhofx)(1-\rho).\]

    \item If $\rho \in [\rhofx, \frac{1+\rhofx}{2})$, then
    \[|\dualact(\rho) - \rhofx| \leq (1 - \tfrac{\mutil}{2}(1-\rhofx))|\rho - \rhofx|.\]

    \item If $\rho \in [0, \rhofx)$, then
    \[|\dualact(\rho) - \rhofx| \leq \dot\dualact(\rhofx)|\rho - \rhofx|.\]

    \item If $\rho \in [-1, 0)$, then $|\dualact(\rho)| \leq \dualact(-\rho)$, and 
    \[\dualact(\rho) \geq \rho + \dualact(0).\]
\end{itemize}    
\end{lemma}
\begin{proof}
We analyze the four cases separately:
\paragraph{Case 1: $\rho \in [\frac{1+\rhofx}{2}, 1]$.} In this case, writing $\rho$ as a convex combination of $\frac{1+\rhofx}{2}$ and $1$ and using the convexity of $\dualact$ in $[0, 1]$, we have
\[\dualact(\rho) \leq \tfrac{2-2\rho}{1-\rhofx}\dualact(\tfrac{1+\rhofx}{2}) + \tfrac{2\rho-1-\rhofx}{1-\rhofx}\dualact(1) \leq \tfrac{2-2\rho}{1-\rhofx}(\tfrac{1+\rhofx}{2} - \tfrac{\mutil}{4}(1-\rhofx)^2) + \tfrac{2\rho-1-\rhofx}{1-\rhofx} = \rho - \tfrac{\mutil}{2}(1-\rhofx)(1-\rho).\]
The second inequality above follows from \pref{lem:fixed-point}. 

\paragraph{Case 2: $\rho \in [\rhofx, \frac{1+\rhofx}{2})$.} Again, writing $\rho$ as a convex combination of $\rhofx$ and $\frac{1+\rhofx}{2}$ and using the convexity of $\dualact$ in $[0, 1]$, we have
\[\dualact(\rho) \leq \tfrac{1 + \rhofx - 2\rho}{1-\rhofx}\dualact(\rhofx) + \tfrac{2\rho - 2\rhofx}{1-\rhofx}\dualact(\tfrac{1+\rhofx}{2}) \leq \tfrac{1 + \rhofx - 2\rho}{1-\rhofx}\cdot \rhofx + \tfrac{2\rho - 2\rhofx}{1-\rhofx}(\tfrac{1+\rhofx}{2} - \tfrac{\mutil}{4}(1-\rhofx)^2) = \rho - \tfrac{\mutil}{2}(1-\rhofx)(\rho - \rhofx).\]
The second inequality above follows from \pref{lem:fixed-point}. Finally, since $\dualact$ is non-decreasing, and $\rho \geq \rhofx$, we must have $\dualact(\rho) \geq \dualact(\rhofx) = \rhofx$. Hence, we have
\[|\dualact(\rho) - \rhofx| = \dualact(\rho) - \rhofx \leq \rho - \tfrac{\mutil}{2}(1-\rhofx)(\rho - \rhofx) - \rhofx = (1 - \tfrac{\mutil}{2}(1-\rhofx))|\rho - \rhofx|,\]
as required.

\paragraph{Case 3: $\rho \in [0, \rhofx)$.} In this case, the convexity of $\dualact$ in $[0, 1]$ implies that
\[\dualact(\rho) \geq \dualact(\rhofx) + \dot\dualact(\rhofx)(\rho - \rhofx) = \rhofx + \dot\dualact(\rhofx)(\rho - \rhofx),\]
or in other words,
\[\rhofx - \dualact(\rho) \leq \dot\dualact(\rhofx)(\rhofx - \rho).\]
Since $\dualact$ is non-decreasing, and $\rho < \rhofx$, we must have $\dualact(\rho) \leq \dualact(\rhofx) = \rhofx$. Hence, we have
\[|\dualact(\rho) - \rhofx| = \rhofx - \dualact(\rho) \leq \dot\dualact(\rhofx)(\rhofx - \rho) = \dot\dualact(\rhofx)|\rho -\rhofx|,\]
as required.

\paragraph{Case 4: $\rho \in [-1, 0)$.} The bound $|\dualact(\rho)| \leq \dualact(-\rho)$ is obvious from the fact that $\dualact(\rho) = \sum_{i=0}^\infty a_i^2 \rho^i$. Next, we have
\[\dualact(\rho) - \rho = \sum_{i=0}^\infty a_i^2 \rho^i - \sum_{i=0}^\infty a_i^2 \rho = a_0^2(1 - \rho) + \sum_{i=1}^\infty a_i^2(\rho^i - \rho) \geq a_0^2 = \dualact(0).\]
The last inequality follows because $1-\rho \geq 0$ and $\rho^i - \rho \geq 0$.
\end{proof}
Finally, we now have the following generalization of \pref{lem:recursive} showing exponentially fast convergence of dot products to the smallest fixed point of $\dualact$ in $[0, 1]$. We focus on the case when $\dualact(0) > 0$, since the case when $\dualact(0) = 0$ exactly corresponds to normalized activations, which we have already analyzed.
\begin{lemma}[Correlation convergence lemma]
\label{lem:corr-convergence}
Suppose $\sigma$ is a non-affine activation with $\dualact(0) > 0$. Let $\rhofx$ be the smallest fixed point of $\dualact$ in $[0, 1]$, and $\delta \in (0, 1]$. If $\rhofx = 1$, then assume further that $\dot\dualact(1) < 1$. Then, after $L \geq L_0 := \max\left\{\left\lceil\frac{\log(\frac{1-\rhofx}{2\delta})}{\log(1+\frac{\mutil(1-\rhofx)}{2})}\right\rceil, \left\lceil\frac{1}{\dualact(0)}\right\rceil \right\}$ layers, for any $\rho$ such that $|\rho| \leq 1-\delta$, we have
\[|\dualact^{(L)}(\rho) - \rhofx| \leq \max\left\{1-\tfrac{\mutil(1-\rhofx)}{2}, \dot\dualact(\rhofx)\right\}^{L-L_0}(\tfrac{1+\rhofx}{2}).\]
Finally, if $\rhofx = 1$ and $\dot\dualact(1) = 1$, then 
\end{lemma}
\begin{proof}
If $\rho \in (\frac{1+\rhofx}{2}, 1-\delta]$, then using case 1 of \pref{lem:one-layer-unnormalized}, after $L_1 := \left\lceil\frac{\log(\frac{1-\rhofx}{2\delta})}{\log(1+\tfrac{\mutil(1-\rhofx)}{2})}\right\rceil$ layers, we have $\dualact^{(L_1)}(\rho) \leq \frac{1+\rhofx}{2}$. Similarly, if $\rho < 0$, then using case 4 of \pref{lem:one-layer-unnormalized}, after $L_2 := \left\lceil\frac{1}{\dualact(0)}\right\rceil$ layers, we have $\dualact^{(L_2)}(\rho) \geq 0$. Finally, when $\rho \in [0, \frac{1+\rhofx}{2}]$, then using cases 2 and 3 of \pref{lem:one-layer-unnormalized}, we conclude that 
\[|\dualact(\rho) - \rhofx| \leq \max\left\{1-\tfrac{\mutil(1-\rhofx)}{2}, \dot\dualact(\rhofx)\right\}|\rho - \rhofx|.\] The statement of the lemma follows from these observations.
\end{proof}
\pref{thm:unnormalized-main-infinite-width-top-layer} now follows immediately from \pref{lem:corr-convergence} using the fact that $\dot\dualact(\rhofx) < 1$ by \pref{lem:dualact-rotfx-derivative}.

The only setting not covered by \pref{lem:corr-convergence} is when $\rhofx = 1$ and $\dot\dualact(1) = 1$. This case is handled separately in the lemma below:
\begin{lemma}
Suppose $\sigma$ is a non-affine activation with $\dualact(0) > 0$. Suppose $1$ is the unique fixed point of $\dualact$ in $[0, 1]$ and also $\dot\dualact(1) = 1$. Then, for any $\epsilon \in (0, \frac{1}{2})$, after $L \geq L_0 := \max\left\{\left\lceil \frac{\log(2/\epsilon)}{-\log(1-\epsilon\mu/2)}\right\rceil, \left\lceil\frac{1}{\dualact(0)}\right\rceil \right\} = O\left(\frac{\log(1/\epsilon)}{\epsilon}\right)$ layers, for any $\rho \in [-1, 1]$, we have
\[\dualact^{(L)}(\rho) \geq 1 - \epsilon.\]
\end{lemma}
\begin{proof}
First, we note that since $1$ is the unique fixed point of $\dualact$ in $[0, 1]$, we must have $\dualact(\rho) > \rho$ for all $\rho \in [0, 1)$. Also, for any $\rho < 0$, as in case 4 of \pref{lem:one-layer-unnormalized}, we have $\dualact(\rho) \geq \rho + \dualact(0)$. So for any $\rho$, an application of $\dualact$ never decreases its value.

Since $\dot\dualact(1) = 1$, we have $\sum_{i=1}^\infty ia_i^2 = 1$. Hence, for any $\alpha \in [0, 1]$, we have
\[1 - \dot\dualact(1-\alpha) = \sum_{i=1}^\infty ia_i^2(1-(1-\alpha)^{i-1}) = \sum_{i=2}^\infty ia_i^2(1-(1-\alpha)^{i-1}) \geq \sum_{i=2}^\infty 2a_i^2(1-(1-\alpha)) = 2\alpha \mutil.\]
Next, since $\dualact$ is convex in $[0, 1]$ and $\dot\dualact(1) = 1$, for any $\rho \in [0, 1]$, we have $\dot\dualact(\rho) \leq 1$. Thus, $\dualact(1-\alpha) \geq \dualact(1) - \alpha = 1-\alpha$. Again using the convexity of $\dualact$ in $[0, 1]$, for any $\rho \in [0, 1-\alpha]$, we have
\[\dualact(\rho) \geq \dualact(1-\alpha) + \dot\dualact(1-\alpha)\cdot(\rho - (1-\alpha)) \geq 1 - \alpha + (1-2\alpha\mutil)(\rho - (1-\alpha)).\]
The second inequality above uses the facts that $\dualact(1-\alpha) \geq 1-\alpha$,  $\dot\dualact(1-\alpha) \leq 1-2\alpha \mutil$, and $\rho - (1-\alpha) \leq 0$. Simplifying and rearranging, we have
\[1-\dualact(\rho) \leq 2\alpha^2\mutil + (1-2\alpha \mutil)(1-\rho).\]
Thus, starting from any $\rho \in [0, 1-\alpha]$, and applying the above inequality recursively, after $L_1 := \left\lceil \frac{\log(1/\alpha)}{-\log(1-2\alpha\mu)}\right\rceil$ layers, either there is a layer $\ell$ such that $\dualact^{(\ell)}(\rho) > 1-\alpha$, or else, $\dualact^{(L_1)}(\rho) \geq 1 - 2\alpha$.

Finally, if $\rho < 0$, then after $L_2 := \left\lceil\frac{1}{\dualact(0)}\right\rceil$ layers, we reach a non-negative value, at which point the above analysis applies. The lemma now follows by setting $\alpha = \frac{\epsilon}{2}$.
\end{proof}

\subsection{Calculations of the coefficient of non-linearity}

For standard activation functions such as ReLU, it is easy to compute the coefficient of non-linearity of their normalized versions from their Hermite expansion. Specifically, if $\sigma(u) = \sum_{i=0}^\infty a_i h_i(u)$ is the Hermite expansion of an activation function $\sigma$, then the normalized version of $\sigma$, denoted $\bar{\sigma}$, is given by
\[ \bar{\sigma}(u) = \frac{\sum_{i=1}^\infty a_i h_i(u)}{\sqrt{\sum_{i=1}^\infty a_i^2}}.\]
Thus, the coefficient of non-linearity of $\bar{\sigma}$ is given by (see the proof of \pref{lem:mu-bounds}):
\[\mu = 1 - \frac{a_1^2}{\sum_{i=1}^\infty a_i^2}.\]
Since the dual activation of $\sigma$, $\dualact(\rho) = \sum_{i=0}^\infty a_i^2 \rho^i$, can be written as $\dualact(\rho) = \sum_{i=0}^\infty a_i^2 \rho^i$, we can also write the above formula for the coefficient of non-linearity as:
\[\mu = 1 - \frac{\dualact'(0)}{\dualact(1) - \dualact(0)}.\]
This latter formula easily allows us to compute the coefficient of non-linearity for various activations. For example, using Table 1 from \citep{daniely2016toward}, we get the following calculations of $\mu$:
\begin{center}
\begin{tabular}{cllc} 
 \hline
 Activation & $\sigma(u)$  & $\dualact(\rho)$ & $\mu$ \\
 \hline
 Identity & $u$ & $\rho$ & $0$\\ 
 2nd Hermite & $\frac{u^2-1}{\sqrt{2}}$ & $\rho^2$ & $1$ \\ 
 ReLU & $\max\{u, 0\}$ & $\frac{1}{2\pi} + \frac{\rho}{4} + \frac{\rho^2}{2\pi} + \frac{\rho^4}{48\pi} + \cdots = \frac{\sqrt{1-\rho^2} + (\pi - \cos^{-1}(\rho))\rho}{2\pi}$ & $\frac{\pi - 2}{2\pi - 2}$ \\
 Step & $\mathbf{1}[u \geq 0]$ & $\frac{1}{4} + \frac{\rho}{2\pi} + \frac{\rho^3}{12\pi} + \frac{3\rho^5}{80\pi} + \cdots = \frac{\pi - \cos^{-1}(\rho)}{2\pi}$ & $\frac{\pi - 2}{\pi}$\\
 Exponential & $e^{u}$ & $e^3(1 + \rho + \frac{\rho^2}{2} + \frac{\rho^3}{6} + \cdots) = e^{\rho + 3}$ & $\frac{e-2}{e-1}$ \vspace*{1mm}\\
 \hline
\end{tabular}
\end{center}

For the activation function $\normrelu_c$ (defined in \pref{sec:experiments}), it is easier to directly compute the coefficient of non-linearity as follows. First, note that $\normrelu_c$ is already normalized, so it suffices to compute the coefficient $a_1$ in its Hermite expansion. We have
\begin{align*}
  a_1 &= \E_{X \sim \normal(0, 1)}[X\normrelu_c(X)] \\
  &= \E_{X \sim \normal(0, 1)}[X\cdot \lambda(c)\cdot[\max\{X-c, 0\} + b(c)]] \\
  &= \E_{X \sim \normal(0, 1)}[\lambda(c)\cdot \mathbf{1}[X \geq c]] \\
  &= \lambda(c)(1-\varPhi(c)),  
\end{align*}
where the third step uses Stein's lemma, and in the fourth step, $\varPhi(\cdot)$ is the Gaussian cumulative distribution function. Thus, the coefficient of non-linearity for $\normrelu_c$ is
\[\mu = 1 - \lambda(c)^2(1-\varPhi(c))^2.\]
For the value of $c$ used in our experiments, i.e. $c = -1.5975$, we have $\lambda(c) \approx 1.05$, and $\varPhi(c) \approx 0.0551$, and thus $\mu \approx 0.0156$.

\section{Analysis for Non-Unit Length Inputs - Proof of \pref{thm:general_norms}}
\label{app:general_norm_proof}
In this section we provide a proof for \pref{thm:general_norms}. 

\subsection{Analyis for Norms}

As stated earlier our analysis begins first by analyzing and understanding the evolution of the norm of the representations across the network. To this end we prove the following theorem. 

\begin{theorem} 
\label{thm:norms}
Let $\sigma$ be a twice-differentiable non-decreasing activation function which satisfies the following properties: 
\begin{itemize}
    \item $\sigma(0)=0$
    \item $\sigma$ is concave on $\reals^+$ and $\sigma$ is convex on $\reals^-$
\end{itemize}
Furthermore define 
\[\alpha_{\sigma} := \min\left( 2\E_{x \in \normal (0,1)}\left[\sigma^2 \left(\sqrt{0.5}x\right)\right] - 1, 1 - \E_{x \in \normal (0,1)}\left[\sigma \left(x\right)\sigma'(x)x\right]\right).\]
We have that if $\|x\|^2 \in [0.5, \infty)$ then for any $L \geq 0$
    \[|\|\bar{k}(x,x)\|^2 - 1| \leq (1 - \alpha_{\sigma})^L \cdot |\|x\|^2-1|^2\]
\end{theorem}

\begin{remark}
From the proof of the theorem it will be evident under our assumptions on $\sigma$ that $\alpha_{\sigma} \in [0,1]$. Furthermore it will also be evident that the choice of $0.5$ is arbitrary and can be replaced by any constant $> 0$, and the definition of $\alpha_{\sigma}$ changes appropriately.
\end{remark}

In the rest of the section we prove \pref{thm:norms}. Firstly note that it is sufficient to prove the theorem for $L=1$. The general case  follows inductively easily. To prove the base case consider the following function $\hat{\sigma}_l: \reals^+ \rightarrow \reals^+$
\[ \hat{\sigma}_l(\gamma) := \E_{z \sim N\left(0,\gamma\right)}[\sigma^2(z)] = \E_{z \sim N\left(0,1\right)}[\sigma^2(\sqrt{\gamma}z)]\]
A simple application of successive central limit theorems across layers gives us that for any l
\[\bar{k}^{(l)}(x,x) = \hat{\sigma}_{l}^{l}(\|x\|^2)\]
Therefore all we are required to show is that for all $\gamma \geq 0.5$ we have that
\begin{equation}
\label{eqn:norm_target}
   |\hat{\sigma}_{l}(\gamma) - 1| \leq (1 - \alpha_{\sigma})|\gamma - 1| 
\end{equation}
As a reminder note that our assumption on $\sigma$ implies that $\hat{\sigma}_l(1) = 1$. To establish \eqref{eqn:norm_target} we begin with the following lemma characterizing the behaviour of the map $\hat{\sigma}_l$. 

\begin{lemma}
\label{lem:simp-conc}
Let $\sigma$ be a twice-differentiable monotonically increasing activation function which satisfies the following properties: 
\begin{itemize}
    \item $\sigma(0)=0$
    \item $\sigma$ is concave on $\reals^+$ and $\sigma$ is convex on $\reals^-$
\end{itemize}
Then $\hat{\sigma}_l(\gamma)$ is a twice-differentiable non-decreasing concave function on $\reals^{+}$. 
\end{lemma}
We now prove \eqref{eqn:norm_target} using \pref{lem:simp-conc}. We divide the analysis in two cases. Suppose $\gamma \in [0.5,1]$. Note that since $\hat{\sigma}_l(1)=1$ and $\sigma$ is monotonic we have that $\hat{\sigma}_l(\gamma)\leq 1$. Furthermore, since $\hat{\sigma}_l$ is a concave function and $\hat{\sigma}_l(1)=1$, we have that 
\[ \hat{\sigma}_l(\gamma) \geq 2(1 - \gamma)\hat{\sigma}_l(0.5) + 2(\gamma - 0.5)\]  
It now follows that
\begin{align}
    1 - \hat{\sigma}_l(\gamma) \leq (1 - (2\hat{\sigma}_l(0.5)-1))(1 - \gamma) \leq (1 - \alpha_{\sigma})(1 - \gamma).
\end{align}
Note that the concavity and monotonicity of $\hat{\sigma}_l$ also establish that $2\hat{\sigma}_l(0.5) - 1 \in [0,1]$. 

For the case of $\gamma \geq 1$. Note that since $\hat{\sigma}_l$ is concave and it is easy to see that 
\[\hat{\sigma}_l(\gamma) \leq \hat{\sigma}_l(1) + \hat{\sigma}_l'(1) (\gamma - 1), \]
which by noting that $\hat{\sigma}_l(1) = 1$ implies that 
\[\hat{\sigma}_l(\gamma) - 1 \leq  \hat{\sigma}_l'(1) (\gamma - 1) \leq (1 - \alpha_{\sigma})(\gamma - 1).\]
The last inequality follows by noting that $\hat{\sigma}_l'(1) = \E_{x \in \normal (0,1)}\left[\sigma \left(x\right)\sigma'(x)x\right]$. Again by concavity of $\hat{\sigma}$ and the conditions that $\hat{\sigma}(0) = 0$ and $\hat{\sigma}(1) = 1$, it can be readily seen that $\hat{\sigma}'(1) \leq 1$. 

This finishes the proof of \pref{thm:norms}. We finish this section by provide the proof of \pref{lem:simp-conc}

\begin{proof}[Proof of \pref{lem:simp-conc}]
In the rest of the proof we assume $\gamma > 0$. The calculation for $\gamma = 0$ case can be done analogously. The twice differentiability of $\hat{\sigma}_l$ can be seen easily from the definition.  

Further note that
\[\frac{\partial \hat{\sigma}(\gamma)}{\partial \gamma} = \E_{z \sim N\left(0,1\right)}\left[\sigma(\sqrt{\gamma}z)\sigma'(\sqrt{\gamma}z)\frac{z}{\sqrt{\gamma}}\right] \geq 0\]
The inequality follows since $\sigma(0)=0$ and $\sigma$ is non-decreasing (and hence $\sigma(x) \cdot x \geq 0$). Furthermore consider the computation for the second derivative
\[\frac{\partial^2 \hat{\sigma}(\gamma)}{\partial \gamma^2} = \frac{\E_{z \sim N\left(0,1\right)}\left[(\sigma'(\sqrt{\gamma}z))^2z^2\right]}{2\gamma} + \frac{\E_{z \sim N\left(0,1\right)}\left[\sigma(\sqrt{\gamma}z)\sigma''(\sqrt{\gamma}z)z^2\right]}{2\gamma} - \frac{\E_{z \sim N\left(0,1\right)}\left[\sigma(\sqrt{\gamma}z)\sigma'(\sqrt{\gamma}z)z\right]}{2\gamma^{3/2}}\]
We can now analyze by considering every term. Notice that the second term
\[\frac{\E_{z \sim N\left(0,1\right)}\left[\sigma(\sqrt{\gamma}z)\sigma''(\sqrt{\gamma}z)z^2\right]}{2\gamma} \leq 0\]
because under the assumptions $\sigma(x)$ and $\sigma''(x)$ always have opposite signs. 
We will now analyze the sum of the first and third terms. 
\begin{align*}
&\frac{\E_{z \sim N\left(0,1\right)}\left[(\sigma'(\sqrt{\gamma}z))^2z^2 \right]}{2\gamma} - \frac{\E_{z \sim N\left(0,1\right)}\left[\sigma(\sqrt{\gamma}z)\sigma'(\sqrt{\gamma}z)z\right]}{2\gamma^{3/2}} \\
&= \frac{\E_{z \sim N\left(0,1\right)}\left[(\sigma'(\sqrt{\gamma}z))^2z^2 \gamma - \sigma(\sqrt{\gamma}z)\sigma'(\sqrt{\gamma}z)z \sqrt{\gamma}\right]}{2\gamma^2} \\
&= \frac{\E_{z \sim N\left(0,\gamma\right)}\left[(\sigma'(z))^2z^2 - \sigma(z)\sigma'(z)z \right]}{2\gamma^2}
\end{align*}
We now show that $\forall z$,  $(\sigma'(z))^2z^2 - \sigma(z)\sigma'(z)z  \geq 0$. To this end note that $\forall z \geq 0$
\[(\sigma'(z))^2z^2 - \sigma(z)\sigma'(z)z = z\sigma'(z)(\sigma'(z)z - \sigma(z)) \leq 0\]
The inequality follows by noting monotonicity and concavity of $\sigma$ for $z \geq 0$. 
Similarly for $z \leq 0$,
\[(\sigma'(z))^2z^2 - \sigma(z)\sigma'(z)z = z\sigma'(z)(\sigma'(z)z - \sigma(z)) \leq 0\] 
which follows by noting monotonicity and convexity of $\sigma$ for $z \leq 0$.
\end{proof}


\subsection{Analysis for Dot-Products - Preliminaries}

Having established that the norms converge to $1$ we focus on the normalized dot-product between the representations and show that for odd functions it never increases and once the norms have converged close to 1 it decreases rapidly. To this end we will first need to define the following function which is a generalization of the $\dualact$ function defined in \pref{app:conditioning}. For any two vectors $x,y$ define $\dualact_c(\cdot, \cdot)$ as 
\[\hat{\sigma}_c(x,y) := \E_{w \sim \normal(0,I)}[\sigma(w^{\top}x)\sigma(w^{\top}y)]\]
A simple parametrization shows that this is equivalent to the following quantity
\[\hat{\sigma}_c(x,y) = \E_{z \sim N\left(0,\Sigma(x,y)\right)}[\sigma(z_1)\sigma(z_2)] \quad\text{ where }\quad \Sigma(x,y) = \begin{bmatrix}\|x\|^2& x\cdot y  \\x \cdot y & \|y\|^2\end{bmatrix}\]

To analyze the above quantity we use a general notion of (probabilist's) Hermite polynomials $h_{j}^{\{\gamma\}}(x)$ defined for any $\gamma$, defined to be the Hermite polynomials corresponding to the base distribution being $\normal(0, \gamma)$. We use the following specific definition derived from \cite{o2014analysis}. 

Consider for any $\gamma$ the quantity $exp(tz - \frac{\gamma^2t^2}{2})$. Considering the power series we get that the coefficient in front of $t^j$ is a polynomial in $z$ (with coefficients depending on $\gamma$). Defining this polnomial as $H_{j}^{\{\gamma\}}(z)$ we get that 
\begin{equation}
\label{eqn:mg}
    exp(tz - \frac{\gamma^2t^2}{2}) = \sum \frac{1}{j!}H_{j}^{\{\gamma\}}(z)\cdot t^j
\end{equation}
We can now define the hermite polynomials for any $\gamma \geq 0$ and $j \in \{0, 1, \ldots \}$ formally as 
\[ h_j^{\{\gamma\}}(z) := \frac{H_j^{\{\gamma\}}(z)}{\sqrt{\gamma^j \cdot j!}}\]
We show the following simple lemma about these polynomials (which also establishes that these polynomials form a basis under the distribution $\normal(0, \gamma)$). 
\begin{lemma}
\label{lem:generalized_orthogonality}
Given three numbers $\gamma_1 > 0 , \gamma_2 > 0, \gamma_3$ with $\gamma_3^2 \leq \gamma_1 \gamma_2$, define 
\[\Sigma := \begin{bmatrix}\gamma_1& \gamma_3  \\\gamma_3 & \gamma_2\end{bmatrix}\]
We have that the polynomial family $\{h_j^{\{\gamma\}}\}$ satisfies the following condition
\[\E_{z \sim \mathcal{N}(0, \Sigma)}[h_i^{\gamma_1}(z_1)h_j^{\gamma_2}(z_2)] = \begin{cases}
    \left(\frac{\gamma_3}{\sqrt{\gamma_1 \gamma_2}}\right)^j & \text{if } i=j\\
    0              & \text{otherwise}
\end{cases}\]
\end{lemma}
Note that setting $\gamma_1 = \gamma_2 = \gamma_3 = \gamma$ also establishes that for any $\gamma \geq 0$, $\{h_j^{\{\gamma\}}\}$ is an orthonormal family of polynomials. Therefore we can decompose $\sigma$ in the basis given by $h_j^{\{\gamma\}}$ as follows
\[\sigma(x) = \sum_{j=0}^{\infty} a_j^{\{\gamma\}} h_{j}^{\{\gamma\}}(x) \quad\text{ where }\quad a_j^{\{\gamma\}} = E_{z \sim \mathcal{N}(0, \gamma)} [\sigma(z)h_{j}^{\{\gamma\}}(z)]\]
Using \pref{lem:generalized_orthogonality} the following statement follows by decomposing in the appropriate bases
\begin{equation}
    \label{eqn:sigmahat}
    \hat{\sigma}_c(x,y) = \sum_{j=0}^{\infty} a_j^{\{\|x\|^2\}}a_j^{\{\|y\|^2\}} \rho(x,y)^i \quad\text{ where } \quad \rho(x,y) := \frac{x \cdot y}{\|x\|\|y\|}
\end{equation}
A special case of the above $\hat{\sigma}_c$ is when $x=y$. In this case it is easy see that $\hat{\sigma}_c(x,x) = \hat{\sigma}_l(\|x\|^2)$ (defined in the previous section). Accordingly
\begin{equation}
\label{eqn:sigmal_def}
    \hat{\sigma}_c(x,x) = \hat{\sigma}_l(\|x\|^2) = \sum_{j=0}^{\infty} (a_j^{\{\|x\|^2\}})^2
\end{equation}

We now show the following lemma which is a generalization of \pref{lem:one-layer}. 

\begin{lemma} \label{lem:one-layer-general}
    Given any number $\gamma > 0$ define the generalized coefficient of linearity as
    \[ \mu^{\{\gamma\}} := 1 - \frac{\left(a_1^{\{\gamma\}}\right)^2}{\sum_{j=0}^{\infty} \left(a_j^{\{\gamma\}}\right)^2}\]
    Further let $\delta = 1 - |\rho(x,y)| := 1 -\frac{|x^{\top}y|}{\|x\|\|y\|}\in [0, 1]$, then for any odd function $\sigma$, we have that 
    \[\frac{|\dualact_{c}(x,y)|}{\sqrt{\dualact_{c}(x,x)\dualact_{c}(y,y)}} \leq \begin{cases}
        \left(\sqrt{1 - (1+\frac{\mu^{\{\|x\|^2\}}}{2}) \delta}\right) \cdot \left(\sqrt{1 - (1+\frac{\mu^{\{\|y\|^2\}}}{2})\delta}\right) & \text{ if }\ \delta \leq \frac{1}{2} \\
        \sqrt{\left(1 - \frac{\mu^{\{\|x\|^2\}}}{2}\right)} \cdot \sqrt{\left(1 - \frac{\mu^{\{\|y\|^2\}}}{2}\right)}(1 - \delta) & \text{ if }\ \delta > \frac{1}{2}.
    \end{cases}\]
\end{lemma}

A simple consequence of the above theorem which follows by noting that since $\mu^{\{\gamma\}} \in [0,1]$ for all $\gamma$, we have that for odd function $\sigma$, the normalized dot-products always decrease in absolute value, i.e.  
\[\frac{|\dualact_{c}(x,y)|}{\sqrt{\dualact_{c}(x,x)\dualact_{c}(y,y)}} \leq |\rho(x,y)|.\]

We end this section with the proofs of \pref{lem:one-layer-general} and \pref{lem:generalized_orthogonality}.
\begin{proof}[Proof of \pref{lem:one-layer-general}]
     Firstly note that since $\sigma$ is an odd function it is easy to check that $a_0^{\{\gamma\}} = 0$ for all $\gamma$. Therefore we have the following consequence, 
\[ \frac{\dualact_c(x,y)}{\sqrt{\dualact_c(x,x) \cdot \dualact_c(y,y)}} = \frac{\sum_{j=1}^{\infty} a_j^{\{\|x\|^2\}}a_j^{\{\|y\|^2\}} \rho(x,y)^j}{\sqrt{\sum_{j=1}^{\infty} \left(a_j^{\{\|x\|^2\}}\right)^2}\sqrt{\sum_{j=1}^{\infty} \left(a_j^{\{\|y\|^2\}}\right)^2}}\]
A simple application of Cauchy-Schwartz gives
\[ \frac{\sum_{j=1}^{\infty} |a_j^{\{\|x\|^2\}}||a_j^{\{\|y\|^2\}}| |\rho(x,y)|^j}{\sqrt{\sum_{j=1}^{\infty} \left(a_j^{\{\|x\|^2\}}\right)^2}\sqrt{\sum_{j=1}^{\infty} \left(a_j^{\{\|y\|^2\}}\right)^2}} \leq \sqrt{\frac{\sum_{j=1}^{\infty}\left(a_j^{\{\|x\|^2\}}\right)^2 |\rho(x,y)|^j}{\sum_{j=1}^{\infty} \left(a_j^{\{\|x\|^2\}}\right)^2}} \sqrt{\frac{\sum_{j=1}^{\infty}\left(a_j^{\{\|y\|^2\}}\right)^2 |\rho(x,y)|^j}{\sum_{j=1}^{\infty} \left(a_j^{\{\|y\|^2\}}\right)^2}}\]
We will now analyse the left term first 
\begin{align}
    \frac{\sum_{j=1}^{\infty}\left(a_j^{\{\|x\|^2\}}\right)^2 |\rho(x,y)|^j}{\sum_{j=1}^{\infty} \left(a_j^{\{\|x\|^2\}}\right)^2} &\leq (1 - \mu^{\{\|x\|^2\}})|\rho(x,y)| + \mu^{\{\|x\|^2\}} \rho(x,y)^2 \\
    &= (1 - \mu^{\{\|x\|^2\}})(1 - \delta) + \mu^{\{\|x\|^2\}} (1 - \delta)^2 \\ &
    = (1 - \delta)(1 - \mu^{\{\|x\|^2\}}\delta)
\end{align}

    Now if $\delta > \frac{1}{2}$, we have $(1 - \delta)(1 - \mu^{\{\gamma\}}\delta) \leq (1 - \frac{\mu^{\{\gamma\}}}{2})(1 - \delta)$. If $\delta \leq \frac{1}{2}$, we have $(1 - \delta)(1 - \mu^{\{\gamma\}}\delta) = 1 - (1 + \mu^{\{\gamma\}})\delta + \mu^{\{\gamma\}}\delta^2 \leq 1 - (1+\frac{\mu^{\{\gamma\}}}{2})\delta$. Therefore we have that
    \[\frac{\sum_{j=1}^{\infty}\left(a_j^{\{\|x\|^2\}}\right)^2 |\rho(x,y)|^j}{\sum_{j=1}^{\infty} \left(a_j^{\{\|x\|^2\}}\right)^2} \leq \begin{cases}
        1 - (1+\frac{\mu^{\{\gamma\}}}{2})\delta & \text{ if }\ \delta \leq \frac{1}{2} \\
        (1 - \frac{\mu^{\{\gamma\}}}{2})(1 - \delta) & \text{ if }\ \delta > \frac{1}{2}.
    \end{cases}\]
    Repeating the same analysis for the second term and combining finishes the proof. 
\end{proof}

\begin{proof} [Proof of \pref{lem:generalized_orthogonality}]
Lets consider the moment generating function for a gaussian with covariance $\Sigma$.
\[\E_{z \sim \mathcal{N}(0, \Sigma)}[exp(sz_1 + tz_2)]\]
We can reparametrize this distribution as follows - consider two vectors $u,v$ such that $\|u\|=\gamma_1$, $\|v\|=\gamma_2$ and $u \cdot v = \gamma_3$. Now we see that 
\begin{align*}
    E_{z \sim \mathcal{N}(0, \Sigma)}[exp(sz_1 + tz_2)] &= E_{g \sim \mathcal{N}(0, I)}[exp(s(u_1g_1 + u_2g_2) + t(v_1g_1 + v_2g_2))] \\
    &= E_{g \sim \mathcal{N}(0, 1)}[exp((su_1 + tv_1)g)] \cdot E_{g \sim \mathcal{N}(0, 1)}[exp((su_2 + tv_2)g)] \\
    &= exp(0.5(su_1 + tv_1)^2) \cdot  exp(0.5(su_2 + tv_2)^2) \\
    &= exp(0.5(s^2\|u\|^2 + t^2\|v\|^2) + st (u\cdot v))
\end{align*}
Rearranging the above and replacing relevant quantities we get that 
\[E_{z \sim \mathcal{N}(0, \Sigma)}[exp(sz_1 - 0.5 \gamma_1^2 s^2) \cdot exp(tz_2 - 0.5 \gamma_2^2 t^2)] = exp(st \gamma_3) = \sum_{j=0}^{\infty} \frac{1}{j!}s^jt^j \gamma_3^j\]
Working with the expression derived for LHS in \eqref{eqn:mg} we get that 
\[\sum_{i,j} \frac{1}{i!j!} E_{z \sim \mathcal{N}(0, \Sigma)}[H_i^{\gamma_1}(z_1)H_j^{\gamma_2}(z_2)] s^it^j = \sum_{j=0}^{\infty} \frac{1}{j!} s^jt^j \gamma_3^j\]
Matching coefficient we get that 
\[E_{z \sim \mathcal{N}(0, \Sigma)}[H_i^{\gamma_1}(z_1)H_j^{\gamma_2}(z_2)] = \begin{cases}
    j! \cdot \gamma_3^j & \text{if } i=j\\
    0              & \text{otherwise}
\end{cases}\]
Using the definition
\[ h_j^{\gamma}(z) := \frac{H_j^{(\gamma)}(z)}{\sqrt{\gamma^j \cdot j!}}\]
now finishes the proof.
\end{proof}

\subsection{Proof of \pref{thm:general_norms}}

In this section we prove \pref{thm:general_norms}. We re-state the theorem in a more precise way as follows. 

\begin{theorem} 
\label{thm:general_norms_1}
Let $\sigma$ be a twice-differentiable monotonically increasing odd function which is concave on $\reals^+$. There exists a constant $\alpha_\sigma$ (defined in \pref{thm:norms}) such that for any two inputs $x,y$ such that $\|x\|^2,\|y\|^2 \geq 0.5 $ with $\frac{x^{\top}y}{\|x\|\|y\|} \leq 1 - \delta$ for some $\delta > 0$, after a number of layers 
\[L \geq \hat{L} := \frac{1}{\alpha_{\sigma}}\log\left(\frac{4\max(|\|x\|^2-1|,|\|y\|^2-1|, \mu/4)}{\mu}\right),\] 
we have that
\[ \frac{ \phi_{W}(x)^{\top} \phi_{W}(y)}{\|\phi_{W}(x)\|\|\phi_{W}(y)\|}\leq B_{\mu/2}\left(L - \hat{L}, \delta\right)\] 
\end{theorem}
\begin{proof}[Proof of \pref{thm:general_norms_1}]
    We begin the proof by first noting for any $x$ such that $\|x\|^2 \in [1 - \mu/4, 1 + \mu/4]$ we have that
    \begin{equation}
        \label{eqn:sub-claim}
        \mu^{\{\|x\|^2\}} \geq \mu/2 
    \end{equation}
where $\mu^{\{\|x\|^2\}}$ is as defined in \pref{lem:one-layer-general} and $\mu = \mu^{\{1\}}$ is as defined in \pref{def:coefficient-non-linearity}. To see this first note that since $\mu \leq 1$, $1 - \mu/4 \geq 3/4$ and hence an application of \pref{thm:norms} gives us that
\[  \dualact_{l}(\|x\|^2) = \dualact_{c}(x,x) = \sum_{j=0}^{\infty} (a_j^{\{\|x\|^2\}})^2 \in [1 - \mu/4, 1 + \mu/4] \]
As a consequence since $\sum_{j=0}^{\infty} (a_j^{\{1\}})^2 = 1$ we also get that $a_1^{\{\|x\|^2\}} \leq 1 - \mu + \mu/4$. Therefore we have that 
\[\mu^{\{\|x\|^2\}} := 1 - \frac{\left(a_1^{\{\|x\|^2\}}\right)^2}{\sum_{j=1}^{\infty} \left(a_j^{\{\|x\|^2\}}\right)^2} \geq 1 - \frac{1 - 3\mu/4}{1 - \mu/4} \geq \mu/2 \]

The proof now follows by first noting from \pref{thm:general_norms} that after $\hat{L}$ layers we have that the representations have norms smaller than $(1 -\mu/4, 1 + \mu/4)$. Formally for $l \geq \hat{l}$ 
\[ \bar{k}^l(x,x), \bar{k}^l(y,y) \in [1 - \mu/4, 1 + \mu/4]\]
Now using the claim derived in \eqref{eqn:sub-claim}, using \pref{lem:one-layer-general} iteratively and using the definition of $B_{\nu}(L, \delta)$ immediately implies the theorem. 

\end{proof}

\subsection{Discussion}
\label{app:general_result_discussion}

To get the results for inputs with general norms we introduce extra assumptions on the activation function. As noted earlier, the assumption of differentiability is only for convenience, the result will hold for any function with finite non-smooth points. The most restrictive assumption we require on the activation functions is the assumption of it being odd. This is a specification of our unbiased normalization assumption. We note that many activations like tanh, arctan, soft sign etc satisfy this assumption. Finding a more general condition for global convergence of inner products to 0 is an intriguing open question. 

Furthermore our results on the norm convergence holds more generally than odd functions (See \pref{thm:norms}). It requires the activation to have a convex-concave structure and be 0 at 0. This assumption is far weaker and in fact most of the activations used in practice satisfy this (for eg. standard/leaky ReLU, SeLU, sigmoid, tanh, arctan etc.). Once again finding general conditions for global convergence of norm to fixed points is an intriguing open question. \cite{poole} informally (but incorrectly) mention that just monotonicity of activation is a sufficient condition, however the standard ReLU with a little shift, i.e. $\max(x,0) - 1$, itself presents a simple counter-example to the claim.


\section{Optimization Proofs}
\label{app:opt-proofs}

Suppose we train the network using gradient descent on a loss function $\ell: \reals \times \cY \rightarrow \reals$, which defines the empirical loss function \[\loss(\vecW) := \frac{1}{n}\sum_{i=1}^n \ell(f_{\vecW}(x_i), y_i).\] 
For the rest of this section we will assume that the loss function $\ell$ is the square loss, i.e. $\ell(\hat{y},y)=(\hat{y}-y)^2$. The results presented can appropriately be extended to the setting where the loss function is smooth and strongly convex. 
Training a finite-width neural network necessitates the study of the conditioning of the finite-width kernel matrices $K$ and $\ntk$, rather than their infinite-width counterparts. In such settings optimization results typically follow from a simple 2-step modular analysis:
\begin{itemize}
	\item \textbf{Step 1. [Initial Stability]} Standard concentration inequalities imply that if the width is large enough, conditioning of the infinite-width kernel matrices transfers to their finite-width counterparts at initialization.
	\item \textbf{Step 2. [Training Stability]} Standard optimization theory implies that conditioning in finite-width kernel matrices leads to fast training. In the case of training only the top layer this is sufficient. When training all layers, a much more careful analysis is needed to show that the NTK stays "close" to initialization, leading to conditioning throughout the training process. 
\end{itemize}

We now provide a couple of representative optimization results that follow from this type of analysis. Our goal here is to merely provide representative examples of typical optimization scenarios and highlight what benefits conditioning can lead to. Indeed, we believe extensions and improvements can be derived with significantly better bounds.  

\subsection{Training only the top layer}

We consider a mode of training where only the top layer weight vector, $v$, is updated, while keeping $W_1, W_2, \ldots, W_L$ frozen at their randomly initialized values. To highlight this we introduce the notation $\vecW_{1:L} = \{W_1 \ldots W_L\}$. Let  $\eta > 0$ be a step size, the update rule at iteration $t$ is given by 
\[v_{t+1} = v_t - \eta \cdot \partial_{v}\loss(\{v_t, \vecW_{1:L}\}) = v_t - \eta \cdot \frac{1}{n} \sum_{i=1}^n 2(v_t \cdot \Phi_{\vecW}(x_i) - y_i) \Phi_{\vecW}(x_i).\] 
Note that in this mode of training, the associated optimization problem is convex in $v$. To implement Step 1 of the modular analysis, we  appeal to the results of \citet{daniely2016toward}. They show that when the activations are suitably bounded (see Definition 6 in their paper for $C$-bounded activations) and the width is large enough, then with high probability, each entry in the kernel matrix $K$ is close to the corresponding entry in $\bar{K}$. Specifically, via Theorems 2 and 3 in their paper, we have the following version of \pref{thm:main-infinite-width-top-layer} for finite width neural networks:
\begin{lemma}[Via Theorem 2 in \cite{daniely2016toward}]
\label{lem:main-finite-width-top-layer}
    For any $\gamma > 0$, suppose that the activation $\sigma$ is $C$-bounded and $m = \Omega\left(\frac{(4C)^{L+1} \log(n)}{\gamma^2}\right)$, then with high probability, we have that for all $i,j$,  $|K_{ij} - \bar{K}_{ij}| \leq \gamma$. 
\end{lemma}

Step 2 follows by using standard convex optimization theory \citep{nesterov-lectures}, which tells us that the convergence rate of gradient descent for this problem depends on the condition number of $K$. Specifically, we have the following result:

\begin{theorem}
\label{thm:top-layer-general-app}
Suppose $L = \Theta\left(\frac{\log(n/\delta)}{\mu}\right)$. If $\sigma$ is $C$-bounded and the width $m = \text{poly}(n, \frac{1}{\delta})$, then setting $\eta = \Theta\left(\frac{1}{ \lambda_{\max}(\bar{K})}\right)$, we get that with high probability over the initialization, 
\[ \loss(\{v_t,\vecW_{1:L}\}) \leq e^{-\Omega\left(t\right)}\cdot \loss(\{v_0,\vecW_{1:L}\})\]
Alternatively, in order to find a point that is $\epsilon$ sub-optimal, gradient descent needs $O(\log(\frac{1}{\epsilon}))$ steps. 
\end{theorem}

Similarly, one can also derive a linear convergence theorem for stochastic gradient descent:
\begin{theorem}
\label{thm:sgd-fixed-bottom-app}
With the same choice of parameters as in  \pref{thm:top-layer-general-app}, appropriate choice of $\eta$ and with high probability over the initialization, stochastic gradient descent finds a point that is $\epsilon$-sub-optimal in expectation in at most $O\left(\log(\frac{1}{\epsilon})\right)$ steps.
\end{theorem}

\begin{remark}
The rate in the exponent in the theorem above naturally depends upon the condition number of the kernel matrix K. For simplicity, we choose to state the theorem for a depth at which the condition number is $O(1)$. Precise rates depending on $L$, can be derived from \pref{cor:main-top-layer-cn}. 
\end{remark}

\subsection{Training All The Layers Together}

In this section we provide a representative result for the training dynamics when all the layers are trained together with a fixed common learning rate. The dynamics are given by

\[\vecW(t+1) = \vecW(t) - \eta \partial_{\vecW}\loss(\vecW(t))\]

Now since the bottom layers also move the kernel changes at every step. The standard analysis in this setting follows from carefully establishing that the NTK does not change too much during the training procedure allowing for the rest of the analysis to go through. The following theorem from \cite{lee-wide-neural-2019} summarizes one such setting for smooth activation functions. 

\begin{theorem}[Thereom G.4 in \cite{lee-wide-neural-2019}]
\label{thm:lee-smooth-act-optimization-app}
Suppose that the activation $\sigma$ and its derivative $\sigma'$ further satisfies the properties that there exists a constant $c$,  such that for all $x,x'$ 
\[|\sigma(x)|, |\sigma'(x)|, \frac{|\sigma'(x)-\sigma'(x')|}{|x-x'|} \leq c.\]
Then there exists a constant $N$ (depending on L, n, $\delta$) such that for width $m > N$ and setting the learning rate $\eta  = 2(\lambda_{\min}(\bar{K}) + \lambda_{\max}(\bar{K}))^{-1}$, with high probability over the initialization the following is satisfied for gradient descent for all $t$, 
\[\loss(\vecW(t)) \leq e^{- \Omega\left(\frac{t}{\kappa(\bar{K})}\right)} \loss(\vecW(0))\]
\end{theorem}
The following corollary is now a simple application of the above theorem and \pref{cor:main-infinite-width-ntk-cn}.
\begin{corollary}
\label{cor:train_all_layers-app}
Suppose the conditions in \pref{thm:lee-smooth-act-optimization-app} are satisfied and the width is taken to be a large enough constant (depending on $L,n,\delta$) and further $L = \Theta(\frac{\log(n/\delta)}{\mu})$, then gradient descent with high probability finds an $\epsilon$ suboptimal point in total time $O(\log(1/\epsilon))$.  
\end{corollary}

\begin{remark}
As stated in \pref{thm:lee-smooth-act-optimization-app} the width required could be a very large constant. However, note that we require the depth to be logarithmic in $\frac{1}{\delta}$ for achieving constant condition number. Therefore the exponential in L factors accrued in the analysis of \pref{thm:lee-smooth-act-optimization-app} are actually polynomial in $\frac{1}{\delta}$. Therefore, merging results from \cite{arora2019exact}, we can derive a polynomial in $\frac{1}{\delta}$ upper bound on the width of the network. This matches the best known bounds on the overparameterization while improving the optimization rates exponentially (in $\frac{1}{\delta}$). Further we believe similar results can also be derived for ReLU activations following techniques in \cite{allen2018convergence}. 
\end{remark}

The proofs for the theorems in this section follow easily from our established results and standard arguments from optimization theory. We provide them next for completeness. 

\subsection{Proofs}

We begin by proving simple well-known theorems regarding gradient descent and stochastic gradient descent for linear regression. Consider the following problem 
\[\loss(w) = \frac{1}{n}\sum_{i=1}^{n} \|a_i^{\top}w - y\|^2\]
Let $A$ be the matrix whose rows are $a_i$. We will assume that $\lambda_{\min}(A^{\top}A) > 0$, which in particular implies that $\min_w \loss(w) = 0$.  Lets first consider gradient descent, i.e.
\[w_{t+1} = w_{t} - \eta \nabla \loss(w_t)\]
We have the following well known guarantee \citep{nesterov-lectures}. 
\begin{theorem}
\label{thm:grad-descent}
For gradient descent we have that 
\[\loss(w_t) \leq e^{- \frac{t}{4\kappa(A^{\top}A)}} \loss(w_{0})\]
\end{theorem}
Next we consider the stochastic gradient descent algorithm,
\[w_{t+1} = w_{t} - \eta \tilde{\nabla} \loss(w_t),\]
where $\tilde{\nabla} \loss(w_t)$ is a gradient over a single $a_i$, which is uniformly randomly sampled. Since we have assumed that the loss is 0, even SGD is known to have linear convergence in this setting. Since we did not find a concise proof of this fact anywhere we include it here for completeness. 
\begin{theorem}
\label{thm:std-grad-descent}
Let $\|a_i\|^2 \leq \beta$ for all $i$, then 
stochastic gradient descent produces an $\epsilon-$sub-optimal point in expectation in total number steps bounded by 
\[\frac{8n\beta}{\lambda_{\min}(A^{\top}A)}\cdot \log(1/\epsilon))\]
\end{theorem}
\begin{proof}
\begin{align*}
    E_t[\|w_{t+1} - w^*\|^2] &\leq E_t[\|w_{t} - w^*\|^2] - \eta (w_{t} - w^*)^{\top}E_t[\tilde{\nabla} \loss(w_t)] + \eta^2 E_t[\|\tilde{\nabla} \loss(w_t)\|^2] \\  
    &\leq \|w_{t} - w^*\|^2 - \eta (w_{t} - w^*)^{\top} \nabla \loss(w_t) + \eta^2 \|\tilde{\nabla} \loss(w_t)\|^2 \\
    &\leq \|w_{t} - w^*\|^2 - \eta \loss(w_t) + \eta^2 E_i[(a_i^{\top}w - y_i)^2\|a_i\|^2] \\
    & \leq \|w_{t} - w^*\|^2 - \eta \loss(w_t) + \eta^2\beta \loss(w_t)
\end{align*}
where $E_t$ refers to expectation conditioned on all the randomness till step $t$. Rearranging the above we get,
\[\loss(w_t)(\eta - \eta^2\beta) \leq \|w_{t} - w^*\|^2 - E_t[\|w_{t+1} - w^*\|^2]\]
Summing the above over time $T$ gives us that
\[E\left[\loss\left(\frac{1}{T}\cdot \sum_{t=1}^{T} w_t\right)\right] \leq E\left[\frac{1}{T}\cdot \sum_{t=1}^{T} \loss\left(w_t\right)\right] \leq \frac{E[\|w_{1} - w^*\|^2]}{T(\eta - \eta^2\beta)} \leq \frac{n\cdot E[\loss(w_1)]}{\lambda_{\min}(A^{\top}A)T(\eta - \eta^2\beta)}\]
Setting $\eta$ to $1/(2\beta)$ and $T = 8n\beta/\lambda_{\min}(A^{\top}A)$ we get that
\[E\left[\loss\left(\frac{1}{T}\cdot \sum_{t=1}^{T} w_t\right)\right] \leq \frac{E[\loss(w_1)]}{2}\]
Repeating this process a total of $\log(1/\epsilon)$ times gives us that after $\frac{8n\beta}{\lambda_{\min}(A^{\top}A)}\log(1/\epsilon)$ steps, SGD produces a point which is  $\epsilon$-sub-optimal point in expectation.
\end{proof}

\begin{proof}[Proof of Theorem \ref{thm:top-layer-general-app}]
The statement follows by noticing that at that setting of depth, \pref{cor:main-top-layer-cn} implies that the infinite-width kernel has constant condition number. Now invoking \pref{lem:main-finite-width-top-layer} implies that the finite-width kernel also has a constant condition number. The statement then follows from \pref{thm:grad-descent}.
\end{proof}
\begin{proof}[Proof of \pref{thm:sgd-fixed-bottom-app}]
We wish to invoke \pref{thm:std-grad-descent}. To this end note that, using \pref{lem:main-finite-width-top-layer} and the fact that the diagonal entries are 1 in $\bar{K}$, we get that $\beta \leq 2$ w.h.p. Similarly using \pref{cor:main-top-layer-cn}, we can derive that $\lambda_{\min} = \Omega(n)$ w.h.p. Therefore using \pref{thm:std-grad-descent} we get the required result. 
\end{proof}

\begin{proof}[Proof of \pref{cor:train_all_layers-app}]
The theorem follows by noticing that \pref{cor:main-infinite-width-ntk-cn} implies that at that depth, the condition number of the infinite-width NTK is constant. The statement now follows from \pref{thm:lee-smooth-act-optimization-app}.
\end{proof}

\section{Proofs on SQ Learnability of Random Deep Neural Nets}

In this section we prove Theorem~\ref{thm:sq-lower-bound} regarding SQ learnability of randomly initialized deep neural networks.
Specifically, we prove our result under the assumption that the (normalized) activation $\sigma$ is subgaussian with constant subgaussian norm. In particular we assume that
\begin{align}
\label{eq:activation_subgaussian}
\E_{X \sim \normal(0,1)}[e^{\lambda \sigma(X)}] \leq e^{\lambda^2 \alpha^2/2},
\end{align}
for a constant $\alpha > 0$. Many activations such as the sign, ReLU and tanh satisfy this assumption.

A key component in establishing SQ hardness of learning is to show that given two non-collinear unit length vectors, a randomly initialized network of depth $h$ and sufficiently large width width makes, in expectation, the pair nearly orthogonal. In other words, the magnitude of the expected dot product between any pair decreases exponentially with depth.  While \citet{panigrahy-random-nets} proved the result for sign activations, we prove the statement for more general activations and then use it to establish SQ hardness of learning.
As mentioned in Section~\ref{sec:sq}, we will work with networks that will normalize the output of each layer to unit length via the operation $\Pi: \mathbb{R}^m \rightarrow \mathbb{R}^m$. Then we have the following theorem:
\begin{theorem}
\label{th:dot-product-applications}
Let $\sigma: \reals \rightarrow \reals$ be a non linear activation with $\mu$ being the coefficient of non-linearity as in Definition~\ref{def:coefficient-non-linearity} and satisfying \eqref{eq:activation_subgaussian}. Let $x_i, x_j \in \mathbb{R}^d$ be unit length vectors such that $|x_i \cdot x_j| \leq 1-\delta$.  Define $\Phi_{\vecW}(x) := \tfrac{1}{\sqrt{m}}\Pi(\sigma(W_L \tfrac{1}{\sqrt{m}} \Pi(\sigma(W_{L-1} \cdots \tfrac{1}{\sqrt{m}} \Pi(\sigma(W_1x)\cdots)))$, where each column of $W_1$ is sampled from $\mathcal{N}(0, I_{d \times d})$ and each column of $W_i$ is sampled from $\mathcal{N}(0, I_{m \times m})$ for $i > 1$. Furthermore, the operation $\Pi$ normalizes the output of each layer to unit length. 
Let  $m > c_1 \frac{L}{\mu^2 \delta^2}$ for a universal constant $c_1 > 0$ and for $h \in [1,L]$ define $\rho_h$ be the dot product obtained by taking the representation of $x_i, x_j$ at depth $h$ of the network defined above. Then for any $h> 1$, it holds that
\begin{align*}
\big|\E[\rho_{L_0(\delta) + h}] \big| &\leq e^{-\Omega(h)} + L e^{-\Omega(L)}.
\end{align*}
where $L_0(\delta) = c_1 \frac{\log(\frac{1}{\delta})}{\mu}$ and $c_1 > 0$ is a universal constant.
\end{theorem}
While the above theorem is not a black blox application of our main result~(\pref{thm:main-infinite-width-top-layer}) since careful concentration arguments are required due to finite width, the calculations are of a similar flavor. 

We now show how the above theorem can be used to generalize the SQ lower bound of \cite{panigrahy-random-nets}. Before describing our results, we recall that in the SQ model \citep{kearns1998efficient} the learning algorithm does not have access to a labeled training set. Instead, for a given target function $f$ and a distribution $D$ over $\mathbb{R}^d$, the algorithm has access to a query oracle $SQ_{f,D}(\psi, \tau)$. The oracle takes as input a query function $\psi$, and outputs a value $v$ such that $|\E_D[\psi(x,f(x))] - v| \leq \tau$. The goal of the algorithm is to use the query algorithm to output a function $g$ that $\epsilon$ approximates $f$, i.e., $Pr_D[g(x) f(x)] \geq \epsilon$, for a given $\epsilon > 0$. 

\citep{panigrahy-random-nets} established an SQ learnability lower bound for a subclass $\mathcal{F}$ of neural networks with the property that a randomly initialized neural network falls in $\mathcal{F}$ with high probability. This however only establishes that the class $\mathcal{F}$ is hard to SQ learn as opposed to showing that a randomly initialized neural network is hard to learn. Furthermore, the lower bound only applies to networks with sign activations. We now show how to generalize their result in two ways: (a) we allow arbitrary activations satisfying \eqref{eq:activation_subgaussian}, and (b) our lower bound shows that a randomly initialized network is hard to learn in the SQ model with constant probability. We achieve the stronger lower bound by carefully adapting the lower bound technique of \citet{bshouty2002using}. 

In our context we will fix a non linear activation $\sigma: \reals \rightarrow \reals$ and let the target be of the form $\text{sgn}(f_{\vecW}(x))$ where,
$$
f_{\vecW}(x) =  \big(v \cdot \tfrac{1}{\sqrt{m}}\Pi(\sigma(W_L \tfrac{1}{\sqrt{m}} \Pi(\sigma(W_{L-1} \cdots \tfrac{1}{\sqrt{m}} \Pi(\sigma(W_1x)\cdots))) \big) 
$$
where each column of $W_1$ is sampled from $\mathcal{N}(0, I_{d \times d})$ and $v$ and each column of $W_i$ is sampled from $\mathcal{N}(0, I_{m \times m})$ for $i > 1$. Furthermore we will use the depth $L$ and the dimensionality $d$ to parameterize the bit complexity of the network description. We say that an algorithm $(p(d,L), r(d,L), q(d,L))$-SQ learns $\text{sgn}(f_{\vecW}(x))$ if with probability at least $1/2$ over the randomness in $\vec{W}$, the algorithm makes at most $p(d,L)$ queries to the SQ oracle for $\text{sgn}(f_{\vecW}(x))$, receives responses from the oracle up to tolerance $\tau = 1/r(d,L)$ and outputs a $g$ that $\epsilon = 1/q(d,L)$-approximates $f$. Furthermore it is the case that each query function $\psi$ used by the algorithm can be evaluated in time $q(d,L)$.

Then we have the following theorem, that is a more formal restatement of \pref{thm:sq-lower-bound} and extends the result of \cite{panigrahy-random-nets}.
\begin{theorem}
 \label{thm:sq-lower-bound-app}
Fix any non linear activation $\sigma$ with the coefficient of non-linearity being $\mu$ that satisfies \eqref{eq:activation_subgaussian}. Any algorithm that $(p(d,L), poly(d,L), poly(d,L))$-SQ learns the random depth $L$ networks as defined above with width $m = \Omega(L\frac{\mu^2}{\delta^2})$ must satisfy $p(d,L) \geq e^{\Omega(L)}$.
\end{theorem}  


\label{app:sq-app}
\begin{proof}[Proof of Theorem~\ref{th:dot-product-applications}]
We use the following notation in the proof. Given input $x_i$, we denote $x_i^{(h)}$ to be the representation obtained at depth $h$ of the network and $\hat{x}_i^{(h)}$ to be the corresponding normalized input. Recall that we are normalizing the output of each layer to be unit length. Similarly, given $x_i, x_j$, we denote by $\hat{\rho}_h = \hat{x}_i^{(h)} \cdot \hat{x}_j^{(h)}$ and $\rho_h = x_i^{(h)} \cdot x_j^{(h)}$. Next we have that conditioned on $x^{(h-1)}_i$ and $x^{(h-1)}_j$,
\begin{align*}
\|x_i^{(h)}\|^2 &= \frac{1}{m} \sum_{j=1}^m \sigma^2(w_j \cdot \hat{x}^{(h-1)})
\end{align*}
where $w_j \sim N(0,I)$ and $\E[\|x_i^{(h)}\|^2] = 1$. Furthermore, since $\sigma(w_j \cdot \hat{x}^{(h-1)})$ is a subgaussian random variable with constant subgaussian norm, $\|x_i^{(h)}\|^2$ is a sum of subexpoential random variables. By Bernstein's inequality for subexponential random variables \citep{vershynin2018high} we have that for a universal constant $c > 0$,
\begin{align}
\label{sq:prob-bound-1}
Pr(\big| \|x_i^{(h)}\|^2 - 1\big| > t) &\leq 2e^{-c \min \big({mt^2}, mt \big)}.
\end{align}
Similarly we have that
\begin{align*}
\rho_h &= \frac{1}{m} \sum_{j=1}^m \sigma(w_j \cdot \hat{x_i}^{(h-1)}) \sigma(w_j \cdot \hat{x_j}^{(h-1)})
\end{align*}
with $\E[\rho_h] = \hat{\sigma}(\hat{\rho}_{h-1})$. Noting that product of subgaussian random variables is subexponential and again applying Bernstein's inequality for subexponential random variables we get that
\begin{align}
\label{sq:prob-bound-2}
Pr(\big| \rho_h - \hat{\sigma}(\rho_{h-1})\big| > t) &\leq 2e^{-c \min \big({mt^2}, mt \big)}.
\end{align}
Next, we will use \eqref{sq:prob-bound-1} and \eqref{sq:prob-bound-2} to argue that with high probability $\hat{\rho}_h$ remains close to $\hat{\sigma}(\hat{\rho}_{h-1})$. For suitable constant $\epsilon < 1$ to be chosen later, we have that
\begin{align*}
Pr(|\hat{\rho}_h - \hat{\sigma}(\hat{\rho}_{h-1})| > t) &= Pr(\Big|\frac{\rho_h}{\|x_i^{(h)}\| \|x_j^{(h)}\|} - \hat{\sigma}(\hat{\rho}_{h-1}) \Big| > t)\\
&\leq Pr(\Big|\frac{\rho_h - \hat{\sigma}(\hat{\rho}_{h-1})}{\|x_i^{(h)}\| \|x_j^{(h)}\|} \Big| > t) + Pr(\Big|\hat{\sigma}(\hat{\rho}_{h-1}) \big(\frac{1}{\|x^{(h)}_i\| \|x^{(h)}_j\|} - 1 \big) \Big| > t)\\
&\leq Pr(\|x_i^{(h)}\| \|x_j^{(h)}\| > (1+\epsilon)^2) + Pr(\Big|\rho_h - \hat{\sigma}(\hat{\rho}_{h-1})\Big| > t(1+\epsilon)^2)\\
&+ Pr(\Big| \|x^{(h)}_i\| \|x^{(h)}_j\| - 1 \Big| > \frac{2t}{\hat{\sigma}(\hat{\rho}_{h-1})}). 
\end{align*}
Noticing that $\hat{\sigma}(\hat{\rho}_{h-1}) \leq 1-(1+\frac{\mu}{2})\delta$, and using \eqref{sq:prob-bound-1} and \eqref{sq:prob-bound-2}, we get that
\begin{align*}
Pr(|\hat{\rho}_h - \hat{\sigma}(\hat{\rho}_{h-1})| > t) &\leq 2 \Big( e^{-cm\epsilon^2} + e^{-c m (1+\epsilon)^4 t^2} + e^{-\frac{mt}{2(1-(1+\mu/2)\delta)}}\Big).
\end{align*}
Setting $t = \delta \mu/4$ and $\epsilon$ to be a small enough constant we get that 
\begin{align*}
Pr(|\hat{\rho}_h - \hat{\sigma}(\hat{\rho}_{h-1})| > \frac{\delta \mu}{4}) &\leq 2 \Big( e^{-\Omega(m)} + e^{-\Omega(m \mu^2 \delta^2 )} + e^{- \Omega(m \delta \mu)}\Big).
\end{align*}
Setting $m \geq c_1 \frac{L}{\mu^2 \delta^2}$ and using a union bound over all layers we get that with probability at least $1 - L e^{-\Omega(L)}$, the updates of $\hat{\rho}_h$ will approximately satisfy the ideal updates from \pref{thm:main-infinite-width-top-layer} and as a result, for a constant $c_1 > 0$, after $L_0(\delta) = c_1 \frac{\log(1/\delta)}{\mu}$ depth, with high probability, $\hat{\rho}_h$~(and $\rho_h$) will fall below $1/4$ and will continue to be below $1/2$ for all $L$. Define $G$ to be the intersection of above good event and that $\rho_h \in [(1-\epsilon) \hat{\rho}_h, (1+\epsilon) \hat{\rho}_h]$ for all $h \in [L]$. Then we know that $P(G) \geq 1-2Le^{-\Omega(L)}$.
Conditioned on this good event and using \pref{lem:one-layer} we have that for $h > L_0(\delta)$, 
\begin{align*}
\big|\E[\rho_{h} | G, \rho_{h-1}] \big| &= \big|\hat{\sigma}(\hat{\rho}_{h-1}) \big|\\
&\leq (1-\frac{\mu}{2})|\hat{\rho}_{h-1}|\\
&\leq (1+\epsilon)(1-\frac{\mu}{2}) |\rho_{h-1}|\\
&\leq (1-\frac{\mu}{4}) |\rho_{h-1}|
\end{align*}
for a small enough constant $\epsilon$. Hence we get that for $h > L_0(\delta)$, $|\E[\rho_h | G]| \leq e^{-\Omega(h)}$. Finally notice that
$$
\E[\rho_h] = P(G) \E[\rho_h | G] + P(\bar{G})\E[\rho_h | \bar{G}]
$$
Combined with the probability of the good event and noticing that $\hat{\sigma}(\hat{\rho}_h)$ is always bounded, we get that
$$
|\E[\rho_{L_0(\delta) + h}]| \leq e^{-\Omega(h)} + Le^{-\Omega(L)}. 
$$

\end{proof}
\begin{proof}[Proof of Theorem~\ref{thm:sq-lower-bound} (same as Theorem~\ref{thm:sq-lower-bound-app})]
We will consider a randomly initialized deep neural network defined as 
$$
f_{\vecW}(x) = v \cdot \Pi(\sigma(W_L \tfrac{1}{\sqrt{m}} \Pi(\sigma(W_{L-1} \cdots \tfrac{1}{\sqrt{m}} \Pi(\sigma(W_1x)\cdots))) \big) 
$$
We consider a distribution $D$ that is the uniform distribution over a set $\mathcal{S}$ that consists of half of the inputs in $\{+1,-1\}^d$ thereby ensuring that no two inputs are collinear. In particular, one can take the set of all $2^{d-1}$ inputs that fall on one side of a fixed halfspace. The first step in the analysis is to show that $f = \text{sgn}(f_{\vecW})$ is uncorrelated with any fixed function $g: \mathbb{R}^d \rightarrow [-1,1]$. In particular we have that
\begin{align*}
\E_{\vecW}[\E_x [g(x)f(x)]^2] &= \frac{1}{|S|^2} \Big( \sum_x \E_{\vecW}[g^2(x)f^2_{\vecW}(x))] + 2\sum_{x \neq y} \E_{\vecW}[g(x)g(y)\text{sgn}(f_{\vecW}(x))\text{sgn}(f_{\vecW}(y))] \Big)\\
&= \frac{1}{|S|^2} \Big( \sum_x \E_{\vecW} [g^2(x)]  + 2\sum_{x \neq y} g(x)g(y) \E_{\vecW} [\text{sgn}(f_{\vecW}(x))\text{sgn}(f_{\vecW}(y))] \Big)\\
&\leq \frac{1}{|S|^2} \Big( \sum_x \E_{\vecW} [g^2(x)]  + 2\sum_{x \neq y} |g(x)||g(y)| |\E_{\vecW} [\Phi_{\vecW}(x) \cdot \Phi_{\vecW}(y)]| \Big)\\
\end{align*}
Next, Theorem~\ref{th:dot-product-applications} implies that 
$$
\E_{\vecW} |\Phi_{\vecW}(x) \cdot \Phi_{\vecW}(y)| \leq L e^{-\Omega(L)},
$$

Substituting above and noticing that $g(x) \in [-1,1]$ we have
\begin{align}
\E_{\vecW}[\E_x [g(x)f(x)]^2] &\leq \frac{1}{|S|^2} \Big( \sum_x 1  + 2\sum_{x \neq y}  Le^{-\Omega(L)}] \Big). \nonumber \\
&\leq Le^{-\Omega(L)} \label{eq:exp-small-corr}.
\end{align}
Next assume that there exists an algorithm $\mathcal{A}$ that makes $p(d,L)$ queries of tolerance $r(d,L)$ to an SQ oracle for a random function $\text{sgn}(f_{\vecW}(x))$ as defined above, and with probability at least half (over the randomness of the algorithm and the random draw of the function), outputs a function $g$ such that
$$
\E_{D}[\text{sgn}(f_{\vecW}(x)) h(x)] \geq \frac{1}{q(d,L)}.
$$
Here we assume that both $r(d,L)$ and $q(d,L)$ are polynomial in $d$ and $L$. To get a contradiction we will use the technique from the work of~\cite{bshouty2002using}~(see Theorem 31). As a first step, since we are in the case of learning with respect to a fixed distribution, from the work of \cite{bshouty2002using} it follows that we can, without loss of generality, assume that the statistical queries $\psi(x,\text{sgn}(f_{\vecW}(x)))$ used by the algorithm are correlation queries, i.e. $\psi(x,f_{\vecW}(x)) = \E_D[\text{sgn}(f_{\vecW}(x))g(x)]$. Next we simulate the algorithm $\mathcal{A}$ and each time the algorithm makes a statistical query $\E_D[\text{sgn}(f_{\vecW}(x))g_i(x)]$, we add $g_i$ to a set $\mathcal{H}$. Finally, if the algorithm outputs a hypothesis $h$ at the end, we add $h$ to $\mathcal{H}$ as well. Notice that if $\mathcal{A}$ makes $p(d,L)$ queries then $|\mathcal{H}| \leq p(d,L) + 1$. Next from \eqref{eq:exp-small-corr} and a union bound over $\mathcal{H}$ we can say that
\begin{multline*}
\label{eq:exp-small-corr-all}
\Pr_{\vecW} \left[ \exists h \in \mathcal{H}: \E_D[(\text{sgn}(f_{\vecW}(x)) h(x))]^2 > \frac{1}{\max(q^2(d,L), r^2(d,L))} \right] \leq \\ O\Big({\max(q^2(d,L), r^2(d,L)) (p(d,L) + 1) e^{-\Omega(L)} }\Big).
\end{multline*}
Since the correlation of each function in $\mathcal{H}$ with $\text{sgn}(f_{\vecW}(x))$ is at most $1/r(d,L)$, a zero answer to every query asked by the algorithm is a valid output of the SQ oracle, and hence with probability at least $1/2$, the algorithm must output a function in $\mathcal{H}$ that is correlated with $\text{sgn}(f_{\vecW}(x))$. In other words, we have that
\begin{align*}
\Pr_{\vecW} \big[ \exists h \in \mathcal{H}: \E_D[(\text{sgn}(f_{\vecW}(x)) h(x))]^2 > \frac{1}{\max(q^2(d,L), r^2(d,L))} \big] &\geq \frac 1 2.
\end{align*}
From the above we get that 
\begin{align*}
\max(q^2(d,L), r^2(d,L)) (p(d,L) + 1) & \geq e^{\Omega(L)}.
\end{align*}

\end{proof}


\section{Interpolation analysis}
\label{app:interpolation}
\newcommand{\beps}{\mathbf{\epsilon}}


We need a bit more notation for the analysis. We denote by the infinite matrix $X$ the linear map from $\RKHS \rightarrow \reals^n$ corresponding to the inputs $x_1, x_2, \ldots, x_n$, so that for any $v \in \RKHS$, $Xv \in \reals^n$ has $i$th component $v^\top\Phi(x_i)$. Note that $XX^\top = \bar{K}$, the kernel matrix for the training data defined by $\bar{\kernel}$. We denote by $\yvec$ the vector $\langle y_1, y_2, \ldots, y_n \rangle^\top \in \reals^n$, and by $\Sigma = \E_{(x, y)}[\Phi(x)\Phi(x)^\top]$ the data covariance matrix.
If $\bar{K}$ is non-singular, then the linear predictor $v_S = X^\top \bar{K}^{-1}\yvec$ interpolates on $S$, i.e., $v_S^\top \Phi(x_i) = y_i$ for all $i \in [n]$, and indeed, is the {\em minimum norm} interpolating linear predictor. 

The first step in the analysis is the following bound on $\|\Sigma\|$:
\begin{lemma}
\label{lem:Sigma-bound}
For any positive integer $N$ and $\gamma \in (0, \nicefrac{1}{2})$, if the number of hidden layers $L \geq  \lceil \frac{\log(N)}{-\log(1 - \frac{\mu}{2})}\rceil + L_0(\Delta(N, \gamma))$, then $\|\Sigma\| \leq \frac{8\ln(N)}{N}$.
\end{lemma}
\begin{proof}
Let $T = \{x'_1, x'_2, \ldots, x'_{N}\}$ be a sample set of size $N$ drawn i.i.d. from the marginal distribution of $\mathcal{D}$ over the $x$-coordinate. Let $X'$ denote by the infinite matrix corresponding to the linear map from $\RKHS \rightarrow \reals^N$ such that for any $v \in \RKHS$, $X'v \in \reals^m$ has $i$th component $v^\top\Phi(x'_i)$. Let $X'{X'}^\top = \bar{K}'$, the top layer kernel matrix for the training data defined by $\bar{\kernel}$.

Then by the definition of $\Delta$, with probability at least $1 - \gamma$, \pref{ass:separation} holds for $T$ with $\delta = \Delta(N, \gamma)$. Conditioned on this \pref{ass:separation} holding, \pref{thm:main-infinite-width-top-layer} (part 1) implies that $\|\bar{K}'\| \leq 2$ since $L \geq \lceil \frac{\log(N)}{-\log(1 - \frac{\mu}{2})}\rceil + L_0(\Delta(N, \gamma))$. Thus, $\|{X'}^\top X'\| = \|X'{X'}^\top\| \leq 2$, which implies that $\Pr_T[\|{X'}^\top X'\| \leq 2] \geq 1 - \gamma$.

Note that for any $x \in \sphere^{d-1}$, we have $\|\Phi(x)\|^2 = \bar{\kernel}(x, x) = \dualact^{(L)}(1) = 1$. Thus $\|\Phi(x)\Phi(x)^\top\| = 1$, $\text{Tr}(\Sigma) \leq 1$, $\text{intdim}(\Sigma) := \frac{\text{Tr}(\Sigma)}{\|\Sigma\|} \leq \frac{1}{\|\Sigma\|}$, and $(\Phi(x)\Phi(x)^\top)^2 = \Phi(x)\Phi(x)^\top$. Thus, Theorem 7.7.1 in \citep{tropp} and some simple calculations imply that for $\ell := \ln(\nicefrac{8}{\|\Sigma\|})$, we have
\[\Pr_T\left[\|\tfrac{1}{N}{X'}^\top X' - \Sigma\| > \tfrac{2\ell}{3N} + \sqrt{\tfrac{2\ell\|\Sigma\|}{N}}\right] \leq \tfrac{1}{2}.\]
By a union bound, we have
\[\Pr_T\left[\|{X'}^\top X'\| \leq 2 \text{ and } \|\tfrac{1}{N}{X'}^\top X' - \Sigma\| \leq \tfrac{2\ell}{3N} + \sqrt{\tfrac{2\ell\|\Sigma\|}{N}}\right] \geq \tfrac{1}{2} - \gamma > 0.\]
This implies that $\|\Sigma\| \leq \tfrac{2\ell}{N} + \sqrt{\tfrac{2\ell\|\Sigma\|}{N}} \leq \tfrac{2\ell}{N} + \tfrac{\nicefrac{2\ell}{N} + \|\Sigma\|}{2} \Rightarrow \|\Sigma\| \leq \frac{6\ell}{N} \Rightarrow \|\Sigma\| \leq \frac{8 \ln(N)}{N}$ using the fact that $\ell = \ln(\nicefrac{8}{\|\Sigma\|})$.
\end{proof}
The following lemma is a standard calculation that is a slight generalization of a similar statement\footnote{Here, we don't need the $\E[y | \Phi(x)]$ to be a linear function of $\Phi(x)$.} in \citep{bartlett2019benign}:
\begin{lemma}
\label{lem:bias-variance}
The excess risk of the minimum norm estimator satisfies
\[ \E_{(x, y)}[(y - v_S^\top \Phi(x))^2] - \E_{(x, y)}[(y - {v^*}^\top \Phi(x))^2] \leq 2{v^*}^\top Bv^* + 2\beps^\top C\beps,\]
where $\beps = \yvec - Xv^*$, $B = (I - X^\top\bar{K}^{-1}X)\Sigma(I - X^\top\bar{K}^{-1}X)$, and $C = \bar{K}^{-1}X\Sigma X^\top \bar{K}^{-1}$.
\end{lemma}
\begin{proof}
Since $v^*$ is a minimizer of $\E_{(x, y)}[(y - v^\top\Phi(x))^2]$, we have $\nabla_v \E_{(x, y)}[(y - {v^*}^\top\Phi(x))^2] = 0$, which implies that $\E_{(x, y)}[(y - v^\top\Phi(x))\Phi(x)] = 0$. Using this fact, we have
\begin{align*}
\E_{(x, y)}[(y - v_S^\top \Phi(x))^2] &= \E_{(x, y)}[(y - {v^*}^\top\Phi(x) + (v_S - v^*)^\top \Phi(x))^2]\\
&= \E_{(x, y)}[(y - {v^*}^\top\Phi(x))^2] + 2(v_S - v^*)^\top \E_{(x, y)}[(y - {v^*}^\top\Phi(x))\Phi(x)] \\
& \quad + \E_{(x, y)}[((v_S - v^*)^\top \Phi(x))^2] \\
&= \E_{(x, y)}[(y - {v^*}^\top\Phi(x))^2] + \E_{(x, y)}[((v_S - v^*)^\top \Phi(x))^2].
\end{align*}
Using this fact, and that $v_S = X^\top\bar{K}^{-1}\yvec = X^\top\bar{K}^{-1}(Xv^* + \beps)$, we get that the excess risk equals 
\begin{align*}
\E_{(x, y)}[((v_S - v^*)^\top \Phi(x))^2] &= 	\E_{(x, y)}[((X^\top\bar{K}^{-1}(Xv^* + \beps) - v^*)^\top \Phi(x))^2] \\
&= \E_{(x, y)}[(X^\top\bar{K}^{-1}X - I)v^* + X^\top\bar{K}^{-1}\beps)^\top \Phi(x))^2]\\
&\leq 2\E_{(x, y)}[(X^\top\bar{K}^{-1}X - I)v^*)^\top \Phi(x))^2] + 2\E_{(x, y)}[(X^\top\bar{K}^{-1}\beps)^\top \Phi(x))^2] \\
&= 2{v^*}^\top Bv^* + 2\beps^\top C\beps.
\end{align*}
The last equality uses the fact that for any $v \in \RKHS$, we have 
\[\E_{(x, y)}[(v^\top \Phi(x))^2] = \E_{(x, y)}[v^\top \Phi(x)\Phi(x)^\top v] = v^\top \Sigma v.\]
\end{proof}

We can now prove \pref{thm:interpolation}:
\begin{proof}[(\pref{thm:interpolation})]
First, as in the proof of \pref{lem:Sigma-bound}, by the definition of $\Delta$, with probability at least $1 - \gamma$, \pref{ass:separation} holds for $S$ with $\delta = \Delta(n^2, \gamma)$. Conditioned on this \pref{ass:separation} holding, \pref{thm:main-infinite-width-top-layer} (part 1) implies that $\lambda_{\min}(\bar{K}) \geq \nicefrac{1}{2}$ since $L = \lceil \frac{\log(n^2)}{-\log(1 - \frac{\mu}{2})}\rceil + L_0(\Delta(n^2, \gamma))$. Thus, with probability at least $1-\gamma$ over the choice of $S$, $\bar{K}$ is non-singular, and hence there exists an interpolating linear predictor.

We now bound the excess risk via \pref{lem:bias-variance}. We first analyze the ${v^*}^\top Bv^*$ part of the bound. Note that $(I - X^\top\bar{K}^{-1}X)$ is the matrix corresponding to the projection on to the orthogonal complement of the row space of $X$, and so $\|I - X^\top\bar{K}^{-1}X\| \leq 1$. Thus, 
$\|B\| \leq \|(I - X^\top\bar{K}^{-1}X)\|\|\Sigma\|\|(I - X^\top\bar{K}^{-1}X)\| \leq \|\Sigma\|$, and so
\begin{equation}
\label{eq:bias}
{v^*}^\top Bv^* \leq \|\Sigma\|\|v^*\|^2.
\end{equation}

Next, we turn to bounding the $\beps^\top C\beps$ part. We have 
\[\beps^\top C\beps \leq \|C\|\|\beps\|^2 = \|C\|\|\yvec - Xv^*\|^2 \leq n(1 + \|v^*\|^2)\|C\|,\]
since for all $i$, $y_i \in [-1, 1]$ and $\|\Phi(x_i)\| = 1$. So now we need to bound $\|C\|$. We have $C \preceq \bar{K}^{-1}X(\|\Sigma\|I) X^\top \bar{K}^{-1} = \|\Sigma\|\bar{K}^{-1}$, so $\|C\| \leq \frac{\|\Sigma\|}{\lambda_{\min}(\bar{K})}$. As described in the beginning of this proof, we have $\Pr_S[\lambda_{\min}(\bar{K}) \geq \nicefrac{1}{2}] \geq 1 - \gamma$. This implies that with probability at least $1-\gamma$ over the choice of $S$, we have 
\begin{equation}
\label{eq:variance}
\beps^\top C\beps \leq 2n(1 + \|v^*\|^2)\|\Sigma\|.
\end{equation}

Finally, note that the setting $L = \lceil \frac{\log(n^2)}{-\log(1 - \frac{\mu}{2})}\rceil + L_0(\Delta(n^2, \gamma))$ implies that \pref{lem:Sigma-bound} holds for $N = n^2$. So, $\|\Sigma\| \leq \frac{16\ln(n)}{n^2}$. Plugging this bound into \pref{eq:bias} and \pref{eq:variance}, and using \pref{lem:bias-variance}, we get the bound stated in the theorem.
\end{proof}

\section{Conditioning for One Layer ReLU Networks}
\label{app:relu_cond}
In this section we establish that given a set of non-collinear points in $\mathbb{R}^d$, a sufficiently wide one layer neural network with (un-normalized) relu activations leads to a non-singular gram matrix at the output layer. 
\begin{theorem}
\label{thm:one-layer-conditioning-sgn}
Let $S = \{x_1, x_2, \dots, x_n\}$ be a set of $n$ vectors in $\mathbb{R}^d$ such that each $x_i$ is a unit length vector and for each $i \neq j$, it holds that $|x_i \cdot x_j| \leq 1-\delta$. 
Let $w_1, w_2, \dots, w_m$ be vectors drawn i.i.d. from $\mathcal{N}(0,I_{d \times d})$ and consider the feature mapping $\Phi: \mathbb{R}^n \to \mathbb{R}^m$ defined as 
$$
\Phi(x) = \frac{1}{\sqrt{m}}(\sigma(w_1 \cdot x), \sigma(w_2 \cdot x), \dots, \sigma(w_m \cdot x)),
$$
where $\sigma: \mathbb{R} \rightarrow \mathbb{R}$ is the (un-normalized) relu activation defined as $\sigma(x) = \max(x,0)$. Let $\Phi(X)$ be the corresponding $m \times n$ data matrix obtained by applying $\Phi$ to points in $S$, i.e., column $i$ of $\Phi(X)$ equals $\Phi(x_i)$. There exists a universal constant $c > 0$, such that if $m \geq c \frac{n^8\log(n/\delta)}{{\delta}^3}$, then w.p. at least $1-1/\text{poly}(n)$, we have that $\sigma_{\min}(\Phi(X)^T \Phi(X)) \geq \Omega(\frac{{\delta^{3/2}}}{n^3})$.
\end{theorem}

\noindent \textbf{Note:} The recent work of \citet{PSG} implies a stronger bound and hence supersedes the result above. In particular, an implication of Theorem 4 in \citet{PSG} is that under the same setting as in the Theorem above, if $m \geq \frac{n^4 \log n}{\delta^3}$, then with probability at least $1-e^{-\tilde{\Omega}\big( m\delta^3 n^{-2} \log^{-3} n\big)}$, we have that $\sigma_{\min}(\Phi(X)^T \Phi(X)) \geq \Omega(\frac{{\delta^{3/2}}}{\text{polylog}(n)})$.

\begin{proof}[Proof of Theorem~\ref{thm:one-layer-conditioning-sgn}]
The proof is a modification of the gradient lower bound argument as detailed in the proof of Lemma 9.3 of \citet{allen2018convergence}. We will show that $\sigma_{\min}(\Phi(X)) \geq \frac{\delta^{\frac 3 4}}{2000 n^{3/2}}$. This will imply the claim of the Theorem. We will first show that for a fixed $\alpha \in \mathbb{R}^n$, such that $\|\alpha\|=1$, $\|\Phi(X)\alpha\|$ is large. Then we will complete the argument using a union bound over an appropriate net for unit length vectors in $\mathbb{R}^n$. We have that 
$$
\Phi(X)\alpha = \sum_{i=1}^n \alpha_i \Phi(x_i)
$$
and hence
$$
\|\Phi(X)\alpha\|^2 = \frac{1}{m}\sum_{j=1}^m (\sum_{i=1}^n \alpha_i h_j(x_i))^2,
$$
where $h_j(x_i) = \sigma(w_j \cdot x_i)$.
Next, fix a particular $j \in [m]$ and let 
$$
T_j = \sum_{i=1}^n \alpha_i h_j(x_i).
$$
We will first show that with non-trivial probability $T_j$ is large. Let $i^* \in [n]$ be such that $|\alpha_{i^*}| \geq \frac{1}{\sqrt{n}}$, with ties broken arbitrarily. Next, we will write
\begin{align*}
    w_j &= \Big(\sqrt{1-\theta^2}z_1 x_{i^*} + g \Big) + \Big(\theta z_2 x_{i^*}\Big)\\
    &= w_{j,1} + w_{j,2}
\end{align*}
Here we pick $\theta = \frac{\sqrt{\delta}}{5n}$ and $z_1, z_2$ are independent $\mathcal{N}(0,1)$ Gaussians and $g$ is a standard $d$ dimensional Gaussians orthogonal to $x_{i^*}$. Next, define $G_j$ to be the following good event
$$
G_j = \mathds{1}\Big(|w_{j,1}\cdot x_{i^*}| \leq \frac{\sqrt{\delta}}{10n} \land \forall i \neq i^*, |w_{j,1}\cdot x_{i}| > \frac{\sqrt{\delta}}{4n} \Big).
$$
Next, we have that $P_{w_{j,1}}(G_j) \geq \frac{\sqrt{\delta}}{50n}$. This is established in Lemma~\ref{lem:conditioning-helper} at the end of the section. Conditioning on $G_j$, i.e., fixing the randomness in $w_{j,1}$, we notice that 
$$|w_{j,2}\cdot x_{i^*}|
= \theta|z_2|$$ and for any $i \neq i^*$, 
$$
|w_{j,2} \cdot x_i| \leq \theta|z_2|.
$$
Since $\theta z_2$ is a standard Gaussian with variance $\theta^2$, we have that the event $E:\{\theta|z_2| \in [\frac{\sqrt{\delta}}{9n}, \frac{\sqrt{\delta}}{5n}]\}$ holds with constant probability, i.e., 
$$
P(E)\geq 0.2.
$$
Now conditioned on $G_j \cap E$, we have that fixing the randomness in $w_{j,1}$ fixes the sign of $\sum_{i \neq i^*}\alpha_i \sigma(w_j \cdot x_i)$. Furthermore, after fixing the randomness in $w_{j,1}$, there is still a probability of $0.5$ over the randomness in $w_{j,2}$ that $\alpha_{i^*} \sigma(w_j \cdot x_{i^*})$ matches that of $\alpha_{i^*} \sigma(w_{j,2} \cdot x_{i^*})$. Combining everything, we get that with probability at least $\frac{\sqrt{\delta}}{500n}$~(over $w_j$), it holds that $T^2_j \geq \frac{\delta}{81 n^2}$. 

Next define $B$ to be the event that for all $j$, $|T_j| \leq 100 \sqrt{n} \sqrt{\log n \log m}$. It is easy to see that $B$ holds with probability at least $1-1/poly(n)$, and hence when conditioned on $B$, we also have that $|T^2_j| \geq \frac{\delta}{81 n^2}$ with probability at least $\frac{\sqrt{\delta}}{500n}$. Next we will argue that when conditioned on $B$, $\|\Phi(X) \alpha\|^2$ is large except with exponentially small probability. Combined with the fact that $B$ happens with high probability, this will imply that over the randomness in $w_1, \dots, w_m$, $\|\Phi(X)\alpha\|$ is large with high probability.

When conditioned on $B$, $\|\Phi(X)\alpha\|^2$ is an average of $m$ independent random variables, each bounded in $[0, 100^2 n \log n \log m]$ and that 
$$
\E[\|\Phi(X)\alpha\|^2] \geq \frac{\delta^{3/2}}{81000 n^3}.
$$
Hence from Chernoff bound and the fact that $B$ holds with high probability we get that with probability at least $1-e^{\Omega(-\frac{\delta^3 m}{n^7 \log n \log m})}$, $\|X\alpha\|^2 \geq \Omega(\frac{\delta^{3/2}}{n^3})$.

Having argued the bound for a fixed $\alpha$, we now consider an appropriate net over unit length vectors in $\mathbb{R}^n$ to argue that over all $\alpha$, $\|\Phi(X)\alpha\|$ is large. In particular, consider an $\epsilon$-net of the unit sphere with $\epsilon = \frac{{\delta}^{3/4}}{2000 n^{5/2} \sqrt{m} \sqrt{\log n \log m}}$. The size of such a net is at most $(3/\epsilon)^n$. Hence, we get that with probability at least $1-e^{-\Omega(\frac{\delta^3 m}{n^7 \log m \log n})} e^{n\log(3/\epsilon)}$, for any vector $\alpha$ in the net, we have $\|\Phi(X)\alpha\| \geq \frac{{\delta^{3/4}}}{1000 n^{3/2}}$. This in turn implies that with the same probability, for any $\alpha$ on the unit sphere with $\hat{\alpha}$ being its closest vector in the net, we have
\begin{align*}
    \|\Phi(X)\alpha\| &\geq \|\Phi(X)\hat{\alpha}\| - \epsilon \|\Phi(X)\|\\
    &\geq \frac{{\delta^{3/4}}}{1000n^{3/2}} -\epsilon \|\Phi(X)\|. \\
\end{align*}
The bound then follows from noticing that via standard Gaussian concentration we have that with probability at least $1-1/poly(n)$, $\|\Phi(X)\| \leq O(\sqrt{m n \log m \log n})$.

\end{proof}

\begin{lemma}
\label{lem:conditioning-helper}
Let $x_1, x_2, \dots, x_n$ be unit length vectors in $\mathbb{R}^m$, where $|x_i \cdot x_j| \leq 1-\delta$ for $i \neq j$. Let 
$$
w = \sqrt{1-\theta^2}zx_1 + g
$$
where $z$ is $\mathcal{N}(0,1)$, $g$ is a standard $d$ dimensional Gaussian orthogonal to $x_1$, and $\theta = {\delta^{1/4}}/(5n)$. Define $G$ to be the event
$$
G = \mathds{1}\Big(|w\cdot x_{1}| \leq \frac{\sqrt{\delta}}{10n} \land \forall i \neq 1, |w\cdot x_{i}| > \frac{\sqrt{\delta}}{4n} \Big).
$$
Then it holds that 
$$
P_w(G_j) \geq \frac{\sqrt{\delta}}{50n}.
$$
\end{lemma}
\begin{proof}
We have that $w\cdot x_1$ is $\mathcal{N}(0,(1-\theta^2))$ and also that $\theta \leq 1/5$. Hence, we have that
\begin{equation}
\label{eq:helper_1}
P(|w \cdot x_1| < \frac{\sqrt{\delta}}{10n}) \geq \frac{\sqrt{\delta}}{25n}.
\end{equation}
For a fixed $i \neq 1$, we have that 
$$
w \cdot x_i = \sqrt{1-\theta^2}z (x_1 \cdot x_i) + g.x_i.
$$
Conditioning on the fact that $|w \cdot x_1|$ is at most $\frac{\sqrt{\delta}}{10n}$, $w \cdot x_i$ is a Gaussian with mean at most $\frac{\sqrt{\delta}}{10n}$ and variance at least $\frac{\delta^2}{4}$~(since projection of $x_i$ on $g$ is at least $\delta/\sqrt{2}$). Hence, with probability at least $1-1/8n$, we have that $|w \cdot x_i| > \frac{\sqrt{\delta}}{4n}$. Using a union bound we get that, conditioned on $z$ being small, with probability at least $0.5$, all $i \neq 1$ satisfies $|w\cdot x_i| > \sqrt{\delta}/(4n)$. Combining with \eqref{eq:helper_1} we get the claim.
\end{proof}

Theorem 1(b) in \citep{gneitling} provides a generalization of \pref{thm:one-layer-conditioning-sgn} to a large class of activations $\sigma$, although it doesn't prove a quantitative lower bound on the smallest eigenvalue. For completeness, we reformulate that theorem in our language here:
\begin{theorem}
\label{thm:general-one-layer-conditioning}
Suppose the Hermite expansion of $\sigma$ has infinitely many even and infinitely many odd coefficients. If the inputs satisfy \pref{ass:separation}, then the kernel matrix $\bar{K}$ for a 1-hidden layer neural network is non-singular.	
\end{theorem}


\section{\normrelu Analyses}
\label{app:normrelu}

In this section, we derive closed form expressions for the functions $\lambda(c)$ and $b(c)$ in the definition of \normrelu:
\[\normrelu_c(x) = \lambda(c) \cdot [\max \{x-c, 0\} + b(c)].\]
Recall that the quantities $\lambda(c)$ and $b(c)$ are chosen so that the function is normalized according to \eqref{eq:normalization}. First, since
\[\E_{X \sim \normal(0, 1)}[\normrelu_c(x)] = \lambda(c) \cdot \E_{X \sim \normal(0, 1)}[\max \{X-c, 0\} + b(c)].\]
So $b(c)$ can be obtained from the equation $\E_{X \sim \normal(0, 1)}[\max \{X-c, 0\} + b(c)] = 0$. Let $\varphi(x) = \frac{1}{\sqrt{2\pi}}\exp(-\frac{x^2}{2})$ be the Gaussian density function $\varPhi(x) = \Pr_{X \sim \normal(0, 1)}[X \leq x]$ be the Gaussian cumulative distribution function. Using formulas for the mean of truncated normal distributions, we have
\[
	\E_{X \sim \normal(0, 1)}[\max \{X-c, 0\} + b(c)] = \int_{x = c}^\infty (x - c) \varphi(x)dx + b(c) = \varphi(c) - (1 - \varPhi(c))c + b(c).
\]
So $b(c) = (1 - \varPhi(c))c - \varphi(c)$. Now, to compute $\lambda(c)$, we note that 
\[\E_{X \sim \normal(0, 1)}[\normrelu_c(X)^2] = \lambda(c)^2 \cdot \E_{X \sim \normal(0, 1)}[(\max \{X-c, 0\} + b(c))^2].\]
So $\lambda(c) = \frac{1}{\sqrt{\E_{X \sim \normal(0, 1)}[(\max \{X-c, 0\} + b(c))^2]}}$. Using formulas for the variance of truncated normal distributions, we have
\begin{align*}
\E_{X \sim \normal(0, 1)}&[(\max \{X-c, 0\} + b(c))^2] = \int_{x = c}^\infty ((x-c) + b(c))^2 \varphi(x)dx + \int_{x=-\infty}^c b(c)^2 \\
&= (1 - \varPhi(c)) + c\varphi(c) + 2(b(c) - c)\varphi(c) + (b(c)-c)^2(1 - \varPhi(c)) + 	b(c)^2\varPhi(c)\\
&= (1-\varPhi(c))\varPhi(c)c^2 + (1-2\varPhi(c))\varphi(c)c + (1-\varPhi(c)-\varphi(c)^2).
\end{align*}
Hence,
\[\lambda(c) = [(1-\varPhi(c))\varPhi(c)c^2 + (1-2\varPhi(c))\varphi(c)c + (1-\varPhi(c)-\varphi(c)^2)]^{-1/2}.\]

\subsection{Conditioning Analysis}
\label{app:normrelubounds}

The following is the main theorem that provides bounds on correlations for the $\normrelu$ activation. We specifically focus on the case $c = -1.5975$ since we use that value of $c$ in all our experiments.

\begin{theorem} 
\label{thm:norm_relu_conditioning}
For the $\normrelu$ activation, there exist constants $\alpha', \delta'$ such that for any two inputs $x,y$ such that $\|x\|^2,\|y\|^2 \in [0.5, 2.0]$ with $\frac{x^{\top}y}{\|x\|\|y\|} \leq 1 - \delta$ for $\delta > \delta'$ and for any $\epsilon>0$, after a number of layers 
\[L \geq \hat{L} := \frac{2}{\alpha'}\log\left(\frac{3}{\min(\epsilon, \mu/4)}\right),\] 
we have that
\[ \frac{ \phi_{W}(x)^{\top} \phi_{W}(y)}{\|\phi_{W}(x)\|\|\phi_{W}(y)\|}\leq B_{\mu/2}\left(L - \hat{L}, \delta - \delta'\right) + \delta'\epsilon\] The constants are given by $\alpha' = 0.0798$ and $\delta' = 0.0185$. 
\end{theorem}

\begin{remark}
Note that the above theorem implies that for any two inputs with corelation sufficiently bounded away from 1 (quantified by $\delta'$), for any $\epsilon \rightarrow 0$, after $O(\log(1/\epsilon))$ layers the corelation decays to $\epsilon$, recovering our exponential conditioning results for $\normrelu$. Furthermore, the range $[0.5, 2.0]$, is chosen arbitrarily, the related constants will depend on the choice of the range. 
\end{remark}

Before stating the proof we will require some preliminaries. We request the reader to recall the notation established in \pref{app:general_norm_proof}. The proof follows the same schematic of the proof for \pref{thm:general_norms}, with the main caveat that unlike in the case of odd function the constant bias term for $\normrelu$, $a_0^{\{\gamma\}}$ is not necessarily $0$ at all $\gamma$ (recall that our conditions impose that $a_0^{\{\gamma\}} = 1$) and hence it needs to be accounted for in our analysis. We now present closed form expressions for the quantities of interest for $\normrelu$. In particular it can be easily derived via calculations similar to those done for the formulas for $b(c)$ and $\lambda(c)$ that 

\[\frac{a_0^{\{\gamma\}}}{\lambda(c)} = \frac{\E_{X \sim \normal(0, 1)}[\normrelu_c(\sqrt{\gamma} X)]}{\lambda(c)} = \sqrt{\gamma}\varphi(c/\sqrt{\gamma}) - (1 - \varPhi(c/\sqrt{\gamma}))c + b(c)\]
Furthermore we can also derive such a closed form expression for the norm transfer function $\hat{\sigma}_l(\gamma)$, and its derivative viz. 
\begin{align*}
	\frac{\hat{\sigma}_l(\gamma)}{\lambda^2(c)} = \sum_{j=0}^{\infty} \frac{\left(a_j^{\{\gamma\}}\right)^2}{\lambda^2(c)} &= \frac{\E_{X \sim \normal(0, 1)}[\normrelu^2_c(\sqrt{\gamma} X)]}{\lambda^2(c)} \\
	&= (c^2 + \gamma - 2cb(c))(1 - \varPhi(c/\sqrt{\gamma})) + (2b(c)  - c) \sqrt{\gamma} \varphi(c/\sqrt{\gamma}) + b(c)^2
\end{align*}
\begin{align*}
	\frac{\hat{\sigma}_l'(\gamma)}{\lambda^2(c)} = 1 - \varPhi(c/\sqrt{\gamma}) + \frac{b(c)}{\sqrt{\gamma}} \varphi(c/\sqrt{\gamma})
\end{align*}

Furthermore we define the following function which represents a form of bias which increases correlations over layers:
\[ \mathrm{bias}(\gamma) := \frac{\left(a_0^{\{\gamma\}}\right)^2}{\sum_{j=0}^{\infty} \left(a_j^{\{\gamma\}}\right)^2}\]

The following observations are evident from the graphs in \pref{fig:normrelucurves}.

\begin{figure}[t]
    \centering
    \begin{minipage}{0.45\textwidth}
        \centering
        \includegraphics[width=1\textwidth]{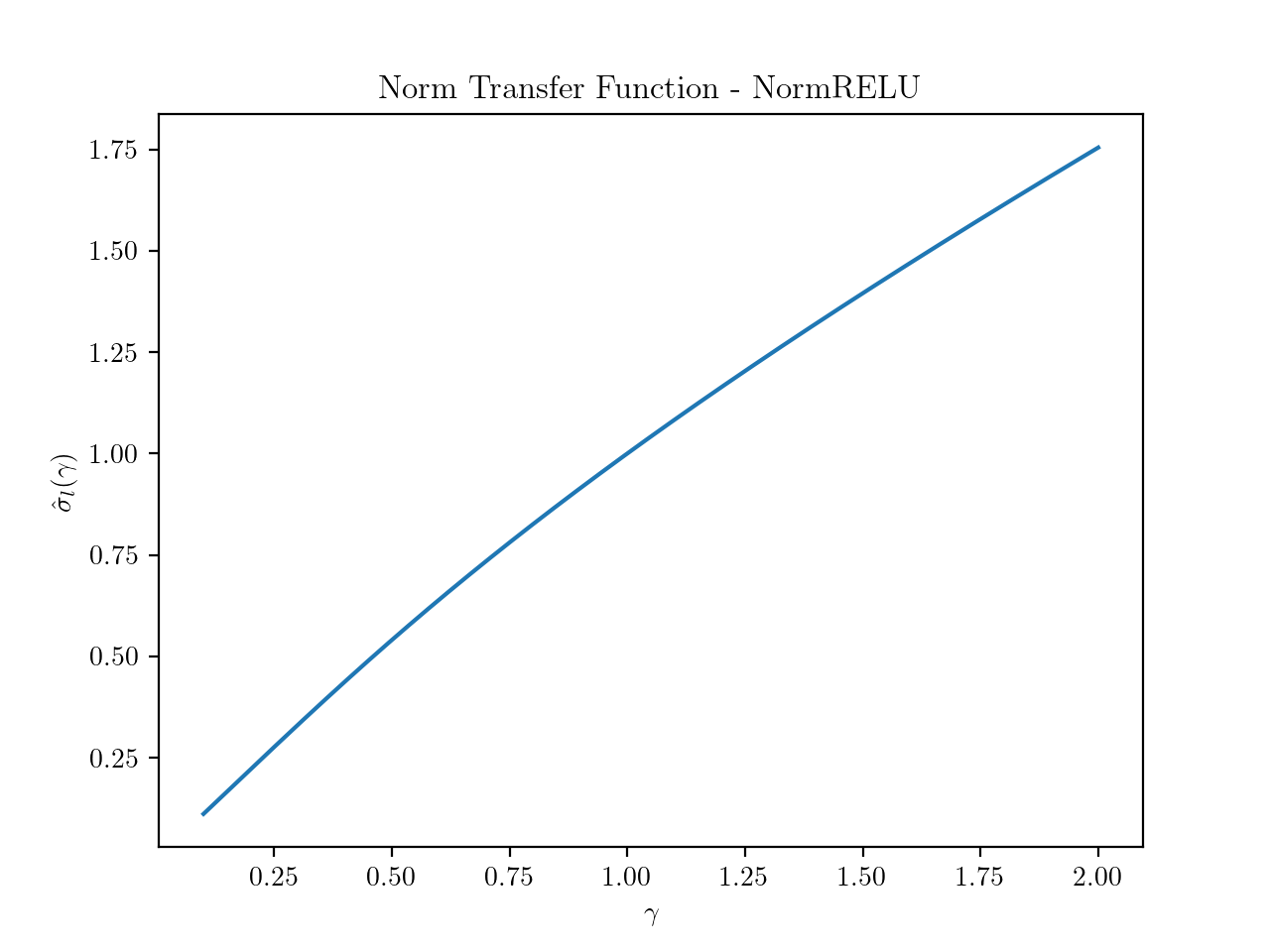} 
    \end{minipage}
    \begin{minipage}{0.45\textwidth}
        \centering
        \includegraphics[width=1\textwidth]{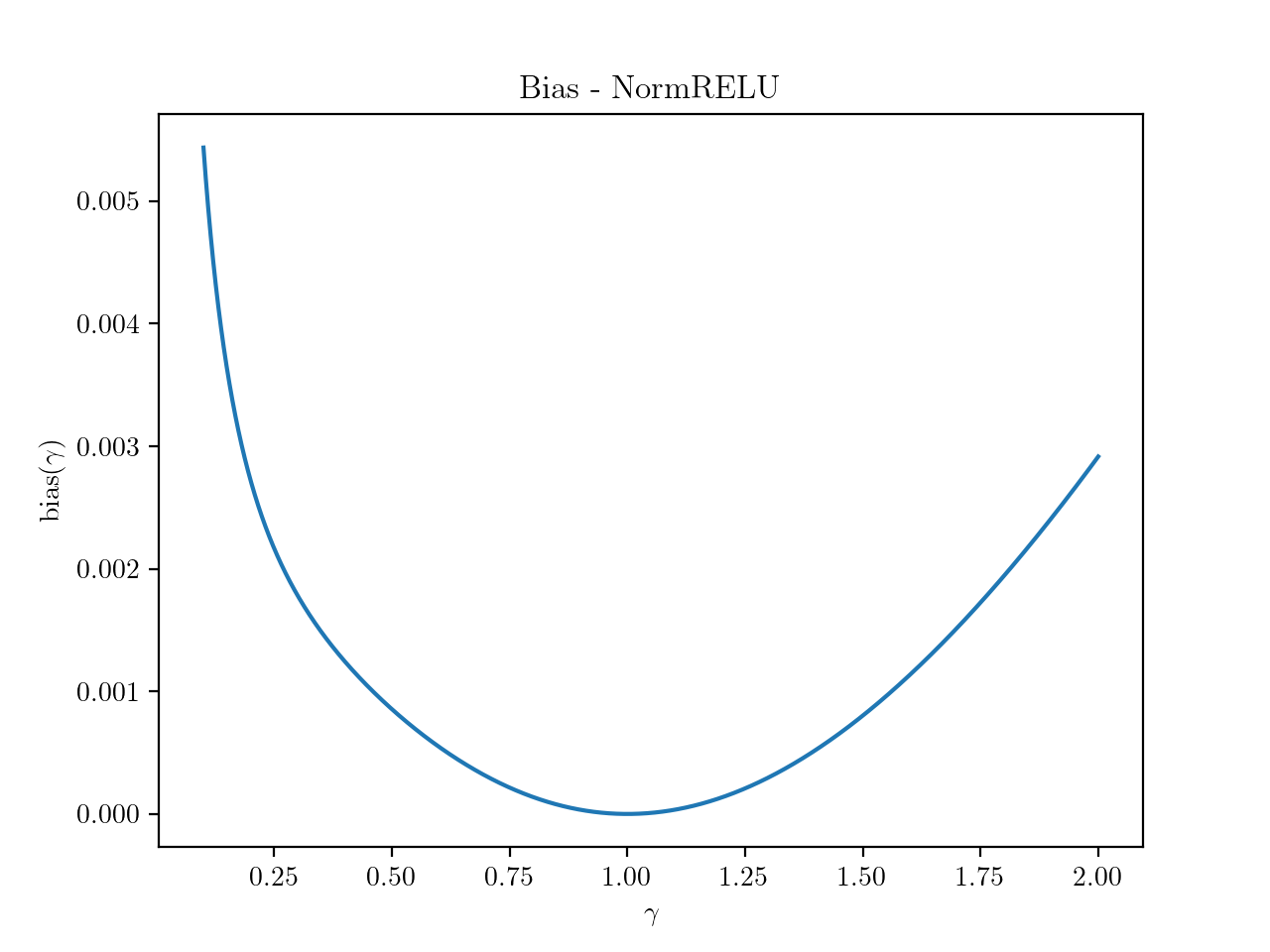} 
    \end{minipage}
    \captionsetup{indention=0.2cm}
   \caption{\label{fig:normrelucurves} Plots of the functions $\hat{\sigma}_l(\gamma)$ and $\mathrm{bias}$ in the range $[0.5,2]$. They can be seen to be concave and convex respectively.}
\end{figure}

\begin{fact}
	The following properties hold for the $\normrelu$ activation
	\begin{enumerate}
		\item For $\gamma \in [0.5,2]$, $\hat{\sigma}_l(\gamma)$ is a non-decreasing concave function with $\hat{\sigma}_l(0.5) \sim 0.5399$ and $\hat{\sigma}'_l(1) \sim 0.8428$.
		\item For $\gamma \in [0.5,2]$, $\mathrm{bias}(\gamma)$ is a convex non-negative function with minimum value at $\gamma=1$, $\mathrm{bias}(1)=0$. Furthermore we have that $\mathrm{bias(0.5)}\sim 0.00086$ and $\mathrm{bias(2)}\sim 0.0029$. 
	\end{enumerate}
\end{fact}


The above facts lead to the following simple calculations. Firstly note that since $\mathrm{bias}(\gamma)$ is a convex function we have the following
\begin{equation}
\label{eq:biaseq}
	\mathrm{bias}(\gamma) \leq  \begin{cases}2\mathrm{bias(0.5)}(1 - \gamma) & \text{if } \gamma \in [0.5,1]\\ \mathrm{bias(2)}(\gamma-1) & \text{if } \gamma \in [1,2]\end{cases}
\end{equation}

The constant (defined in \pref{thm:norms}) $\alpha_{\sigma} = \min(\alpha^-_{\sigma}, \alpha^+_{\sigma})$ for $\normrelu$ activation are as follows. 

\[ \alpha^-_{\sigma} := 2\hat{\sigma}_l(0.5)-1 \sim 0.0798\]
 \[ \alpha^+_{\sigma} := 1 - \hat{\sigma}'_l(1) \sim 0.1572\]

With these in place we are now ready to prove the theorem.

\begin{proof}
	Firstly consider two sequences defined by the following:
	\[ (1-\gamma^-_l) = (1 - \alpha^-_{\sigma})(1-\gamma^-_{l-1}) \qquad \gamma^-_0 = 0.5\]
	\[ (\gamma^+_l - 1) = (1 - \alpha^+_{\sigma})(\gamma^-_{l-1}-1) \qquad \gamma^+_0 = 2.0\]
	The following is immediate from the above derivations above and \pref{eq:biaseq}. 
	\[\sum_{l=0}^{\infty} \mathrm{bias}(\gamma^-_l) = \frac{\mathrm{bias}(0.5)}{\alpha^-_{\sigma}}, \quad \sum_{l=L'}^{\infty} \mathrm{bias}(\gamma^-_l) = \frac{\mathrm{bias}(0.5)\epsilon}{\alpha^-_{\sigma}}\]
	\[\sum_{l=0}^{\infty} \mathrm{bias}(\gamma^+_l) = \frac{\mathrm{bias}(2)}{\alpha^+_{\sigma}}, \quad \sum_{l=L'}^{\infty} \mathrm{bias}(\gamma^+_l) = \frac{\mathrm{bias}(2)\epsilon}{\alpha^+_{\sigma}}\]
	Now note from \pref{thm:norms} we have that if $\|x\|^2 \leq 1$, the norm of the representation at layer $l$, $\bar{k}^l(x,x) \in [\gamma^-_l,1]$ and correspondingly if $\|x\|^2 \geq 1$, $\bar{k}^l(x,x) \in [1,\gamma^+_l]$. The same holds for $y$ as well. Furthermore from the properties of the $\mathrm{bias}$ function we have that,
	\[ \mathrm{bias}(\bar{k}^l(x,x)) \leq \begin{cases} \mathrm{bias}(\gamma_l^-) & \text{if } \|x\| \leq 1 \\ \mathrm{bias}(\gamma_l^+) & \text{if } \|x\| \geq 1\end{cases}.\]
	The same holds for $y$. Define the following shorthand for the corelation between the representations at layer $l$: 
	\[ \rho_l := \frac{\bar{k}^l(x,y)}{\sqrt{\bar{k}^l(x,x)\bar{k}^l(y,y)}},\quad \delta_l:= 1 - |\rho_l|.\]
	Using the analysis derived in the proof of \pref{lem:one-layer-general}we have the following, 
	\begin{align}
	\label{eq:one_layer_normrelu}
		|\rho_l| &\leq \sqrt{\mathrm{bias}(\bar{k}^{l-1}(x,x))\mathrm{bias}(\bar{k}^{l-1}(y,y))} + \sqrt{\frac{\sum_{j=1}^{\infty}\left(a_j^{\{\bar{k}^{l-1}(x,x)\}}\right)^2 |\rho_{l-1}|^j}{\sum_{j=1}^{\infty} \left(a_j^{\{\bar{k}^{l-1}(x,x)\}}\right)^2}} \sqrt{\frac{\sum_{j=1}^{\infty}\left(a_j^{\{\bar{k}^{l-1}(y,y)\}}\right)^2 |\rho_{l-1}|^j}{\sum_{j=1}^{\infty} \left(a_j^{\{\bar{k}^{l-1}(y,y)\}}\right)^2}} \nonumber \\
		&\leq \sqrt{\mathrm{bias}(\bar{k}^{l-1}(x,x))\mathrm{bias}(\bar{k}^{l-1}(y,y))} + (1 - \delta_{l-1})\sqrt{(1 - \mu^{\{\bar{k}^{l-1}(x,x)\}}\delta_{l-1})}\sqrt{(1 - \mu^{\{\bar{k}^{l-1}(y,y)\}}\delta_{l-1})}.
 	\end{align}

Iterating the above equation over $L'$ steps we immediately get that 

 \begin{align*}
		|\rho_{L'}| &\leq 1 - \delta_0 + \sum_{l=1}^{L'} \sqrt{\mathrm{bias}(\bar{k}^{l-1}(x,x))\mathrm{bias}(\bar{k}^{l-1}(y,y))} \\
		&\leq 1 - \delta + \max\left(\sum_{l=0}^{\infty} \mathrm{bias}(\gamma^-_l), \sum_{l=0}^{\infty} \mathrm{bias}(\gamma^+_l)\right) \\
		&= 1 - \delta + \max \left(\frac{\mathrm{bias}(0.5)}{\alpha^-_{\sigma}}, \frac{\mathrm{bias}(2)}{\alpha^+_{\sigma}}\right).
	\end{align*}


We define \[\delta' = \max \left(\frac{\mathrm{bias}(0.5)}{\alpha^-_{\sigma}}, \frac{\mathrm{bias}(2)}{\alpha^+_{\sigma}} \right) \sim 0.0185 . \]
Furthermore as argued in the proof of \pref{thm:general_norms} we have that for any $l \geq L'$ we have that 
\[\mu^{\{\bar{k}^l(x,x)\}}, \mu^{\{\bar{k}^l(y,y)\}} \geq \mu/2.\]

Now using the above and iterating \pref{eq:one_layer_normrelu} from $L'$ onwards we get that 
 \begin{align*}
		|\rho_{L}| &\leq B_{\mu/2}(L-L', \delta - \delta') + \max\left(\sum_{l=L'}^{\infty} \mathrm{bias}(\gamma^-_l), \sum_{l=L'}^{\infty} \mathrm{bias}(\gamma^+_l)\right) \\
		&\leq B_{\mu/2}(L-L', \delta - \delta') + \max \left(\frac{\mathrm{bias}(0.5)\epsilon}{\alpha^-_{\sigma}}, \frac{\mathrm{bias}(2)\epsilon}{\alpha^+_{\sigma}} \right) \\
		&\leq B_{\mu/2}(L-L', \delta - \delta') + \delta'\epsilon.
	\end{align*}
This finishes the proof. 
\end{proof}

\subsection{Comparison of \normrelu and SeLU}
\begin{figure}[t]
    \centering
    \begin{minipage}{0.45\textwidth}
        \centering
        \includegraphics[width=1\textwidth]{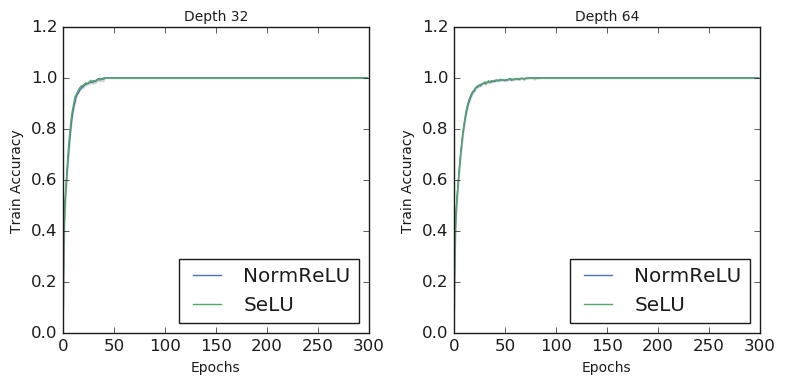} 
    \end{minipage}
    \begin{minipage}{0.45\textwidth}
        \centering
        \includegraphics[width=1\textwidth]{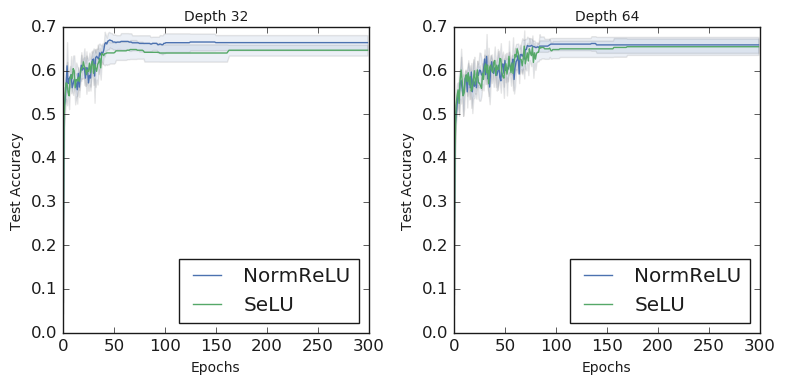} 
    \end{minipage}
    \captionsetup{indention=0.2cm}
   \caption{\label{fig:normrelu_vs_selu_fnn}Train and test accuracy vs. the number of epochs. The plots are obtained by training a depth 32 and a depth 64 fully connected network with 3000 hidden units in each layer.} 
\end{figure}
\begin{figure}
    \centering
    \begin{minipage}{0.45\textwidth}
        \centering
        \includegraphics[width=1\textwidth]{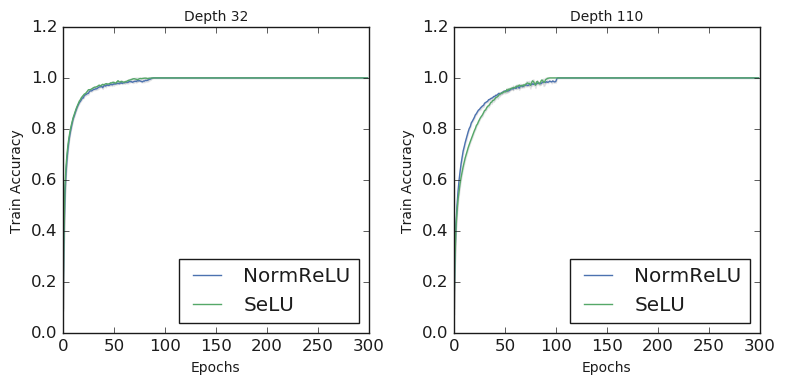} 
    \end{minipage}
    \begin{minipage}{0.45\textwidth}
        \centering
        \includegraphics[width=1\textwidth]{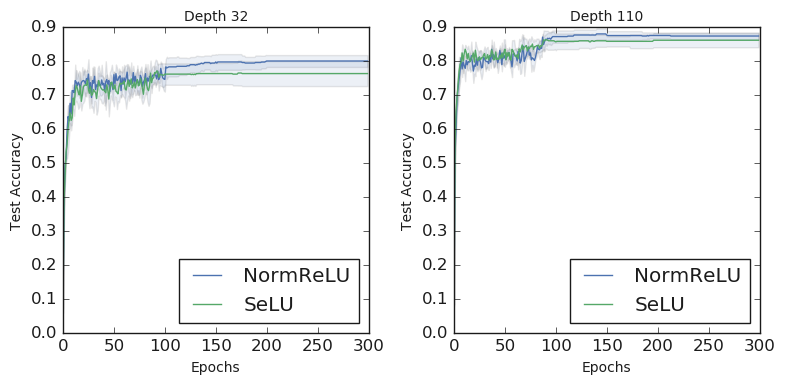} 
    \end{minipage}
    \captionsetup{indention=0.2cm}
   \caption{\label{fig:normrelu_vs_selu_cnn}Train and test accuracy vs. the number of epochs. The plots are obtained by training a depth 32 and a depth 110 ResNet architecture with skip connections but without the use of batch normalization.} 
\end{figure}
In this section we compare the training and generalization behavior of the \normrelu activation that we proposed in this work and the closely related SeLU activation~\cite{klambauer2017self} (see Figure~\ref{fig:norm_relu}). Similar to Section~\ref{sec:experiments} we first train, on the CIFAR-10 dataset, fully connected feedforward networks with depths $32$ and $64$ and containing $3000$ hidden units in each layer. Figure~\ref{fig:normrelu_vs_selu_fnn} shows the training and the test accuracy achieved when using either \normrelu or the SeLU activation. As can be seen, the training and generalization behavior of the two activations is very similar. 

We next train ResNet architectures of depth $32$ and $110$ and with skip connections on the CIFAR-10 dataset. As in Section~\ref{sec:experiments}, we switch off batch normalization. Figure~\ref{fig:normrelu_vs_selu_cnn} shows that in this case as well, the two activations behave similarly in terms of speed of training and generalization performance, with \normrelu achieving slightly higher test accuracies.

\subsection{Discussion of Training Settings}
All the experiments in this work were performed on the CIFAR-10 dataset~\citep{krizhevsky2009learning}. In each case, the learning rate was chosen via cross validation. All the models were trained for 300 epochs and the learning rate was reduced by a factor of $0.1$ after every $100$ iterations. In each case, the model was trained for $10$ independent runs. For the experiment in Figure~\ref{fig:test_acc_vs_dp} the fully connected networks were trained with learning rates of $0.005$ for batch normalization, $0.003$ for layer normalization and $0.0009$ for \normrelu. The convolutional networks were trained with learning rates of $0.1$ for batch normalization, $0.05$ for layer normalization and $0.001$ for \normrelu. For the experiment in Figure~\ref{fig:fnn}, the fully connected networks of depth $32$ were trained with learning rates of $0.01$ for batch normalization, $0.001$ for layer normalization and $0.0009$ for \normrelu. The depth $64$ networks were trained with learning rates of $0.007$ for batch normalization, $0.001$ for layer normalization and $0.0009$ for \normrelu. 

For the experiment in Figure~\ref{fig:cnn} the depth $32$ networks were trained with learning rates of $0.1$ for batch normalization and $0.009$ for \normrelu. The depth $110$ networks were trained with learning rates of $0.05$ for batch normalization and $0.005$ for \normrelu. To compare with fixup initialization as in Figure~\ref{fig:fixup} we use the fixup-110 architecture as proposed in the work of~\citet{zhang2019fixup}. We switch on data augmentation as used in~\citet{zhang2019fixup} and train with fixup initialization using a learning rate of $0.1$ and the learning rate schedule as proposed by the authors in~\citet{zhang2019fixup}. We train with \normrelu using a learning rate of $0.005$. For the experiment in Figure~\ref{fig:normrelu_vs_selu_fnn} we use a learning rate of $0.001$ for depth $32$ and of $0.0005$ for depth $64$. The same learning rate is used for both \normrelu and SeLU. Finally, for the experiment in Figure~\ref{fig:normrelu_vs_selu_cnn} we use a learning rate of $0.01$ at depth $32$ and of $0.009$ at depth $110$ and keep it the same for both the activations.

\end{document}